\documentclass[11pt]{article}

\usepackage[utf8]{inputenc} 

\usepackage[T1]{fontenc}    
\usepackage{CJKutf8}        

\DeclareUnicodeCharacter{FF1A}{:} 

\usepackage{url}            
\usepackage{booktabs}       
\usepackage{amsfonts}       
\usepackage{microtype}      
\usepackage{xcolor}         
\usepackage{amsmath}
\usepackage{amssymb}
\usepackage{mathtools}
\usepackage{amsthm}
\usepackage{algorithm}
\usepackage{algorithmic}
\usepackage{cite}
\usepackage{colortbl} 
\usepackage{subfigure}

\usepackage[margin=1in]{geometry}

\newtheorem{theorem}{Theorem}
\newtheorem{proposition}{Proposition}
\newtheorem{lemma}{Lemma}

\theoremstyle{definition}
\newtheorem{definition}{Definition}
\newtheorem{assumption}{Assumption}
\newtheorem{remark}{Remark}
\newcommand{\matr}[1]{\mathbf{#1}}

\newcommand{\bm}{\mathbf}
\usepackage{authblk}

\usepackage{hyperref}

\begin{document}

\title{Clustered Federated Learning based on Nonconvex Pairwise Fusion}

\author{Xue~Yu,
	Ziyi~Liu,
	Wu~Wang
	and~Yifan~Sun}

\affil{Center for Applied Statistics, School of Statistics, Renmin University of China, Beijing, China \\
\{xueyu\_2019, ziyiliu, wu.wang, sunyifan\}@ruc.edu.cn.}

\maketitle
\begin{abstract}
	This study investigates clustered federated learning (FL), one of the formulations of FL with non-i.i.d. data, where the devices are partitioned into clusters and each cluster optimally fits its data with a localized model. We propose a clustered FL framework that incorporates a nonconvex penalty to pairwise differences of parameters. Without \emph{a priori} knowledge of the set of devices in each cluster and the number of clusters, this framework can autonomously estimate cluster structures. To implement the proposed framework, we introduce a novel clustered FL method called Fusion Penalized Federated Clustering (FPFC). Building upon the standard alternating direction method of
	multipliers (ADMM), FPFC can perform partial updates at each communication round and allows parallel computation with variable workload. These strategies significantly reduce the communication cost while ensuring privacy, making it practical for FL. We also propose a new warmup strategy for hyperparameter tuning in FL settings and explore the asynchronous variant of FPFC (asyncFPFC). Theoretical analysis provides convergence guarantees for FPFC with general losses and establishes the statistical convergence rate under a linear model with squared loss. Extensive experiments have demonstrated the superiority of FPFC compared to current methods, including robustness and generalization capability.
\end{abstract}


\maketitle


\section{Introduction}
Increasing amounts of data are generated by end users' own devices (also called clients), such as mobile phones, wearable devices, and autonomous vehicles. The fast-growing storage and computational capacities of these terminal devices, coupled with concerns over privacy, have led to a growing interest in federated learning (FL)\cite{Federated2020, MAL-083}. Standard FL approaches collaboratively train a shared global model for all devices while maintaining data on each device \cite{konecny2016, Mc2017, mohri19a, Li2018, Karimireddy20a}. As such, critical concerns such as privacy, security, and data access rights are well-addressed in the FL framework. However, as discussed in \cite{Sattler2021}, the local data stored on each device usually have heterogeneous conditional distributions due to diverse user characteristics. The data and statistical heterogeneity observed across different devices obfuscate the efforts to build a global model in many FL applications.  

In addressing the limitations of a global model, a recent proposal, clustered FL \cite{Sattler2021, Ghosh2020, Mansour2020,Yu_Liu_Sun_Wang_2023}, assumes that devices can be grouped into clusters, where devices within the same cluster share identical models. Cluster structures are very common in many applications, such as recommendation systems \cite{Li2003} and precision medicine \cite{Lagakos2006}. Thus, clustered FL has drawn increasing attention and is considered a formulation of FL that can address the data heterogeneity problem. 
However, the key issue regarding clustered FL is that the set of devices in each cluster is unknown \emph{a priori}. To tackle this challenge, several methods have been proposed. Sattler et al.\cite{Sattler2021} proposed a clustered FL (CFL) algorithm under the clustered FL framework in the spirit of the hierarchical divisive clustering algorithm. However, the clustering procedure is implemented at the central server, resulting in high computation costs, especially when the number of devices is large. Inspired by the classical K-means algorithm, Mansour et al. \cite{Mansour2020} introduced the HYPCLUSTER algorithm and Ghosh
et al. \cite{Ghosh2020} presented the Iterative Federated Clustering Algorithm (IFCA). More recently, Marfoq et al. \cite{Marfoq2021} developed the Federated Expectation-Maximization (FedEM) algorithm, which offers a versatile approach for clustered FL and personalized FL. Although the three algorithms eliminate the need for centralized clustering procedures, they require specifying the number of clusters, and the clustering results are usually sensitive to the initial values. 

Motivated by the limitations mentioned above in current studies, we propose a fusion penalization-based framework for clustered FL and develop a novel clustered FL method called \emph{fusion penalized federated clustering} ({FPFC}). In this work, we introduce a nonconvex function, such as the smoothly clipped absolute deviation (SCAD) penalty \cite{Fan2001}, to penalize the pairwise differences of local model parameters. The penalty function can force pairs of local parameters to be exactly equal if they are close to each other, effectively creating clusters within the local parameters. This method places the clustered FL on a solid theoretical footing based on a well-defined objective function. In addition to conventional convergence analysis of algorithms, the statistical properties, e.g., the consistency of identifying the cluster structure of devices, of this method can be rigorously established. The main contributions of our work can be summarized as follows:

\begin{itemize}
	\item We formulate clustered FL as a penalized optimization problem, where the nonconvex fusion penalization term shrinks the pairwise differences of model parameters of devices and encourages minimal differences, thus clustering the local devices.
	
	\item Under the proposed framework, we develop a novel method called \emph{fusion penalized federated clustering} (FPFC) to address general clustered FL problems. This method can automatically determine the cluster membership of devices and derive optimal models specific to each cluster in a distributed setting, without requiring prior knowledge regarding the number of clusters and the composition of devices within each cluster.
	
	\item The proposed method is a general framework that can handle learning tasks with nonconvex losses, such as neural networks. Regarding the algorithm, we extend the standard Alternating Direction Method of Multipliers (ADMM) algorithm to the FL setup. Specifically, we decouple the joint optimization problem into a series of subproblems that can be solved by local devices in parallel, in the meantime, the communication cost is greatly reduced compared to similar methods. Moreover, it allows for inexact minimization of each subproblem, further enhancing computational efficiency.
	
	\item We introduce a novel warmup strategy for hyperparameter tuning in the FL setting. This approach significantly reduces the communication cost compared to the conventional cross-validation method, while simultaneously achieving better performance.
	
	\item Theoretically, we establish the convergence rate of FPFC for general nonconvex losses under standard assumptions. More importantly, unlike IFCA\cite{Ghosh2020}, which assumes full device participation to achieve convergence, our analysis allows partial participation by selecting a subset of devices to perform local updates at each communication round. Furthermore, the statistical convergence rate of the proposed method has been established under a linear model with squared loss. 
	
	\item Finally, we conduct extensive experiments to compare our method with state-of-the-art approaches. We explore the performance of FPFC in aspects of robustness, generalization capability, and communication efficiency. We also introduce several variants of FPFC, including asynchronous updates. The results consistently demonstrate the versatility and effectiveness of FPFC across a diverse range of settings.
	
\end{itemize}

This paper is structured as follows. Section 2 outlines related works. Section 3 provides the relevant background and formulates the main problem. Section 4 details the proposed FPFC algorithm and the regularization parameter tuning strategy. Section 5 presents the theoretical results of FPFC and its statistical properties. The experimental results are presented in Section 6. Section 7 concludes the paper.

\section{Related work}
As this work addresses several issues of existing clustered FL algorithms via a pairwise fusion penalty, we review the related work in two subtopics: heterogeneous FL and clustering via pairwise fusion penalty.

\subsection{Heterogeneous federated learning}
Federated learning has sparked significant interest since it was first proposed by \cite{Mc2017}. The most popular method is the FederatedAveraging (FedAvg) algorithm, and it has been shown to work well empirically, particularly for non-convex problems. FedAvg converges when the data is i.i.d \cite{Stich2018}, but it can diverge in practical settings when the data on the users' devices are non-i.i.d, i.e., heterogeneous \cite{Zhang2020}. Several formulations have been proposed to tackle these heterogeneity issues, such as adding some adaptation terms to the global model, personalized FL, and clustered FL.

Karimireddy et al. \cite{Karimireddy20a} pointed out that FedAvg may suffer from device drift, so they used control variables to reduce the drift between different devices and achieve favorable performance. Li et al. \cite{Li2018} applied a proximal term to limit the impact of statistical heterogeneity under an assumption of bounded local dissimilarity. When data distributions across devices exhibit significant non-i.i.d. characteristics, it becomes challenging to train a global model that performs well across all devices. Another choice is to build personalized models for each device instead of learning a single shared model. As one of the formulations, personalized FL aims to jointly train personalized models that can work better than the global model and the local models without collaboration. The most common method for personalization involves two steps \cite{Kulkarni2020}. First, a global model is trained under the FL framework, and then each device can train a customized model with its private data based on the global model \cite{Wang2019, Finn2017}. \cite{Fallah2020} proposed Personalized FedAvg (Per-FedAvg) to train a shared model that performs well after each device updates it by running a few steps of gradient descent with respect to its own data. Another efficient technique is to regularize personalized models using the relationships between related devices \cite{Zantedeschi2019, Hanzely2020, Dinh2020, Huang2021AAAI}. Smith et al. \cite{Smith2017} suggested a multi-task learning (MTL) framework to handle the statistical challenges of FL and proposed MOCHA to address the systems challenges of federated MTL based on the relationships among tasks. Lin et al. \cite{Lin2022} proposed to project local
models into a shared-and-fixed low-dimensional random subspace and use infimal convolution to achieve personalization.

The formulation of clustered FL has been proposed in recent years. Instead of building a unified model for all devices, clustered FL aims to construct a set of heterogeneous models, each corresponding to one cluster of devices. As one of the earliest clustered FL methods, clustered FL (CFL) \cite{Sattler2021} follows the idea of the divisive hierarchical clustering algorithm. Specifically, it divides the devices recursively in a top-down manner, relying on the cosine similarity between the gradient updates of the devices. This method has several favorable properties. In particular, it is applicable to general nonconvex losses, does not require the number of clusters to be known \emph{a priori}, and comes with theoretical guarantees on the clustering quality. However, it also has several limitations. Practically, optimally bisecting devices requires significant computational resources for the central server, especially when dealing with a large number of devices. Theoretically, its convergence behavior has not been rigorously characterized. Another line of research aims to minimize an objective function (without penalization) involving possible clusterings by using some iterative algorithms that are in the similar spirit of K-means clustering, e.g., Iterative Federated Clustering Algorithm (IFCA) \cite{Ghosh2020} and {HYPCLUSTER} \cite{Mansour2020}. Although their convergence behaviors have been rigorously characterized, these methods need to know the number of clusters in advance and rely on a relatively good initialization to ensure optimal results. Moreover, convergence analysis of IFCA has only been shown for strongly convex losses, although it has good performances in some nonconvex settings. Federated Expectation-Maximization (FedEM) \cite{Marfoq2021} is a state-of-the-art federated multi-task learning method with clustered FL as a special case. Similar to IFCA and HYPCLUSTER, it requires specification of the number of clusters and underlying data distribution, which is infeasible in practice.

Recently, Cho et al. \cite{Cho2023} proposed a framework named COMET which combines clustered knowledge transfer and personalized federated learning to reduce communication costs and accommodate model heterogeneity. Similar to IFCA, COMET also utilizes the K-means clustering algorithm on the server for clustering. Devices are required to have access to a public dataset in this framework. Additionally, Learning Across Domains and Devices (LADD) \cite{shenaj2023ladd} applied K-means clustering to partition devices into clusters based on the image styles of each device in autonomous driving applications. To address the challenges of the straggler effect in synchronous FL and model staleness in asynchronous FL, Zhang et al. \cite{Zhang2023} presented a semi-asynchronous clustered federated learning framework. They utilized spectral clustering on an affinity matrix that incorporates information about devices' similarity, gradient direction, and the latency of model updates. Since clustering is conducted at a pre-processing stage, the proposed method can be categorized as a one-shot clustering approach. Another one-shot clustering method is Principal Angles analysis for Clustered Federated Learning (PACFL) \cite{Vahidian2023}. By constructing a proximity matrix based on the principal angles between the local data subspaces, PACFL performed agglomerative hierarchical clustering.
The benefit of one-shot clustering is that it reduces communication costs as it only requires running clustering before training. However, both algorithms require specifying the number of groups or adjusting the clustering threshold to determine the structure of clusters.

In contrast to the aforementioned methods, our proposed method, FPFC, can automatically cluster devices without the knowledge of the number of clusters or underlying data distribution. In addition, it does not require a complex centralized clustering procedure. It updates only a subset of devices at each communication round and allows each participating device to solve the corresponding minimization problem inexactly up to a given relative accuracy.

\subsection{Clustering via pairwise fusion penalty}       
The pairwise fusion penalty is a recently proposed technique for the clustering problem. Pelckmans et al. \cite{Pelckmans2005} and Hocking
et al. \cite{Hocking2011} applied the pairwise fusion penalty to unsupervised clustering and proposed the convex clustering method. Recently, Ma et al. \cite{Ma2017A, Ma2019} generalized this technique to supervised clustering. Clustered FL can be deemed as a supervised clustering problem in a distributed setting, where subjects are the devices. Another related method is the {Network Lasso} \cite{Hallac2015}. As a generalization of the group lasso to a network setting, {Network Lasso} allows simultaneous clustering and optimization on a network by penalizing the $ \ell_2 $-norm of differences of parameters at adjacent nodes. However, it is specifically designed for the \emph{fully decentralized scenario, where the network structure of devices is known in advance}. The communication cost and requirements of exact minimization and full device participation make it less applicable to FL. (See Section \ref{sec:NL} for further discussion). In the existing literature, one of the most relevant works is Federated Attentive Message Passing (FedAMP) \cite{Huang2021}, which focuses on personalized FL by encouraging collaboration between similar devices. While FedAMP also applies a pairwise fusion penalty, it differs from FPFC in that it aims to seek the optimal personalized model for \emph{each device}. FedAMP cannot cluster devices because it penalizes the \emph{squared norm} instead of the norm of the pairwise differences of parameters. Moreover, it requires full device participation, and each device solves the corresponding minimization problem exactly, making it impractical for FL.

In this work, we apply the pairwise fusion penalty to supervised clustering in the FL setting and develop a distributed optimization algorithm. In contrast to Network Lasso, our focus is on a standard FL setting with a hub-and-spoke topology, where the hub represents a coordinating server and the spokes connect to devices. The server only sends one parameter to each participating device at each communication round, significantly reducing communication costs. Moreover, to attain a reliable cluster structure, we resort to nonconvex penalties, which are asymptotically unbiased and are more aggressive than convex penalties, e.g. $\ell_{1}$ norm penalty, in enforcing a sparser solution. The convergence analysis of FPFC for general nonconvex losses has been rigorously established under standard assumptions (Assumptions \ref{as:smooth1}-\ref{as:pi}). 

\textbf{Notations.} $[m]$ is the set of integers $\{1,\ldots,m\}$, and $\|\cdot\|$ represents the $\ell_2$ norm of vectors.

\section{Clustered federated learning via nonconvex pairwise fusion}
In a system consisting of $m$ devices and one server, communication between the server and devices occurs according to predefined protocols. Each device $i$ is associated with its parameter $\bm{\omega}_i \in \mathbb{R}^d$. The devices are grouped into $L$ disjoint clusters, where devices within the same cluster share identical parameters. Specifically, let $G=\{G_1,\ldots,G_L\}$ represents a mutually exclusive partition of $[m]$ and $\mathbf  \alpha_l$ denotes the common value for $\mathbf  \omega_i$ within the $l$-th cluster. Thus, for all $i\in G_l$, $\mathbf \omega_i=\mathbf \alpha_l$. In practice, the number of $L$ is unknown, and we also have no prior knowledge of the cluster structures. To address these issues, we minimize the following objective function:
\begin{equation}
\label{eq:obj}
\min_{\mathbf \omega} \{F(\mathbf \omega)= \sum_{i=1}^m f_i(\mathbf  \omega_i)+\frac{1}{2m}\sum_{i=1}^m \sum_{j=1}^m g(\|\mathbf \omega_i-\mathbf \omega_j\|,\lambda)\},
\end{equation}
where $\mathbf \omega=(\mathbf  \omega_1^\top,\ldots,\mathbf \omega_m^\top)^\top$, and $g(\cdot,\lambda)$ is a penalty function with hyperparameter $\lambda>0$. 
Here, $f_i(\mathbf \omega_i)=\frac{1}{n_i}\sum_{s=1}^{n_i}\ell(\omega_{i}; z_{i}^s)$, where $ \{z_{i}^s\equiv (X_i^s,y_i^s), s\in [n_i]\} $ is the dataset stored at the $i$-th device and $\ell(\omega_{i}; z_{i}^s)$ is the loss function associated with model parameter $ \omega_{i} $. The first term $\sum_{i=1}^m f_i(\mathbf  \omega_i)$ in \eqref{eq:obj} is the sum of the empirical losses of all devices. 
The second term serves as a pairwise fusion penalty, encouraging the partitioning of devices into clusters. This framework is advantageous over IFCA and many others, as the number and structure of clusters are automatically and simultaneously determined in the minimization process. A popular choice for the penalty function $g(\cdot,\lambda)$ is the $\ell_{1}$ norm or Lasso penalty with $g(t,\lambda)=\lambda|t|$. Nevertheless, this convex function will result in biased parameter estimates. Therefore, our attention is directed towards nonconvex functions. A well-known example is the smoothly clipped absolute deviation (SCAD) penalty \cite{Fan2001}, denoted as $P_{a}(t,\lambda)$, which has the form
\begin{equation}\label{eq:SCAD}
\begin{aligned}
P_{a}(t,\lambda) = \begin{cases} \lambda |t|, & |t| \leq \lambda \\  \frac{a\lambda |t|- 0.5(t^{2}+ \lambda^{2})}{a-1},& \lambda < |t| \leq a\lambda \\ \frac{\lambda^{2}(a+1)}{2}, & |t| > a \lambda. \end{cases}
\end{aligned}
\end{equation}

It is worth noting that, in contrast to FedAMP \cite{Huang2021}, we use $\|\mathbf \omega_i-\mathbf \omega_j\|$ instead of $\|\mathbf \omega_i-\mathbf\omega_j\|^2$ in the penalty function. To see why the nonconvex SCAD penalty consistently recovers the underlying cluster memberships, the first three subfigures of Fig. \ref{fig:solution_paths} present the solution paths for the estimated values of $ \omega $ against $ \lambda $ using the squared $\ell_{2}$ norm, LASSO ($\ell_{1}$ norm), and SCAD penalties in a toy example \cite{Ma2017A}. In this example, we consider a simple linear regression problem with 50 devices divided into 2 clusters. It can be observed that for the SCAD penalty, the estimated values of each device converge to two different values around $ -1 $ and $ 1 $ which are the true values of the two groups when $ \lambda $ reaches a certain value (around 0.4). However, for the $\ell_{1}$ norm penalty, the estimated values $\omega$ converge to a common value quickly as $ \lambda $ increases to $ 0.08 $. For the squared $\ell_{2}$ norm penalty ($ \|\cdot\|_2^2 $), the estimated values of different devices gradually approach each other but do not ultimately converge to a common value.  Neither squared $\ell_{2}$ norm nor $\ell_{1}$ norm penalties can identify the clusters for any $ \lambda $.

\begin{figure}[h]
	\centering
	\includegraphics[width=0.8\columnwidth]{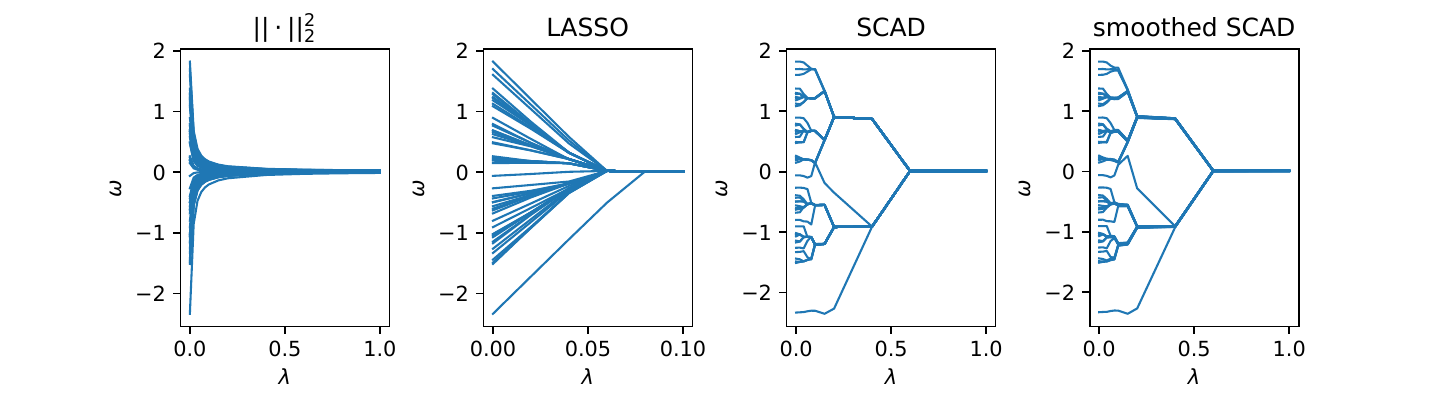}
	\caption{Solution paths for the estimated values of each device against $ \lambda $ by using the squared $\ell_{2}$ norm, LASSO ($\ell_{1}$ norm), SCAD, and smoothed SCAD penalties, respectively.}
	\label{fig:solution_paths}
\end{figure}

\noindent\textbf{Smoothing technique for the SCAD penalty.}
The non-differentiability of the SCAD penalty poses challenges in both implementation and theoretical analysis when employing gradient descent algorithms. To overcome the computational and theoretical challenges, we approximate the non-differentiable SCAD penalty $P_a(t,\lambda)$ by the following continuously differentiable function, called the smoothed SCAD penalty: 
\begin{equation}
\label{eq:sSCAD}
\tilde{P}_{a}(t,\lambda) = \left(\frac{\lambda}{2\xi}t^{2} + \frac{\xi\lambda}{2}\right)I_{\{|t|\leq \xi\}}(t)+P_a(t,\lambda)I_{\{|t|>\xi\}}(t), 
\end{equation}
where $0 < \xi<\lambda$ and $ I_{C}(t) $ is the indicator function of the set $ C $. The last subgraph of Fig. 1 shows the solution paths for the estimated values against $ \lambda $ by using the smoothed SCAD penalty. Because we only approximate the SCAD function around the nondifferential point (i.e., $ 0 $), using the smoothed SCAD function does not affect the clustering results. When $ \lambda $ reaches a certain value, the smoothed SCAD penalty term still identifies the clusters correctly. The properties of $\tilde{P}_a(t,\lambda)$ are summarized in the following proposition.

\begin{proposition}\label{pro:sSCAD}
	The function $\tilde{P}_{a}(t,\lambda)$ is continuously differentiable at any $t\in \mathbb{R}$. Its derivative $\frac{\partial}{\partial t} \tilde{P}_{a}(t,\lambda)$ is Lipschitz continuous with Lipschitz constant $L_{\tilde{g}}=\max\{\frac{\lambda}{\xi},\frac{1}{a-1}\}$. Moreover, 
	$\begin{small}
	P_a(t,\lambda)\leq \tilde{P}_{a}(t,\lambda)\leq P_a(t,\lambda)+\frac{\xi\lambda}{2}, 
	\end{small}$
	for any $t\in \mathbb{R}$, $\lambda>0$, and $0<\xi<\lambda$. 
\end{proposition}
Proposition \ref{pro:sSCAD} shows that $\tilde{P}_{a}(t,\lambda)\to P_a(t,\lambda)$ as $\xi\to 0^+$. This justifies the use of $ \tilde{P}_{a}(t,\lambda) $ as a surrogate of the original SCAD penalty. Our experimental results analyze the effect of different choices of $ \xi $ (see \ref{app:E.2.6} in Appendix E.) and show that
the parameter $ \xi $ will not affect the performance as long as it is small. Let $\tilde{g}(t,\lambda) = \tilde{P}_{a}(t,\lambda)$ and employing the smoothed SCAD penalty $\tilde{g}(t,\lambda)$, we obtain an approximate objective function $\tilde{F}(\mathbf \omega)$ of the original objective function $F(\mathbf \omega)$ in \eqref{eq:obj}. In this manner, our problem is transformed into solving the following optimization problem:
\begin{equation}
\label{eq:sobj}
\min_{\mathbf \omega} \{\tilde{F}(\mathbf \omega)= \sum_{i=1}^m f_i(\mathbf  \omega_i)+\frac{1}{2m}\sum_{i=1}^m \sum_{j=1}^m \tilde{g}(\|\mathbf \omega_i-\mathbf \omega_j\|,\lambda)\}.
\end{equation}
To simplify notations, we omit $\lambda$ in $g(\cdot,\lambda)$ and $\tilde{g}(\cdot,\lambda)$ in the following and will specify it if necessary.


\section{Fusion penalized federated clustering}
Because the penalty function ${g}(\|\omega_i-\omega_j\|)$ (and $\tilde{g}(\|\omega_i-\omega_j\|)$) is non-separable in $\omega_i$, a natural idea is to use ADMM \cite{Boyd2011, Han2022, Bai2022} to recast the optimization problem into separable pieces. Network Lasso \cite{Hallac2015} provides an ADMM algorithm in a distributed manner. In this section, we first present a brief overview of the Network Lasso algorithm and discuss the potential drawbacks of this algorithm in FL. Then, in Section \ref{subsec:FPFC}, we propose a novel distributed optimization algorithm based on ADMM, {FPFC}, to address the clustered FL problem.   
\subsection{Network Lasso} 
\label{sec:NL}
Network Lasso (NetLasso) generalizes the group lasso to a network setting that allows simultaneous clustering and optimization on networks. In fact, our objective function in \eqref{eq:obj} resembles that of NetLasso under a fully connected network, where any two devices are connected by an edge. An ADMM algorithm has been proposed for the NetLasso problem. With ADMM, each individual node solves its own private optimization problem, broadcasts the solution to its neighbors, and repeats the process until the algorithm converges. 

Despite satisfactory numerical performances, the ADMM algorithm of NetLasso has several key drawbacks such that it cannot be directly applied to the FL setup. First, at the local update step, each node (device) requires the model parameters of all its neighbors. In the setup of FL with a hub-and-spoke topology, the server must broadcast all the model parameters of $m$ devices to each device, which will incur heavy communication costs, and more importantly, may lead to private information leaks due to the exposure of model parameters of other devices. Second, the minimization problem must be solved exactly at each node. This is impractical in FL because some devices may have low computational power which makes it difficult to ensure that each minimization step is conducted exactly, especially when $f(\cdot)$ is a complex model, such as a neural network. Third, NetLassso requires all devices to update at each communication round, making it less practical and applicable to FL due to the constraint on communication in federated networks. In addition,  NetLasso only considers the convex case. Although it can be extended to nonconvex cases, convergence analysis for the nonconvex losses has not been established.  

\subsection{Fusion penalized federated clustering}
\label{subsec:FPFC}
To address the aforementioned concerns, we propose a novel distributed optimization method called {FPFC} to cluster devices in the setup of FL. First, we reformulate \eqref{eq:sobj} into a constrained optimization problem by introducing new parameters $\theta_{ij}\equiv \omega_i-\omega_j$ \eqref{eq:const}. Next, we apply a Douglas-Rachford (DR) splitting strategy employed by ADMM to minimize the augmented Lagrangian $ \tilde{\mathcal{L}}_{\rho} $\eqref{eq:lag}. Finally, we revise the $\omega$-update step by utilizing an inexact randomized-block strategy to reduce the communication cost and preserve privacy. The algorithm is presented in Algorithm \ref{alg:1} and its derivation is given in Appendix A.

\begin{algorithm}[htb]
	\caption{Fusion Penalized Federated Clustering (FPFC)}
	\label{alg:1}
	\begin{algorithmic}
		\STATE {\bfseries Input:} number of devices $m$; number of communication rounds $K$; number of local gradient epochs $T_i$, $i\in [m]$; stepsize $\alpha$; parameters of $\tilde{g}(\cdot)$: $\xi$, $\lambda$, $a$; and $\rho$.
		\STATE {\bfseries Output:} $ \{\bm\omega_{i}^{K} \}_{i \in [m]} $.
		\STATE Initialize each device $i\in [m]$ with $\bm\omega_1^0=\ldots=\bm\omega_m^0$.  
		Initialize the server with $\bm\zeta_i^0=\bm\omega_i^0$ for $i\in[m]$, $\bm\theta_{ij}^0=\bm 0$ and $v_{ij}^0=\bm 0$ for $i,j\in [m]$. 
		\FOR{$k=0$ {\bfseries to} $K-1$}
		\STATE 1: \textbf{[Active devices]} Server randomly chooses a subset of devices $\mathcal{A}_{k}$.
		\STATE 2: \textbf{[Communication]} Server downloads $ \bm{\zeta}_{i}^{k}$ to each device $i\in \mathcal{A}_k$.
		\STATE 3: \textbf{[Local update]} For $ i \in \mathcal{A}_{k}$: 
		\item \quad
		$\bm\omega_i^{k,0}= \bm\omega_i^{k}$,
		\FOR{$ t = 0$ {\bfseries to} $T_i-1 $}
		\item
		\vspace{-4mm}
		\begin{equation}
		\label{eq:local update}
		\bm{\omega}_{i}^{k,t+1} =  \bm{\omega}_{i}^{k,t} - \alpha[\nabla f_i(\bm{\omega}_{i}^{k,t})+ \rho( \bm{\omega}_{i}^{k,t} - \bm{\zeta}_{i}^{k})]
		\end{equation}
		\ENDFOR
		\item \quad
		$\bm\omega_i^{k+1} = \bm\omega_i^{k,T_i}.$
		\STATE 4: \textbf{[Communication]} Each device $i\in \mathcal{A}_k$ uploads $ \bm{\omega}_{i}^{k+1} $ to the server.
		\STATE 5: \textbf{[Server update]} For $ i,j \in \mathcal{A}_{k} $ ($ i<j $):
		\item Server computes $\bm{\delta}_{ij}^{k+1} = \bm{\omega}_{i}^{k+1} - \bm{\omega}_{j}^{k+1} +\frac{v_{ij}^{k}}{\rho}$ and then updates
		\item 
		\begin{equation} \label{eq:theta update}
		\hspace{-9mm}
		\bm{\theta}_{ij}^{k+1} = \begin{cases} 
		\frac{\xi \rho}{\lambda + \xi\rho} \bm{\delta}_{ij}^{k+1}, &\|\bm{\delta}^{k+1}_{ij}\|  \leq \xi + \frac{\lambda}{\rho} \\
		(1 - \frac{\lambda}{\rho \|\bm{\delta}_{ij}^{k+1}\|}) \bm{\delta}_{ij}^{k+1}, & \xi + \frac{\lambda}{\rho} \leq  \|\bm{\delta}_{ij}^{k+1}\| \leq \lambda + \frac{\lambda}{\rho} \\  
		\frac{\max\{0,1 - \frac{a \lambda}{(a-1)\rho \|\bm{\delta}_{ij}^{k+1}\|}\}}{1-1/[(a-1)\rho]}\bm{\delta}_{ij}^{k+1}, & \lambda + \frac{\lambda}{\rho} < \|\bm{\delta}_{ij}^{k+1}\| \leq  a\lambda \\ \bm{\delta}_{ij}^{k+1}, & \|\bm{\delta}_{ij}^{k+1}\| > a \lambda. \end{cases}
		\end{equation}
		\item  $v_{ij}^{k+1} = v_{ij}^{k} + \rho(\bm{\omega}_{i}^{k+1} - \bm{\omega}_{j}^{k+1} - \bm{\theta}_{ij}^{k+1})$. 
		\item  $ \bm{\theta}_{ji}^{k+1} = -\bm{\theta}_{ij}^{k+1} $, \ $ v_{ji}^{k+1} =- v_{ij}^{k+1} $.
		\item For $ i \notin \mathcal{A}_{k}  $ and $ j \notin \mathcal{A}_{k} $: $ \bm{\theta}_{ij}^{k+1} = \bm{\theta}_{ij}^{k}$, $ v_{ij}^{k+1} =v_{ij}^k $.
		\item For $ i \in [m] $: $\bm\zeta_i^{k+1}=\frac{1}{m}\sum_{j=1}^m(\bm\omega_j^{k+1}+\bm\theta_{ij}^{k+1}-\frac{v_{ij}^{k+1}}{\rho})$.
		\ENDFOR	
	\end{algorithmic}
\end{algorithm}

\begin{equation}
\label{eq:const}
\begin{aligned}
&\min_{\bm\omega,\bm\theta} \sum_{i=1}^m f_i(\bm\omega_i)+\frac{1}{2m}\sum_{i=1}^m\sum_{j=1}^m \tilde{g}(\|\bm\theta_{ij}\| ), \\
&\ \text{s.t.} \ \bm\omega_i-\bm\omega_j = \bm \theta_{ij},\ i,j\in[m],
\end{aligned}
\end{equation}
where $\theta=\{\theta_{ij}^\top,i,j\in[m]\}^\top$. 
\begin{equation}
\begin{aligned}
\label{eq:lag}
\tilde{\mathcal{L}}_{\rho}(\bm\omega,\bm{\theta}, v)&= \sum_{i=1}^m f_i(\bm\omega_i)+ \frac{1}{2m}\sum_{i=1}^m\sum_{j=1}^m \big[\tilde{g}(\|\theta_{ij}\|) +\langle v_{ij},\bm\omega_i-\bm\omega_j-\theta_{ij}\rangle +\frac{\rho}{2}\|\bm\omega_i-\bm\omega_j-\theta_{ij}\|^2\big],
\end{aligned}
\end{equation}

Similar to the standard ADMM, FPFC updates three sequences $\{\omega_i^k\}$, $\{\theta_{ij}^k\}$, and $\{v_{ij}^k\}$, where $\{v_{ij}^k\}$ are the dual variables. The differences between FPFC and ADMM lie in the $\omega$-update step. First, instead of updating $\omega_i$ in a joint manner, FPFC updates them in a distributed manner. Specifically, we divide the minimization over $\bm\omega$ into a series of subproblems solved by different devices simultaneously:
\begin{equation}
\label{eq:w1}
\omega_i^{k+1}\approx \arg\min\limits_{{\omega_i}} f_i(\omega_i)+\frac{\rho}{2m}\sum_{j=1}^m\|\omega_i-\omega_j^{k}-\theta^k_{ij}+\frac{v_{ij}^k}{\rho}\|^2. 
\end{equation}
Second, as defined in Definition \ref{def:bounded relative error}, the minimization is inexactly solved by running $T_i$ epochs of local (stochastic) gradient descent steps \eqref{eq:local update}. Because \eqref{eq:local update} only involves $\bm\omega_i^k$, $\nabla f_i(\bm\omega_i^k)$ and $\bm\zeta_i^k$, where $\bm\omega_i^k$ is stored in the $i$-th device, $\nabla f_i(\bm\omega_i^k)$ is determined by its private training data $\mathcal{D}_i$, and $\bm\zeta_i^k$ can be requested from the server, the $i$-th device can compute $\bm \omega_i^{k+1}$ locally without exposing its private training data $\mathcal{D}_i$ to any other devices or server. Furthermore, because $\bm{\zeta}_{i}^{k} $ is an average of $\{\bm\omega^k_j,j\in [m]\}$, $\{\bm\theta^k_{ij},j\in [m]\}$, and $\{v^k_{ij},j\in [m]\}$, the $i$-th device can not obtain the model parameters or the private data of any other devices, thus preserving data privacy. Note that transferring $\bm{\zeta}_{i}^{k} $ may also leak the personal information, but we do not
consider this "advanced" privacy in this work. More importantly, different from the exact $\bm\omega_i$-minimization, the central server only needs to send $\bm{\zeta}_{i}^{k}$ to the $i$-th device, instead of broadcasting all parameters, leading to a substantial reduction in communication cost. Lastly, instead of performing updates for all devices $i\in[m]$, we adopt a randomized-block strategy, where only a subset of devices $\mathcal{A}_k\subseteq [m]$ perform local updates and send their model parameters to the server. The overview of FPFC framework is shown in Fig \ref{fig:FPFC}.

FPFC is different from existing proximal gradient-based methods, such as {FedAMP} \cite{Huang2021}, because we rely on a DR splitting scheme and can cluster the devices. Here, three iterates $\bm\omega_i^k$, $\bm\theta_{ij}^k$, and $v_{ij}^k$ are updated sequentially, making it challenging to analyze convergence. Moreover, as existing clustered FL methods, e.g., {IFCA} \cite{Ghosh2020}, in each iteration, only a subset of devices participate in the training process, get the updated $\bm\omega_i$ values, and send them to the server.

\begin{figure*}
	\centering
	\subfigure[]{
		\includegraphics[width=0.3\columnwidth]{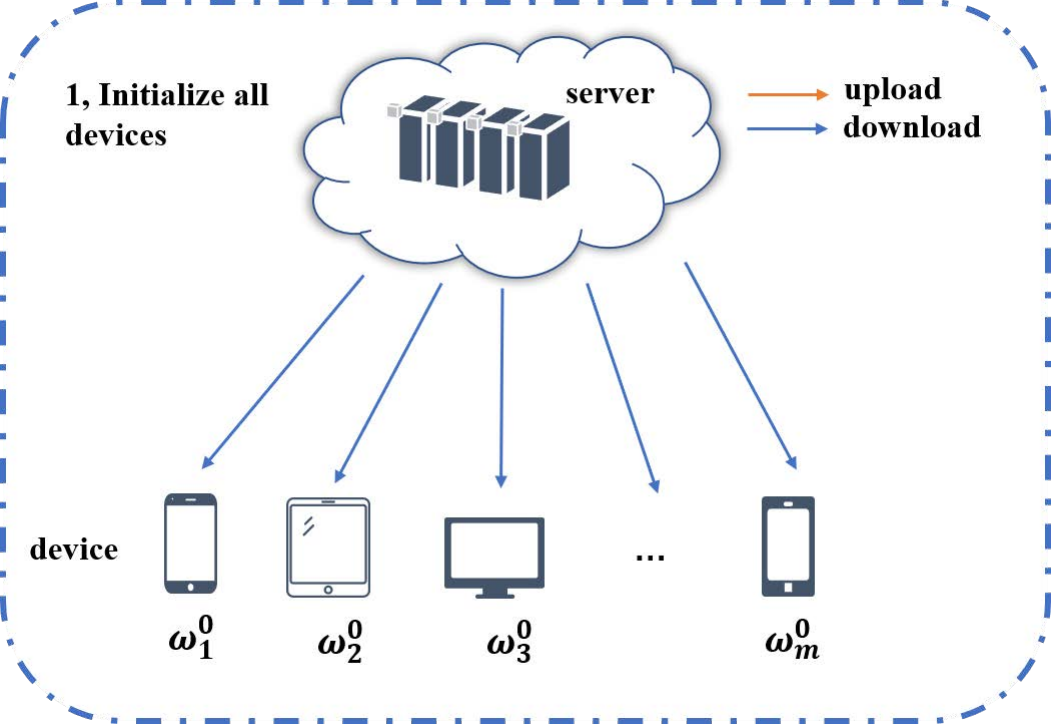}}
	\subfigure[]{
		\includegraphics[width=0.3\columnwidth]{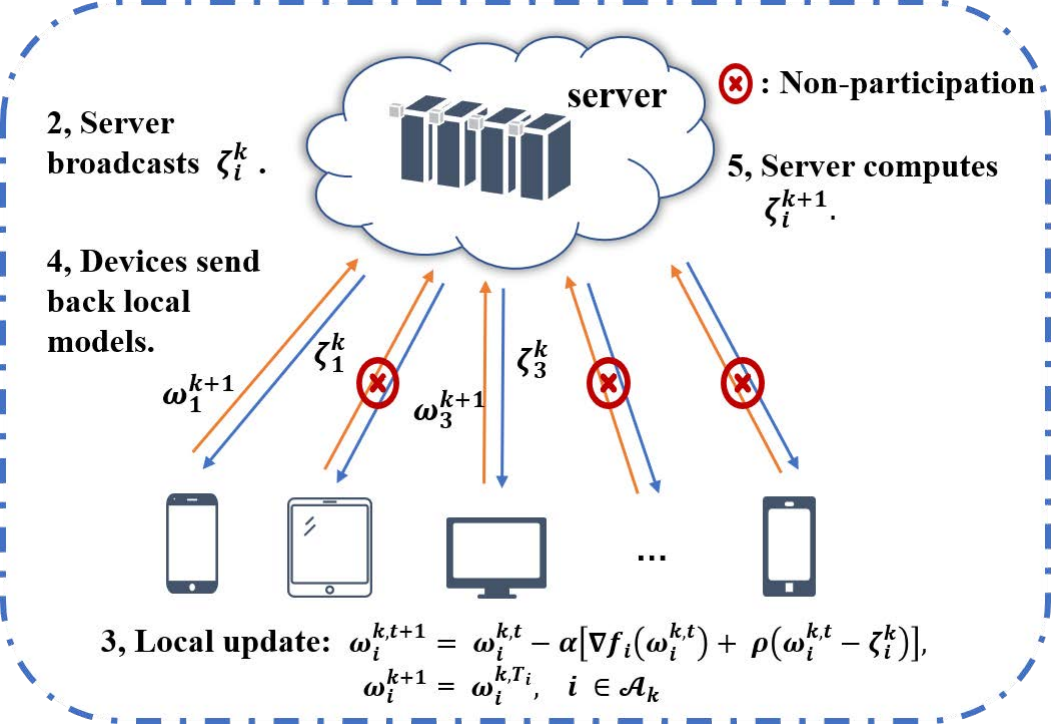}}
	\subfigure[]{
		\includegraphics[width=0.3\columnwidth]{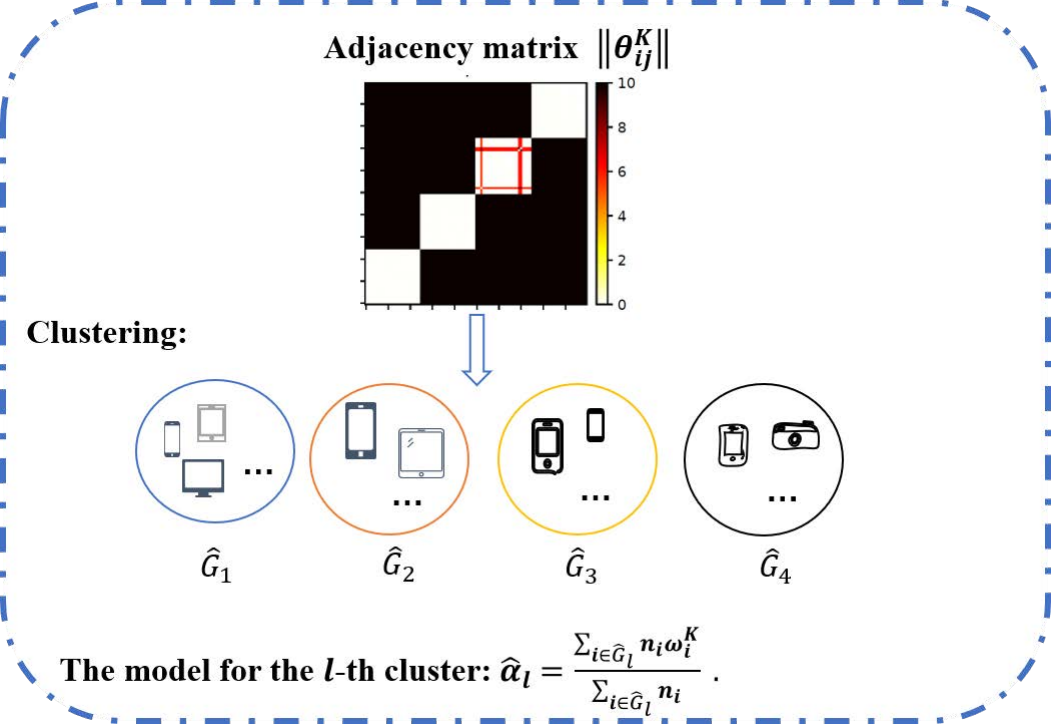}}
	\caption{The framework of FPFC. (a) Initialize each device $i\in [m]$ with $\bm\omega_1^0=\ldots=\bm\omega_m^0$. (b) The server randomly selects a subset of devices $\mathcal{A}_{k}$ and then sends $ \bm{\zeta}_{i}^{k}$ to each device $i\in \mathcal{A}_k$. The devices update their models using their local data and send back the updated models to the server. The server updates $ \theta_{ij}^{k+1} , v_{ij}^{k+1}$ and computes $ \zeta_{i}^{k+1} $. (c) Devices are clustered according to $\|\bm\theta_{ij}^{K}\|$.
	}
	\label{fig:FPFC}
\end{figure*}

\begin{remark}({\bf Computational complexity of the server})
	It is usually assumed that the server has far greater computational power than devices. Note that the server only needs to conduct simple assignment operations in FPFC and only a subset of devices participate in training at each communication round. For all $ i \notin \mathcal{A}_{k} $ and $ j \notin  \mathcal{A}_{k}$, $ \theta_{ij}^{k+1} = \theta_{ij}^{k}, v_{ij}^{k+1} = v_{ij}^{k} $. Together with the symmetry of $ \theta $ and $ v $, the server only needs to conduct  $ \mathcal{O}(\tau(2-\tau)m^2) $ assignment operations at each round, where $ \tau $ is the proportion of participating devices, e.g., $\tau=30\%$. The computation cost of {FPFC} is significantly smaller than that of {CFL}, which is $\mathcal{O}(m^3)$.
\end{remark}

\begin{remark}
	({\bf Clustering in practice})
	As there is no available prior knowledge of the cluster structure, clustering is performed only after the algorithm converges. For the SCAD penalty, $\bm\theta_{ij}^{K}$ can be zero if devices $i$ and $j$ belong to the same cluster. Because we adopt the smoothed SCAD penalty, $\bm\theta_{ij}^{K}$ can only be approximately zero but not exactly zero. In practice, devices $i$ and $j$ are assigned to the same cluster if $\|\bm{\theta}_{ij}^{K}\| \leq \nu$, where $\nu$ represents a threshold. In our experiments, we observed that the results are not very sensitive to $ \nu $, and we can obtain satisfactory results by setting $\nu \in [\xi,0.5]$. As a result, we obtain $\hat{L}$ clusters $\hat{G}_1,\ldots,\hat{G}_{\hat{L}}$, and $\bm{\hat{\alpha}}_l=\frac{\sum_{i\in \hat{G}_l} n_i\bm{{\omega}}^{K}_i}{\sum_{i\in \hat{G}_l} n_i}$ ($ l $ in $[\hat{L}] $).        
\end{remark}
\begin{remark}
	({\bf Choices of initial values})
	In Algorithm 1, the initial values $\bm\omega_i^0$($i\in [m]$) are generated randomly. Our experimental results (see \ref{app:E.2.2} in Appendix E.) show that the final results are very robust to the initial values although the objective function is nonconvex. 
\end{remark}

\subsection{Regularization parameter tuning scheme}
\label{sec:warmup}
The regularization parameter $\lambda$ is a key parameter of {FPFC}, it controls clustering and in turn prediction performance. The popular cross-validation method for parameter tuning divides the data on each device into training, testing, and validation sets, and then chooses the parameter that performs the best on the validation set. Usually, we are interested in training models not just for a single value of $ \lambda $, but for a range of values $ \{\lambda_{1},...,\lambda_{S}\} $. However, it is infeasible to train all regularization parameters in parallel, making it difficult to tune parameters in FL. Although hyperparameter selection is a vital topic, most existing works only tune parameters from a range of values sequentially, which is not plausible due to the large communication cost \cite{Li2018, Deng2020, Huang2021}. We informally have not found relevant discussions about parameter tuning in FL. In this section, we provide a practical guide to tuning the regularization parameter $ \lambda $ in FL.

Breheny and Huang \cite{Breheny2011} argued that the estimated coefficients vary continuously with $ \lambda $ when the objective function is strictly convex. They constructed a decreasing sequence of values for $ \lambda $ and used the estimated coefficients based on the previous value of $ \lambda $ as the initial values for the current $ \lambda $. Inspired by this heuristic method, we present a regularization parameter tuning strategy. Specifically, given a sequence of $ \lambda $: $ \{\lambda_{1},...,\lambda_{S}\}(\lambda_{1} <\lambda_{2}< .. < \lambda_{S}) $, we first run {FPFC} algorithm with the smallest parameter $ \lambda_{1} $, and then consecutively increase $ \lambda $ to the next value when certain criteria are met, such as constraints on rounds of communication or the change of accuracy. The model parameters returned by training with $ \lambda_{i-1} $ are used as the initial values for the model with $ \lambda_i $. We refer to this process as the \emph{warmup} strategy for parameter tuning.

Intuitively, a small increase of $ \lambda $ will only cause a small change in the local minimizers of $ F(\omega) $, and the optimized model parameters with the tuning parameter $ \lambda_{i-1} $ will serve as a reasonably good initial value for the model with $ \lambda_i $. Therefore, the warmup strategy makes the algorithm converge very fast. During the process, we always keep the model which has the best performance on validation sets. As shown in our experiments (see Fig. \ref{fig:warmup_sepetate}), the validation accuracy will usually increase as $ \lambda $ increases, particularly when $ \lambda $ is relatively small. However, beyond a certain level, further increments in $ \lambda $ result in a decrease in the validation accuracy. This occurs because a large $ \lambda $ will aggregate all devices into a single cluster as illustrated in Fig.\ref{fig:solution_paths} and destroy the underlying cluster structure. Finally, we return to train the best model obtained with $ \lambda_{s} $ until the algorithm converges. 


For the traditional cross-validation method, the initial values corresponding to each  $ \lambda_i $ are randomly generated or prespecified, thus it may be time-consuming to approach the minima point for each $\lambda_{i} $. In contrast, by adopting the warmup strategy, only the initial values of the first parameter $ \lambda_{1} $ are randomly generated, and then the initial values corresponding to the following regularization parameters are based on the model obtained from the previous parameters. Because the initial values are close to the minimum values, the required communication rounds for each following parameter can be reduced significantly. It also allows a substantial reduction in running time without parallel tuning. In addition, we may acquire better performance with good initialization for nonconvex losses, see the results in section \ref{sec:sim_warmup}.

\section{Theoretical guarantees}
In this section, we first study the convergence properties of {FPFC} in the general setting of nonconvex loss functions. Then, we characterize the statistical convergence rate of the proposed penalization framework under a linear model with squared loss. All technical details and proofs are deferred to the appendix due to space limitations. We first make the following assumptions.  

\begin{assumption}{($L_f$-smoothness)}
	\label{as:smooth1}
	There exists a Lipschitz constant $L_f>0$ such that 
	\begin{equation*}
	\|\nabla f_i(x)-\nabla f_i(y)\|\leq L_f\|x - y\|, \forall x, y \in \mathbb{R}^d.
	\end{equation*}
\end{assumption} 
\begin{assumption}{(Boundedness)}
	\label{as:bound}
	The optimal value $ F^{*} = \inf_{\bm \omega \in \mathbb{R}^{d}} F(\omega) > -\infty$.
\end{assumption}

Assumption \ref{as:smooth1} implies that all functions $f_i(\cdot)$ ($i\in [m]$) are continuously differentiable. Assumption \ref{as:smooth1} and \ref{as:bound} are standard in nonconvex optimization, and have been widely used in FL \cite{Li2018, Quoc2021, Huang2021}. 

Our convergence analysis allows only a subset of devices to update parameters at each communication round. Similar to the work in \cite{Quoc2021}, we define a proper sampling scheme $ \hat{\mathcal{A}} $ of $ [m] $, which is a random set-valued mapping with values in $ 2^{[m]} $, and $ \mathcal{A}_{k} $ is an i.i.d. realization of $ \hat{\mathcal{A}} $. We make the following assumption about the sampling scheme $ \hat{\mathcal{A}} $.
\begin{assumption}\label{as:pi}
	For each device $ i \in [m] $, there exist $ p_{i}>0 $ such that $ P(i \in \hat{\mathcal{A}}) = p_{i}>0 $.
\end{assumption}
According to Assumption \ref{as:pi}, we have $ p_{i} = \sum_{\mathcal{A}: i\in \mathcal{A}}P(\mathcal{A})$. Define $ \hat{p} = \min_{i \in [m]} \{p_{i}\} $. This assumption suggests that every device has a non-zero probability of participating in the training process.

One of the improvements of {FPFC} over standard ADMM is that all subproblems in (\ref{eq:w1}) can be solved inexactly, which is more realistic in FL. Specifically, we define the $\epsilon$-inexact solution of $\bm\omega_i$-minimization for the $i$-th device at the $k$-th iteration below.
\begin{definition}{\rm{($\epsilon$-inexact solution)}}\label{def:bounded relative error} 
	For the function $h_i(\bm\omega_i)=f_i(\bm\omega_i)+\frac{\rho}{2m}\sum_{j=1}^m \|\bm\omega_i-\bm\omega_j^k-\bm\theta_{ij}^k+\frac{v_{ij}^k}{\rho}\|^2$, we say that $ \bm{\omega}_{i}^{k+1}$ is the $\epsilon$-inexact solution of $\min_{\bm\omega_i}h_i(\bm \omega_i)$ if there is a constant $ \epsilon\in [0,1]$ such that 
	\begin{equation*}
	\|\bm{\omega}_{i}^{k+1} - \hat{\bm{\omega}}_{i}\| \leq \epsilon \|\bm{\omega}_{i}^{k+1} - \bm{\omega}_{i}^{k} \|,
	\end{equation*} 
	where $ \hat{\bm{\omega}}_{i}$ is the exact solution of $\min_{\bm\omega_i}h_i(\bm\omega_i)$.
\end{definition} 
Definition \ref{def:bounded relative error} measures the inexactness of the solution by a relative error $\epsilon$, and such idea concept has been widely used in the literature \cite{Quoc2021, Liu2021}. 
Because the objective function in \eqref{eq:obj} is nonconvex, we only expect to find a stationary point $\omega^*$ of \eqref{eq:obj}. Because \eqref{eq:obj} is equivalent to the constrained problem
\begin{equation}
\label{eq:const1}
\begin{aligned}
&\min_{\bm\omega,\bm\theta} \sum_{i=1}^m f_i(\bm\omega_i)+\frac{1}{2m}\sum_{i=1}^m\sum_{j=1}^m {g}(\|\bm\theta_{ij}\| ), \\
&\ \text{s.t.}\ \bm\omega_i-\bm\omega_j = \bm \theta_{ij},\ i,j\in[m],
\end{aligned}
\end{equation}
each stationary point $\omega^*$ of \eqref{eq:obj} satisfies the KKT condition of \eqref{eq:const1}, i.e., 
\begin{equation}
\label{eq:KKT}
\begin{cases} 
\nabla_{\bm\omega}\mathcal{L}_0 (\bm\omega^*,\bm\theta^*,  v^*)= 0,\\ 
0 \in \partial_{\theta} \mathcal{L}_{0}(\omega^{*}, \theta^{*}, v^{*}),\\
\bm\omega_i^*- \bm\omega_j^*-\bm\theta_{ij}^*= 0, i,j\in [m],
\end{cases} 
\end{equation}
and vice versa, where $\mathcal{L}_0$ is the Lagrange function of \eqref{eq:const1}:	
\begin{equation}
\small
\hspace{-2mm}
\begin{aligned}
\mathcal{L}_{0}(\bm\omega,\bm{\theta}, v) = \sum_{i=1}^mf(\bm{\omega}_i) + \frac{1}{2m}\sum_{i=1}^m\sum_{j=1}^m \big[g(\|\bm\theta_{ij}\|) +\langle v_{ij},\bm\omega_i-\bm\omega_j-\bm\theta_{ij}\rangle \big].
\end{aligned}
\end{equation}
Let us define the mapping  $ \mathcal{G}(\omega,\bm{\theta},v)= \bm{\theta} - \text{prox}_{\mathcal{L}_{0}}(\bm{\theta}) $. The second condition $ 0 \in \partial_{\theta} \mathcal{L}_{0}(\omega^{*}, \theta^{*}, v^{*})$ in \eqref{eq:KKT} is equal to $ \mathcal{G}(\omega^*, \bm{\theta}^{*}, v^*)  = 0 $. In practice, we often wish to find an $\varepsilon$-stationary point of \eqref{eq:obj} defined as follows. 
\begin{definition}\label{def:sta}
	$\bar{\bm\omega}$ is called an $\varepsilon$-stationary point of \eqref{eq:obj} if there exists $\bar{\bm\theta}$ and $\bar{v}$ such that 
	\begin{equation*}
	\begin{cases}
	\mathbb{E}[\|\nabla_{\bm\omega}\mathcal{L}_0 (\bar{\bm\omega},\bar{\bm\theta}, \bar{ v})\|^2]\leq \varepsilon^2, \\
	\mathbb{E}[\| \mathcal{G}(\bar{\bm\omega},\bar{\bm\theta}, \bar{ v})\|^2]\leq\varepsilon^2, \\
	\mathbb{E}[\|\{(\bar{\bm\omega}_i-\bar{\bm\omega}_j-\bar{\bm\theta}_{ij})^\top,i,j\in [m]\}^\top\|^2]\leq \varepsilon^2.
	\end{cases}
	\end{equation*}
	where the expectation is taken overall the randomness generated by the underlying algorithm.
\end{definition}

We first show that at the $k$-th iteration, $\bm\omega_i^{k+1}$ can be an $\epsilon_i$-inexact solution of $\min_{\bm\omega_i}h_i(\bm\omega_i)$ if we run gradient descent in the local update of Algorithm 1 for sufficient epochs. Proof is provided in Appendix B. 

\begin{theorem}\label{them:inexact}
	Suppose Assumption 1 holds. Assume there exists $L_->0$ such that $\nabla^2 f_i  \succeq -L_{-}\bm{I}$ with $\mu=\rho-L_->0$. For any $ i \in \mathcal{A}_{k} $, $\bm\omega_i^{k,0} =\bm\omega_i^{k}$. The stepsize $ \alpha $ satisfies $ 0 < \alpha \leq  \frac{1}{L_f+\rho+\mu}$. Then, for $ \epsilon_{i} >0 $, after $ T_i=\frac{2\log(\frac{\epsilon_{i}}{1+\epsilon_{i}})}{\log c} $ epochs of gradient descent with $ c = 1-\alpha \frac{2\mu (L_{f}+\rho)}{L_f+\rho+\mu}$, we have
	\begin{equation*}
	\|\bm{\omega}_{i}^{k+1} - \hat{\bm{\omega}}_{i}^{k+1}\| \leq \epsilon_{i} \|\bm{\omega}_{i}^{k+1} - \bm{\omega}_{i}^{k}\|,
	\end{equation*}
	where $ \hat{\bm{\omega}}_{i}^{k+1}=\arg\min_{\bm\omega_i}h_i(\bm\omega_i)$.
\end{theorem}

To prove the convergence bounds of \eqref{eq:obj}, we first prove two descent lemmas (Lemmas \ref{lem:descent duo to w} and \ref{lem:descent of L}) regarding $ \tilde{\mathcal{L}}_{\rho} $ during the update and then show the boundedness of $ \tilde{\mathcal{L}}_{\rho} $ from below (Lemma \ref{lem:bounded}). All proofs are provided in Appendix C. 

\begin{lemma}
	\label{lem:descent duo to w}
	(Descent of $\tilde{\mathcal{L}}_{\rho}$ during $\bm\omega$ update) Suppose the assertions in Theorem \ref{them:inexact} holds. For each device $ i \in \mathcal{A}_{k} $, there exists a $\epsilon_i\in [0,1]$ such that $T_{i} = \frac{2\log(\frac{\epsilon_{i}}{1+\epsilon_{i}})}{\log c} $ ($c$ is defined in Theorem \ref{them:inexact}). Let $ (\bm{\omega}^{k}, \bm{\theta}^{k},v^{k} ) $ be the sequence generated by Algorithm 1. Then, for any $ k \in \mathbb{N} $, we have 		
	\begin{equation*}
	\begin{aligned}
	&\quad \tilde{\mathcal{L}}_{\rho}(\bm{\omega}^{k}, \bm{\theta}^{k}, v^{k}) -  \tilde{\mathcal{L}}_{\rho}(\bm{\omega}^{k+1}, \bm{\theta}^{k}, v^{k}) \geq \sum_{i\in \mathcal{A}_{k}} \frac{\rho - L_{f} - 2(c^{-\frac{T_{i}}{2}} -1)^{-1}L_{h}}{2}\|\bm{\omega}_{i}^{k} - \bm{\omega}_{i}^{k+1}\|^2,
	\end{aligned}
	\end{equation*}
	where $L_h=L_f+\rho$.
\end{lemma}

\begin{lemma}\label{lem:descent of L}
	(Monotonically non-increasing of $\tilde{\mathcal{L}}_{\rho}$) Suppose that the assertions in Lemma \ref{lem:descent duo to w} hold. Let $ (\bm{\omega}^{k}, \bm{\theta}^{k}, v^{k} ) $ be the sequence generated by the $k$-th iteration of Algorithm 1. Then, we have
	\begin{equation*}
	\begin{aligned}
	\tilde{\mathcal{L}}_{\rho}(\bm{\omega}^{k}, \bm{\theta}^{k}, v^{k}) -  \tilde{\mathcal{L}}_{\rho}(\bm{\omega}^{k+1}, \bm{\theta}^{k+1}, v^{k+1}) &\geq  \sum_{i\in \mathcal{A}_{k}} \frac{\rho - L_{f} - 2(c^{-\frac{T_{i}}{2}} -1)^{-1}L_{h}}{2}\|\bm{\omega}_{i}^{k} - \bm{\omega}_{i}^{k+1}\|^2 \\
	&\quad + \frac{1}{2m}(\frac{\rho - L_{\tilde{g}}}{2} - \frac{L_{\tilde{g}}^{2}}{\rho})\sum_{i=1}^{m}\sum_{j=1}^{m}\|\bm{\theta}_{ij}^{k} - \bm{\theta}_{ij}^{k+1}\|^2.
	\end{aligned}
	\end{equation*}	
\end{lemma}

\begin{lemma}\label{lem:bounded}
	(Boundedness of $\tilde{\mathcal{L}}_{\rho}$) Let $( \bm{\omega}^{k}, \bm{\theta}^{k}, v^{k} )$ be the sequence generated by the $k$-th iteration of Algorithm 1. If $ \rho> \max\{\frac{\lambda}{\xi},\frac{1}{a-1}\} $ and Assumption 2 holds, for any $ k \in \mathbb{N}^+ $, we have $ \tilde{\mathcal{L}}_{\rho}(\bm{\omega}^{k}, \bm{\theta}^{k}, v^{k}) \geq F^* $.
\end{lemma}

Finally, we utilize the Proposition \ref{pro:sSCAD} to obtain the following theorem.

\begin{theorem}\label{them:kkt}
	Given the Assumptions \ref{as:smooth1}, \ref{as:bound}, \ref{as:pi}, and the existence of $L_->0$ such that $\nabla^2 f_i \succeq -L_{-}\bm{I}$ with $\mu=\rho-L_->0$. Let $ (\bm{\omega}^{k}, \bm{\theta}^{k},v^{k} )$ be the sequence generated by Algorithm \ref{alg:1}. If $\xi$, $\lambda$,  $a$, $\rho$, $T_i$, and $\alpha$ are chosen such that 
	\begin{small}
		\begin{equation}
		\begin{aligned}
		\label{eq:req} 
		\rho> \max\{\frac{L_f}{1-2c^{\frac{T}{2}}},\frac{2\lambda}{\xi},\frac{2}{a-1} ,L_-\}, 	T_i> -\frac{2\log 2}{\log c}, \ 0<\alpha\leq \frac{1}{L_f+2\rho-L_-}, 
		\end{aligned}
		\end{equation}
	\end{small}
	where $ c = 1-\alpha \frac{2\mu (L_{f}+\rho)}{L_f+\rho+\mu}$ and $T=\min_{i\in[m]}T_i$, then we can derive the following results:
	\begin{small}
		\begin{equation*}
		\begin{aligned}	
		\frac{1}{K}\sum_{k=0}^{K-1}\mathbb{E}[\|\nabla_{\bm{\omega}} \mathcal{L}_{0} ( \bm{\omega}^{k+1}, \bm{\theta}^{k+1},v^{k+1}) \|^2 ]  &\leq \frac{C_{1}[F(\omega^0)- F^* + m \xi \lambda/4]}{K},\\
		\mathbb{E}[\|\mathcal{G}(\bm{\omega}^{k+1},\bm{\theta}^{k+1},v^{k+1})\|^2] &\leq  \frac{m^2\lambda^2 \xi^2}{(\lambda + \xi)^2}, \\
		\frac{1}{K} \sum_{k=0}^{K-1}\mathbb{E}[\|\nabla_{v} \mathcal{L}_{0} [\bm{\omega}^{k+1},\bm{\theta}^{k+1},v^{k+1}]\|^2]
		& \leq \frac{C_{2}[F(\omega^0)- F^*+ m \xi \lambda/4]}{K},
		\end{aligned}
		\end{equation*}
	\end{small}
	where $ C_{1} =  \frac{1}{\hat{p}}[\frac{6[(c^{-\frac{T}{2}} - 1)^{-2}L_{h}^2 + \rho^2}{\rho - L_{f} - 2(c^{-\frac{T}{2}} - 1)^{-1}L_{h}} + \frac{12\rho^3}{\rho^2 - L_{\tilde{g}}\rho - 2L_{\tilde{g}}^2}]$, $ C_{2} = \frac{L_{\tilde{g}}^2}{m\rho  (\rho^2 - L_{\tilde{g}}\rho - 2L_{\tilde{g}}^2)}$.
\end{theorem}
Let $(\tilde{\omega}^K,\tilde{\theta}^K,\tilde{v}^{K})$ be selected uniformly at random from $\{(\omega^1,\theta^1,v^1),\ldots,(\omega^K,\theta^K,v^K)\}$ as the output of Algorithm \ref{alg:1}. Suppose $\xi< \frac{\lambda \varepsilon}{m\lambda - \varepsilon}$. Then, after at most $K= \mathcal{O}(\varepsilon^{-2})$
communication rounds, $(\tilde{\omega}^K,\tilde{\theta}^K,\tilde{v}^{K})$ becomes an  $\varepsilon$-stationary point of \eqref{eq:obj} in the sense of Definition \ref{def:sta}.  

\begin{remark}
	({\bf Choice of hyperparameters}) For local update steps $ T_{i} $, the corresponding relative inexact accuracy $ \epsilon_{i} = 1/(c^{-\frac{T_{i}}{2}}-1)$.
	The requirements on hyperparameters $T_i$, $\alpha$, $\xi$, $\lambda$, $a$, and $\rho$ 
	are not restrictive and can be easily met. For example, for any $\xi>0, \lambda>0, a>0$, one can choose $\rho=\max\{3L_f, \frac{2\lambda}{\xi}, \frac{2}{a-1},L_-\}+0.01$, $\alpha=\frac{1}{L_f+2\rho-L_-}$, and $T_i=-\frac{2\log 3}{\log\left[1-\frac{2(\rho-L_-)(L_f+\rho)}{(L_f+2\rho-L_-)^2}\right]}$, $ \epsilon_{i} = 0.5 $.
\end{remark}

\begin{remark}
	({\bf Comparison of communication complexity}) Because \eqref{eq:obj} is nonconvex, our $\mathcal{O}(\varepsilon^{-2})$ communication complexity is state-of-the-art, matching the lower bound complexity (up to a constant factor) of FL on non-i.i.d. data \cite{Zhang2021FedPD}. Although  
	{IFCA} has a faster convergence rate $\mathcal{O}(\log(1/\varepsilon))$, it only considers strongly convex loss functions, while we consider more general nonconvex losses.    
\end{remark}

To characterize the statistical convergence rate, we focus on a linear model. Specifically, we assume that the data on the devices in the $l$-th cluster are generated as following: For $i\in G_l$,  
$y_i^s=\langle \bm X_i^s, \bm\alpha_l^{true}\rangle+\tau_i^s, s\in [n_i],$
where $y_i^s\in \mathbb{R}$, $\bm X_i^s\in \mathbb{R}^d$, $\bm\alpha_l^{true}$ represents the true model parameters for the $l$-th cluster, and $\tau_i^s$ are the i.i.d. random errors. Further, we use the squared loss function. Denote the true values of parameters $\bm\omega$ as $\bm\omega^{true}=((\bm\omega_1^{true})^\top,\ldots,(\bm\omega_m^{true})^\top)^\top$. Then, $\bm\omega_i^{true}=\bm\alpha_l^{true}$ for $i\in G_l$. Define the minimal differences of the common model parameters between two clusters as 
\begin{equation*}
b=\min_{i\in G_l,j\in G_{l'},l\neq l'}\|\bm\omega_i^{true}-\bm\omega_j^{true}\|=\min_{l\neq l'}\|\bm\alpha_l^
{true}-\bm\alpha_{l'}^{true}\|.	
\end{equation*}
Let
\begin{equation*}
n=\sum_{i=1}^m{n_i}, \ n_{\min}=\min_{l\in [L]} \sum_{i\in G_l} n_i, \ n_{\max}=\max_{l\in [L]} \sum_{i\in G_l} n_i,
\end{equation*}
and $|G_{\min}|=\min_{l\in [L]}|G_{l}|$, where $|\cdot|$ is the cardinality of the set. Under the mild assumptions described in Appendix D, we can establish the following statistical convergence rate.
\begin{theorem}
	\label{thm:thm3}
	Suppose that Conditions (1)-(3) in Appendix D hold. Assume that $b>a\lambda$, where the SCAD penalty function $P_{a}(t,\lambda)$ is constant when $|t|>a\lambda$. If $$\lambda\gg m|G_{\min}|^{-1}\sqrt{Ld^3n_{\min}^{-1}\log n}+m\sqrt{n_{\min}^{-1}d\log n},$$ then there exists a local minimizer $\bm{\omega}^{*} = ((\bm\omega_1^{*})^\top,\ldots,(\bm\omega_m^{*})^\top)^\top $ of objective function $F(\bm\omega)$ such that
	\[P(\underset{i\in [m]}{\sup}\Vert{\bm \omega}_i^{*}-\bm\omega_i^{true}\Vert \leq \Lambda_n)\geq 1-2Ldn^{-n_{\min} n_{\max}^{-1}}-2dn^{-1},\]
	where $ \Lambda_n=C^{-1} |{G}_{\min}|^{-1}c_1^{-1/2}\sqrt{2Ldn_{\min}^{-1}\log n} $, $C$ and $c_1$ are positive constants defined in Conditions (1) and (3), respectively.
\end{theorem}

Theorem \ref{thm:thm3} demonstrates that the objective function (1) has a local minimizer $\bm\omega^*$ that each element $\bm\omega_{i}^*$ of $\bm\omega^*$ stays within a ball of radius $\Lambda_n$ centered at the true model parameters $\bm\omega_{i}^{true}$. $\Lambda_n$ is the uniform convergence rate of the least squares estimator of model parameters of all devices when the cluster structure is known. Similar to the convergence rate of least squares estimators in the conventional setting, $\Lambda_n$ decreases at a rate proportional to the sample size of the smallest cluster. The hyperparameter $\lambda$ is critical in FPFC. On the one hand, $\lambda$ must be less than the upper bound $b/a$ such that different devices in different clusters can be distinguished, while on the other hand, it must be greater than the lower bound listed in Theorem \ref{thm:thm3}, such that the devices in the same cluster can be grouped together. 

\section{Performance evaluation}
\label{sec:num}
In this section, we conduct extensive experiments to evaluate the performance of FPFC. We first describe the experimental settings and benchmark datasets. Then, we show the advantage of the proposed regularization parameter tuning strategy. Besides, we compare the performance of different methods in aspects of robustness, generality, and communication efficiency. We also present several variations of FPFC in practice. More experimental details are provided in Appendix E. Our code
is released at \url{https://github.com/xueyu-ubc/FPFC}.

\subsection{Experimental settings}

\textbf{Datasets and models.} We conduct experiments on different synthetic datasets and three real-world datasets:
\begin{itemize}
	\item \textit{Synthetic.} We begin with a multi-classification problem with synthetic data. There are $m=100$ devices, and the number of samples on each device follows a power law, in the range [250, 25810] \cite{Li2018}. There are $L=4$ clusters, each consisting of 25 devices. For each cluster $l\in [L]$, we first generate a weight matrix $ W_{l}\in \mathbb{R}^{10 \times 60}$ and a bias $ b_l \in \mathbb{R}^{10}$. Each element in $W_l$ and $ b_l$ is independently drawn from $\mathcal{N}(\mu_l,1)$, where $ \mu_{l} \sim \mathcal{N}(0,1)$. For $i\in G_l$, the features $ X_i^s$ for $s\in [n_i]$ are independently generated from $\mathcal{N}(\bm 0,\bm I_{60})$, and response $y_i^s$ is then generated from $y_{i}^s = \arg\max(\text{softmax}(W_{l}X_{i} + b_{l} + \bm{\tau}_{i}^s))$, where each noise vector $\bm{\tau}_{i}^s$ is generated from $\mathcal{N}(\bm 0, 0.5^2\bm I_{10})$. We use the cross-entropy loss functions. 
	
	\item \textit{Housing and Body fat (H$ \& $BF).} We consider a linear regression task based on the Housing and Body fat (H\&BF) benchmark datasets \cite{Dua2019}. We set $m=8$ devices. To simulate the cluster structure of devices, we evenly allocate the Housing dataset to the first six devices and the Body fat dataset to the last two devices, resulting in $L=2$ clusters. We use a linear model with squared loss.
	
	\item \textit{MNIST $ \& $ FMNIST.} We use two public benchmark datasets, MNIST \cite{Lecun1998} and FMNIST \cite{Xiao2017}. Take MNIST as an example, we first partition 60000 training images and 10000 testing images into $m=20$ devices by a Dirichlet distribution, each of which belongs to one of $ L = 4 $ clusters. We then modify data on each device by randomly swapping two labels\cite{Sattler2021}. Specifically, the first cluster consists of devices 1-5, where each devices' data labeled as "0" are relabeled as "8" and vice versa. The second cluster consists of devices 6-10, where each device's data labeled as "1" and "7" are swapped similarly, and so on. The testing data are processed in the same way. We use a multilayer convolutional neural network with the cross-entropy loss. To account for some common properties of the data, we adopt the weight sharing technique from multi-task learning \cite{Caruana1997}, as used in {IFCA}. Specifically, the weights for the first three layers are common across all devices and clustering is only applied for the last layer, producing a predicted label for each image.
\end{itemize}

\textbf{Evaluation metrics.} We focus on four metrics for comparisons: (1) testing accuracy ({Acc}) for classification problems or testing loss for linear regression problems measured by RMSE between the predicted responses and true responses, both of which are averaged over all devices, (2) number of identified clusters (Num), (3) adjusted Rand index (ARI), which measures the agreement between the identified cluster structure and true cluster structure, and (4) communication cost (Cost), which is measured by the number of parameters that need to be transmitted during the communication process. For each dataset, we run all algorithms with three random seeds and report the average and standard deviation of the four metrics. 

\textbf{Setup.} We evaluate the performance of FPFC and compare it with the state-of-the-art clustered FL algorithms (IFCA,  CFL, and PACFL). To make our experiments more comprehensive, we also report the performance of FedAvg and two personalized FL algorithms (Per-FedAvg and LG-FedAvg (LG) \cite{liang2020think}). LOCAL is a baseline that independently trains the personalized model of each device without collaboration. Moreover, we implement FPFC-$\ell_{1}$, a variant of FPFC, which uses an $\ell_1$ penalty, i.e., $g(\|\bm\omega_i-\bm\omega_j\|,\lambda)=\lambda\|\bm{\omega}_{i} - \bm{\omega}_{j}\|_{2}$. Note that FedAMP requires exact computations of each intermediate minimization subproblem, making it unrealistic for FL, we do not consider this method.

\textbf{Hyperparameter.} We randomly split the data on each device into an 80\% training set and a 20\% testing set, where the training set is used to train models, and the testing set is to make predictions. We further divide the training set for each device into $ 80\%$ for training and $ 20\% $ for validation, where the validation set is used to select hyperparameters. For {CFL}, we tune $\varepsilon_1$ in the range $[0,1]$ and $\varepsilon_2$ in the range $[\varepsilon_1,10\varepsilon_1]$, as suggested in \cite{Sattler2021}. For {FPFC} and {FPFC-$\ell_1$}, we tune $\lambda$ in the range $[0,5]$ to achieve the lowest {RMSE} on the validation set for H\&BF and $[0,1]$ to achieve the highest {Acc} for all synthetic data, MNIST and FMNIST. For a fair comparison, we do not adopt the warmup tuning scheme, but we discuss the advantages of different tuning strategies later. In addition, we set $\rho=1$, $a=3.7$, and $\xi=10^{-4}$. For IFCA, we implement gradient averaging in local updates and give the correct number of clusters in advance. For PACFL, there is a threshold controlling the number
of clusters and we tune it in a range that depends upon different datasets. For all algorithms, each device performs 10 epochs of local GD updates for all synthetic data, 20 epochs for H\&BF, and SGD updates for 10 epochs and a batch size of 100 for MNIST and FMNIST. The learning rates are 0.1 for all synthetic data and 0.01 for H\&BF. For MNIST and FMNIST, we initialize the learning rate to 0.01 and multiply it by 0.9 for every 5 communication rounds. For each dataset, we randomly sample $30\%\sim 50\%$ devices to perform the local update at each communication round. 

\subsection{Performance comparison}

\begin{figure}[h]
	\centering
	\includegraphics[width=0.8\columnwidth]{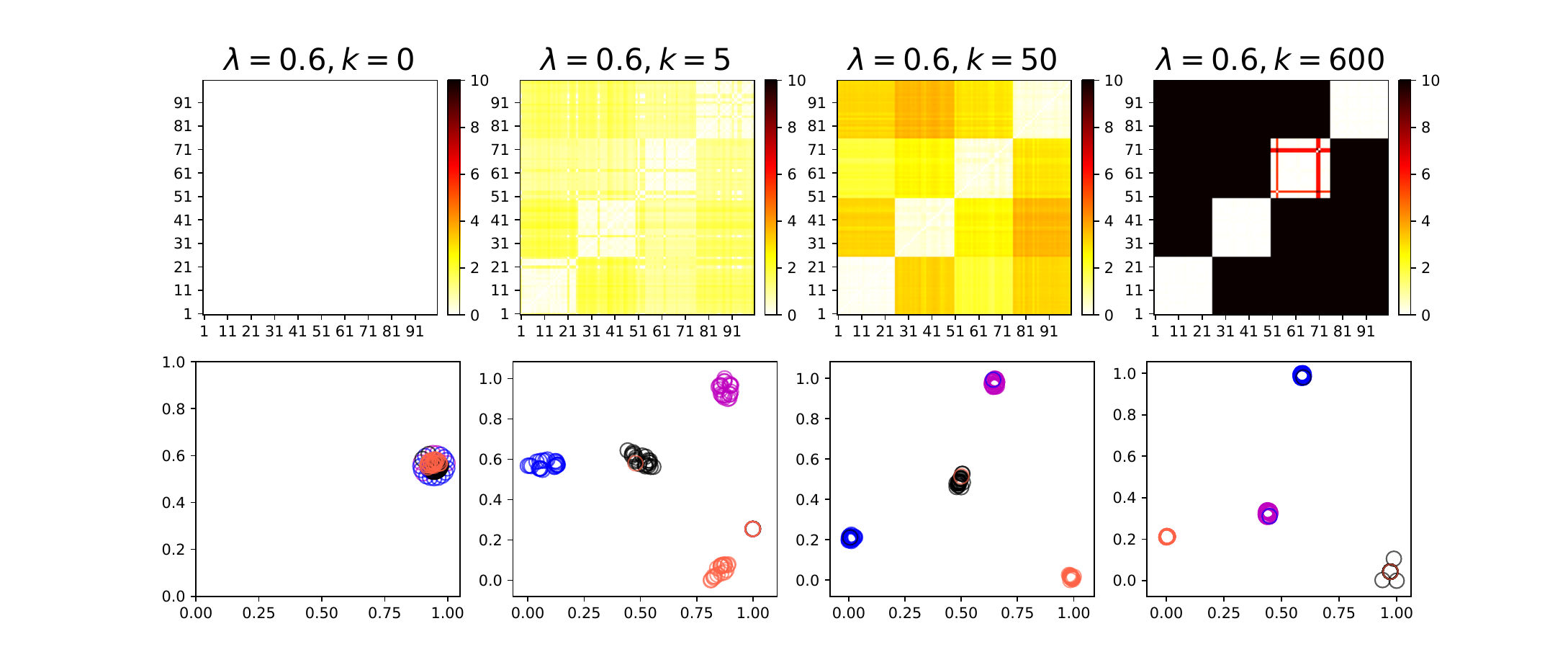}
	\hspace{-0.5cm}
	\caption{Visualization of $\|\bm\omega_i-\bm\omega_j\|$ computed by {FPFC}(upper) and the cluster structure by using t-SNE (below) at $0$-th, $5$-th, $50$-th, and $600$-th communication round under $\lambda=0.6$.}
	\label{fig:group4}
\end{figure}

\begin{table*}[t]
	\caption{Experimental results on synthetic and H\&BF datasets.}
	\label{table:linear}
	\centering
	\resizebox{0.9\linewidth}{!}{
		\begin{tabular}{c|cccc|cccc}
			\toprule
			& & Synthetic & &	& &  H\&BF & &  \\
			\midrule 
			&Acc &Num & ARI & Cost &RMSE & Num& ARI & Cost \\
			\midrule
			LOCAL    &84.97\% $\pm$ 0.06 & $\times$ & $\times$ & 0 & 4.89 $\pm$ 0.19 & $\times$ & $\times$  & 0\\
			FedAvg   & 30.35\% $\pm$ 0.05  & $\times$ & $\times$ &  $\mathbf{ 7.33 \times 10^7 }$& 6.00 $\pm$ 0.36  & $\times$ & $\times$  & \textbf{12120} \\
			LG   & 69.72\% $\pm$ 0.06  & $\times$ & $\times$ &  $\mathbf{7.33 \times 10^7 }$  & 4.52 $\pm$ 0.39  & $\times$ & $\times$  & \textbf{12120}\\
			Per-FedAvg   & 56.28\% $\pm$ 0.06  & $\times$ &$\times$  &  $\mathbf{ 7.33 \times 10^7 }$ & 5.29 $\pm$ 0.27  & $\times$ & $\times$ & \textbf{12120}   \\
			IFCA    & 60.42\% $\pm$ 0.17 & 2.33 $ \pm $ 0.94  & 0.47 $\pm$ 0.33  & $ 1.83 \times 10^9 $ & 5.08 $\pm$ 0.21 &  1.00 $\pm $ 0.00  & 0.00 $\pm$ 0.00   & 282000  \\
			CFL     &86.60\% $\pm$ 0.05 &42.0 $ \pm$ 4.55  & 0.28 $\pm$ 0.11  & $ 1.59 \times 10^8 $& 4.68 $\pm$ 0.41 &1.33 $\pm$ 0.47  & 0.15 $\pm$ 0.22 & 150120\\
			PACFL   &88.30\% $\pm$ 0.06 &32.0 $ \pm $ 37.53  & 0.43 $\pm$ 0.37 & $ 7.34 \times 10^{7} $ & \textbf{4.08 $\pm$ 0.29} & 2.33 $\pm$ 0.47  & 0.82 $\pm$ 0.25 & 12336 \\
			FPFC-$\ell_{1}$  & 83.85\% $\pm$ 0.06  & 5.67 $\pm$ 2.87& 0.63 $\pm$ 0.21 & $\mathbf{ 7.33 \times 10^7 }$  & 4.20 $\pm$ 0.23 & \textbf{2.00 $\pm$ 0.00}  & \textbf{1.00 $\pm$ 0.00} &\textbf{12120} \\
			FPFC    & \textbf{89.46\% $\pm$ 0.04}  &   \textbf{4.00 $\pm$ 0.00}  & \textbf{1.00 $\pm$ 0.00}  & $\mathbf{ 7.33 \times 10^7 }$ & 4.09 $\pm$ 0.24 &   \textbf{2.00 $\pm$ 0.00}  & \textbf{1.00 $\pm$ 0.00}    & \textbf{12120} \\
			\bottomrule
		\end{tabular}
	}
\end{table*}

\begin{figure*}
	\centering
	\includegraphics[width=0.9\columnwidth]{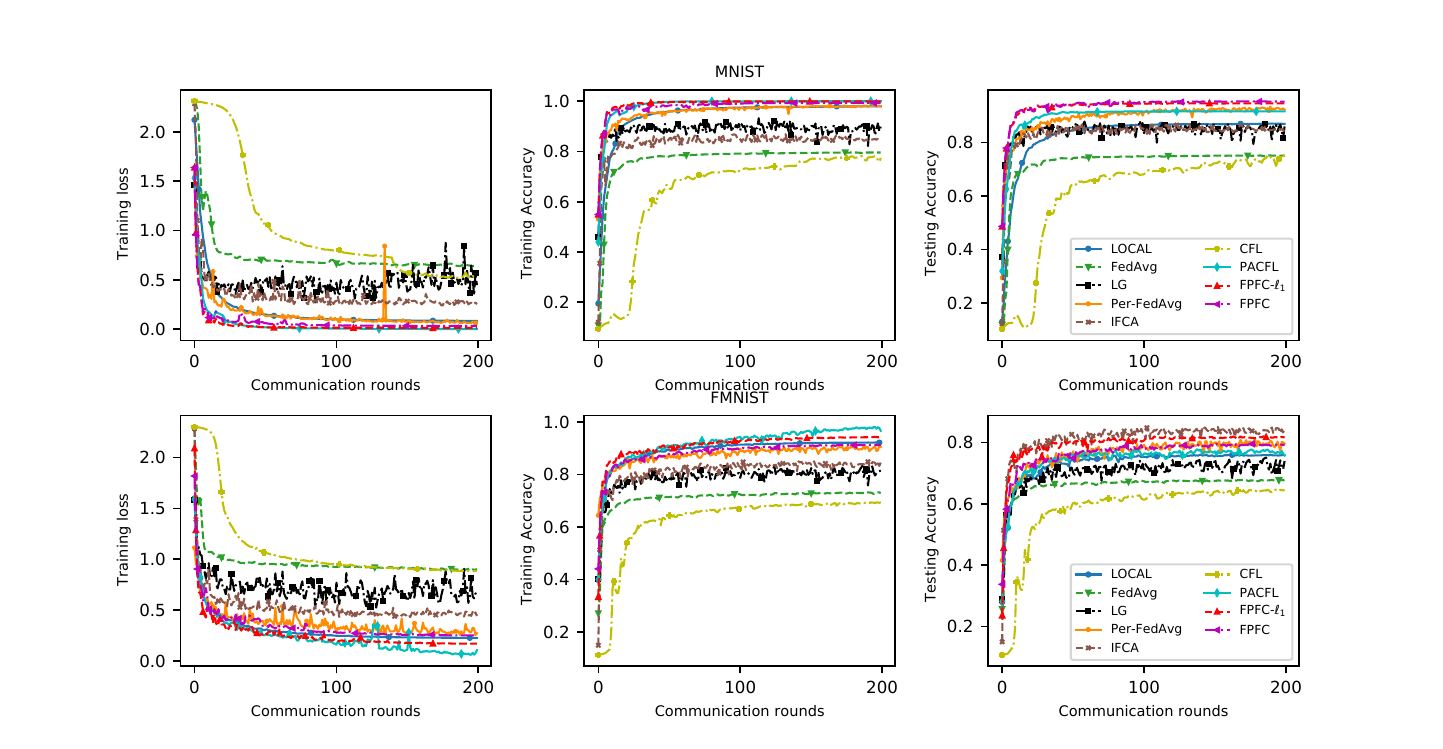}
	\vspace{-0.5cm}
	\caption{Training loss, training accuracy, and testing accuracy on MNIST and FMNIST.}
	\label{fig:mnist}
\end{figure*}

Fig.\ref{fig:group4} visualizes $\|\bm\omega_i-\bm\omega_j\|$($i,j\in [m]$) on a synthetic dataset over the course of 600 communication rounds. We use t-distributed stochastic neighbor embedding (t-SNE) \cite{Laurens2008} to visualize the cluster structure of $ m=100 $ devices. As can be seen, $\bm\omega_i$ converge to four clusters when $\lambda = 0.6$. These results illustrate that FPFC can effectively identify the cluster structure. Table \ref{table:linear} presents a statistical summary of four metrics on synthetic and H$ \& $BF datasets. Because LOCAL, FedAvg, LG, and Per-FedAvg lack the capability to cluster devices, we have omitted reporting their Num and ARI results. FedAvg exhibits the worst performance as expected, which can be attributed to its attempt to learn a single model for all devices, despite their distinct data distributions. These results highlight the limitations of FedAvg in terms of individualized predictions and cluster-based analysis. Both Per-FedAvg and LG have shown superior performance compared to FedAvg. This can be attributed to their ability to personalize the global model for each device. Due to overfitting to local data, LOCAL has worse prediction performance than FPFC. The two clustered FL algorithms, IFCA and CFL, do not perform as well as FPFC. Although they can also cluster devices, their clustering performances are inferior to that of FPFC. Although IFCA is implemented with the correct number of clusters, some clusters have no device in the end, resulting in fewer clusters. PACFL demonstrates competitive results, but it relies on a critical threshold parameter that controls the clustering results. FPFC-$\ell_{1}$ performs better than IFCA and CFL in clustering. However, due to the use of $\ell_{1}$ penalty in FPFC-$\ell_{1}$, when the coefficients of devices that belong to different clusters have significant differences, FPFC-$\ell_{1}$ imposes a stronger penalty to compress these coefficients, leading to incorrect grouping. In contrast, FPFC does not penalize large coefficient differences, and therefore the result of FPFC-$\ell_{1}$ is significantly inferior to FPFC in both clustering and prediction. LOCAL incurs zero communication costs as it does not involve any communication with the server. FedAvg, LG, Per-FedAvg, FPFC, and FPFC-$\ell_{1}$ have the same communication costs due to the consistent size of the transmitted information at each communication round. The communication cost of PACFL is determined by the combined cost of FedAvg and the cost associated with hierarchical clustering in one-shot clustering. IFCA has the highest communication cost due to its sensitivity to initialization, which needs more communication rounds to achieve convergence. Additionally, during each communication round, IFCA requires downloading all cluster models to all devices, resulting in a substantial increase in communication costs.

Fig. \ref{fig:mnist} shows the training loss, training accuracy and testing accuracy results of all methods on MNIST and FMNIST under some fine-tuned parameters. FPFC also achieves a comparable high test accuracy with only a small communication cost. Despite IFCA demonstrating favorable results on the FMNIST dataset, as shown in Fig. \ref{fig:mnist_cost}, it comes at the expense of the highest communication costs. We also consider different cluster structures, the effect of regularization parameter $ \lambda $, and the tolerance to heterogenous training, see more results in Appendix \ref{app:E}.

\begin{figure}
	\centering
	\includegraphics[width=0.4\columnwidth]{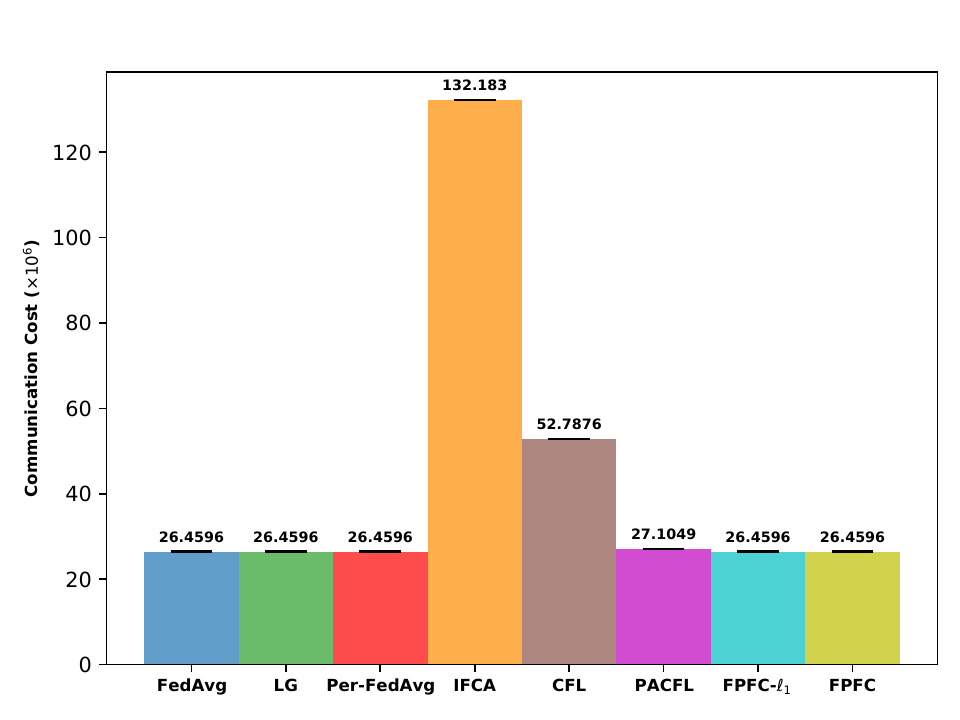}
	\caption{Total communication cost on MNIST/ FMNIST.}
	\label{fig:mnist_cost}
\end{figure}

\subsection{Regularization parameter tuning}
\label{sec:sim_warmup}
\begin{figure*}
	\centering
	\hspace{-1cm}
	\includegraphics[width=1\columnwidth]{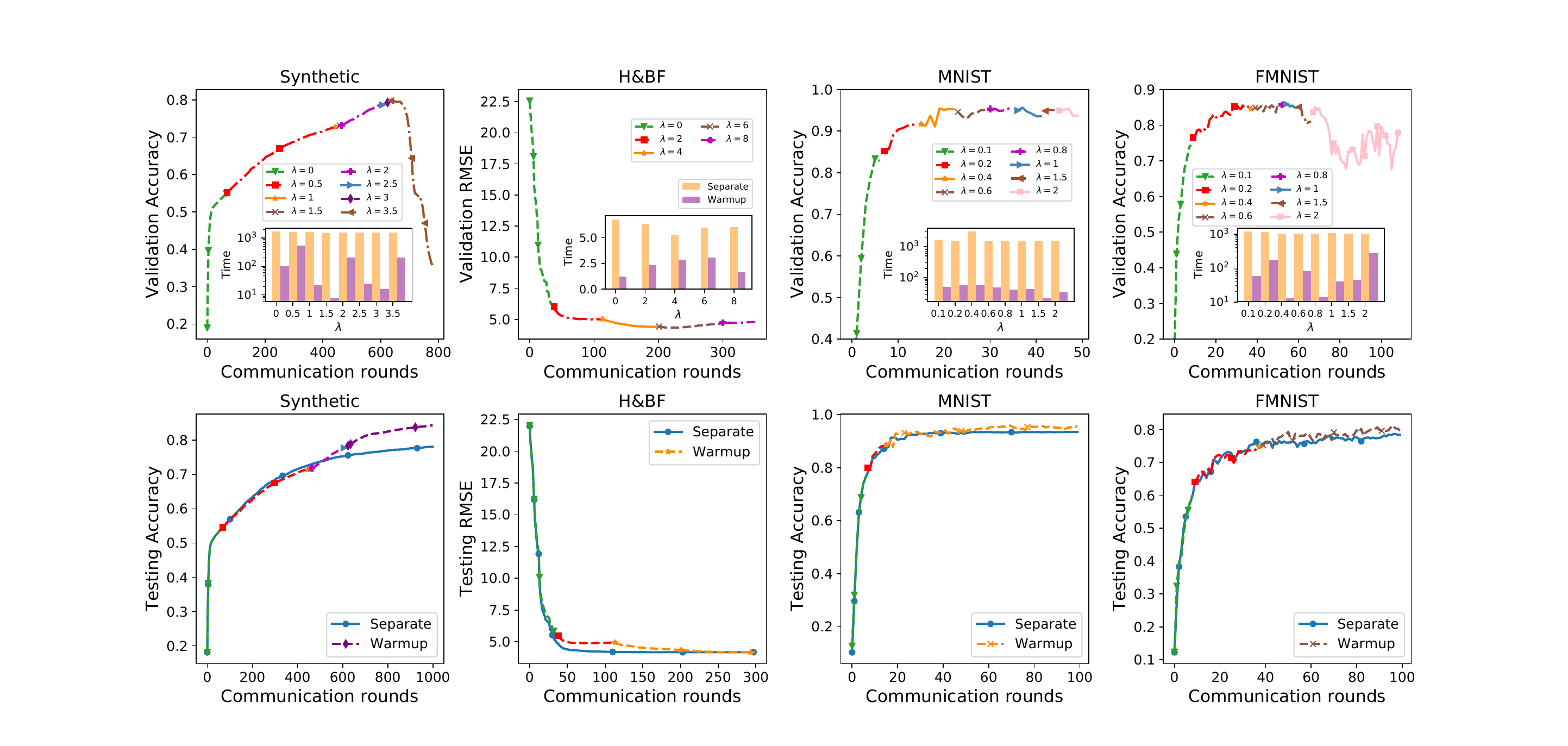}
	\vspace{-0.5cm}
	\caption{Two strategies for regularization parameter tuning on different datasets.}
	\label{fig:warmup_sepetate}
\end{figure*}

To compare different strategies for regularization parameter tuning, we conduct diverse experiments. The four subgraphs in the first row of Fig. \ref{fig:warmup_sepetate} show the process of parameter tuning via warmup. Taking the synthetic dataset as an example, we tune $ \lambda $ from $ \{0, 0.5, 1, 1.5, 2, 2.5, 3, 3.5\} $. Initially, we set $ \lambda = 0 $ and run Algorithm \ref{alg:1} until the change in accuracy on the validation set becomes smaller than a predefined tolerance threshold, e.g., $ 10^{-4} $. Afterward, we use the model obtained with $\lambda = 0$ as the initial model for $\lambda = 0.5$. We then run FPFC with $ \lambda = 0.5 $ until it converges (i.e., the change in the validation accuracy is below a specified threshold). As observed from the results, the final validation accuracy achieved with $ \lambda = 0.5 $ is higher than that obtained with the previous $ \lambda = 0 $.  Thus, we retain the model obtained with $ \lambda = 0.5 $ as the initial model for subsequent $ \lambda = 1 $. Then, we start to train FPFC with $ \lambda = 1 $ based on the initial model and repeat this process. As shown in the first subgraph in the first row of Fig. \ref{fig:warmup_sepetate}, once $ \lambda $ increases to $ 3.5 $, the validation accuracy begins to decline, indicating that $ \lambda = 3.5 $ is too large for training. Therefore, we stop increasing $ \lambda $ and revert to training FPFC using the model obtained with $ \lambda = 3 $ as the initial model until the algorithm converges. The bar charts display the running time(in seconds) of the two strategies for each $ \lambda $ during the parameter tuning process on various datasets. Notably, both strategies tune $ \lambda $ within the same range. The bar charts demonstrate that the warmup strategy which initializes the model with the latest model parameters as $ \lambda $ increases indeed speeds up convergence and saves computing time.

\begin{table*}
	\caption{Selected $ \lambda $ by two strategies and running time on synthetic, H\&BF, MNIST, and FMNIST datasets.}
	\label{table:warmup}
	\centering
	\begin{tabular}{c|cc|cc|cc|cc}
		\toprule
		& \multicolumn{2}{c|}{Synthetic} &\multicolumn{2}{c|}{ H\&BF} &\multicolumn{2}{c|}{MNIST} &\multicolumn{2}{c}{FMNIST} \\
		\midrule 
		&$ \lambda $ &Time &$ \lambda $ &Time &$ \lambda $ &Time &$ \lambda $ &Time \\
		\midrule
		Separate   &0.5& 4.02h  &4 & 39.41s & 0.1 & 3.94h &0.1 &2.80h \\
		Warmup   &3 &\textbf{0.46h} &4 &\textbf{14.74s} & 0.4 & \textbf{0.58h}& 0.6 & \textbf{0.59h}\\
		\bottomrule
	\end{tabular}
\end{table*}

The four subgraphs in the second row of Fig. \ref{fig:warmup_sepetate} show the generalization performance of the two strategies. For FPFC with separate tuning, the initial values corresponding to each  $ \lambda_i $ are randomly generated. Subsequently, we run FPFC using $ \lambda_i $ on the training set and assess its performance (accuracy or RMSE) on the validation set. Finally, we compute the testing accuracy/RMSE of FPFC using the $ \lambda $ that shows the best performance on the validation set. Take the synthetic dataset as an example, the testing accuracies during 0-68th communication rounds are obtained based on the models with $ \lambda = 0 $. The 69th-443th, 444th-458th, 459th-463th, 464th-605th, 606th-623th, 624th-634th, 635th-780th testing results are based on models with $\lambda =0.5, 1, 1.5, 2, 2.5, 3, 3.5$, respectively. Due to the lower validation accuracy observed with $\lambda =3.5$ compared to $\lambda =3$, we finally run the latest model trained with $ \lambda = 3 $ from the 635th communication round and continue running it until the maximum number of iterations is reached. During the warmup stage, the initial values associated with the subsequent $ \lambda $ values gradually converge towards local minimizers. As a result, using better initial values can potentially yield improved results. The final selected regularization parameter and total running time are reported in Table \ref{table:warmup}. Utilizing the warmup scheme significantly reduces the running time, for example, reducing it to 0.58 hours with warmup compared to 3.94 hours when tuning parameters separately on the MNIST dataset.

\subsection{Practical considerations}
In this part, we explore the performance of FPFC with all the mentioned state-of-the-art baselines, focusing on robustness and generality. Additionally, we introduce several practical variants of FPFC, including asynchronous updates (asyncFPFC) and a communication-efficient approach. Through these evaluations, we aim to gain valuable insights into the strengths and limitations of FPFC in practical application scenarios.

\subsubsection{Robustness}
One of the most common issues in FL is that it can suffer from security threats to model training, including malicious devices, communication failures, and adversarial attacks. A significant  challenge to the robustness of distributed systems is known as Byzantine failure or Byzantine fault \cite{Lamport1982}. Byzantine attacks can occur when participating devices provide completely arbitrary or misleading behaviors to the server, thereby disrupting the training process. To address Byzantine failures in FL, several approaches have been proposed to enhance the robustness of FL by designing a Byzantine-robust FL scheme \cite{Blanchard2017,Huang2021,Miao2022,Portnoy2022}.

To evaluate the robustness of different algorithms, similar to the work of \cite{Lin2022}, we consider three different types of Byzantine model update attacks in FL training. Define the model update sent by a malicious device $ k $ is $\breve{\omega}_{k} $, then the three types of attacks can be defined as：
\begin{itemize}
	\item[(a)] Same-value attack: the malicious device $k$ sends the constant value $\breve{\omega}_{k} = c \bm{I} $ to the server as its model update, where $ \bm{I} \in \mathbb{R}^{d} $ is the vector of ones and $ c \sim \mathcal{N}(0, \sigma^2) $.
	\item[(b)] Sign-flipping attack: the malicious device $k$ alters the sign of its model update before sending it to the server. $\breve{\omega}_{k} = -|c|\bm{\omega}_k$, where $ c \sim \mathcal{N}(0, \sigma^2) $.
	\item[(c)] Gaussian attack: the malicious device $k$ sends gaussian noise $\breve{\omega}_{k} \sim \mathcal{N}(\bm{0}, \sigma^2 \bm{I})$ to the server. 
\end{itemize}
\begin{figure*}
	\centering
	\includegraphics[width=1\columnwidth]{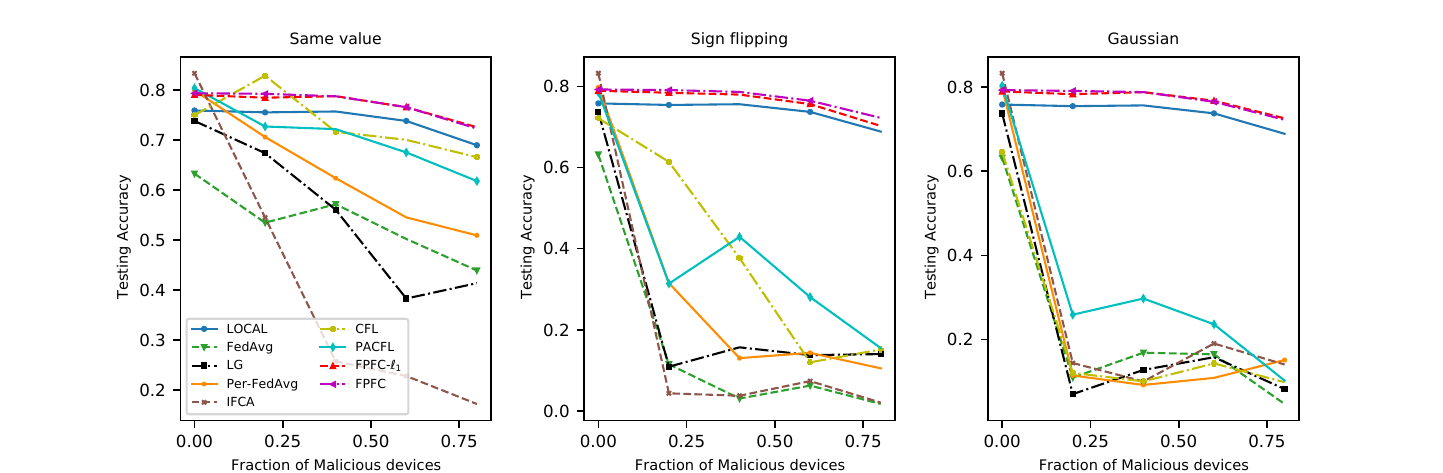}
	\caption{Robustness comparison of different methods on FMNIST dataset.}
	\label{fig:fmnist_attack}
\end{figure*}
To evaluate the robustness of the personalization layers in the model, we conduct experiments involving attacks with varying levels of noise. For the three Byzantine attacks, the noise level $ \sigma $ is set as $ 100, 10, 100 $ respectively, as specified in \cite{Lin2022}. We consider different ratios of malicious devices, ranging from 0.0 to 0.8 with increments of 0.2. Fig.\ref{fig:fmnist_attack} shows the average test accuracy of benign devices on the FMNIST dataset for different methods under the three Byzantine attacks. As we can see, FPFC consistently outperforms other baselines and it is the most stable algorithm across different attack scenarios. For full results on synthetic and MNIST datasets, please refer to Appendix \ref{app:E}. 

The robustness of FPFC can be attributed to the utilization of the nonconvex pairwise fusion penalization. According to the formulation of the SCAD penalty $P_{a}(t,\lambda)$ \eqref{eq:SCAD}, $P_{a}(t,\lambda)$ will become a constant when $ |t| > a\lambda $, resulting in the ineffective shrinkage of parameter values. The difference between malicious updates and normal updates in FPFC is typically greater than the predefined threshold $ a\lambda $ or can be controlled by appropriately selecting the parameters $ \lambda $ and $ a $. This characteristic limits the impact of attacks on the training process of FPFC. As a result, FPFC maintains its robustness and reliability under different attacks.

\subsubsection{Generalization to newcomers}

\begin{table}
	\caption{Average local test accuracy and standard deviation across newcomers on different datasets. The last column represents the average ranking of all algorithms across three datasets.}
	\label{table:generalization}
	\centering
	\begin{tabular}{c|cccc}
		\toprule
		& Synthetic & MNIST &	FMNIST  & Rank \\
		\midrule 
		LOCAL    &79.00\% $\pm$ 0.08 & 90.81\% $ \pm $ 0.06 & 75.29\% $\pm$ 0.06 &4.67 \\
		FedAvg   & 38.73\% $\pm$ 0.29 & 73.67\% $\pm$ 0.04 & 66.26\% $\pm$ 0.06 &8.67 \\
		LG   & 66.22\% $\pm$ 0.11  & 95.71\% $\pm$ 0.03 &\textbf{80.25\% $\pm$ 0.07} &3.33 \\
		Per-FedAvg   & 43.26\% $\pm$ 0.26  & 96.62\% $\pm$ 0.01 & 77.94\% $\pm$ 0.06  &3.67 \\
		IFCA    & 33.57\% $\pm$ 0.28 & \textbf{96.65\% $\pm $ 0.01}  & 73.93\% $\pm$ 0.10  &5.33  \\
		CFL     &69.43\% $\pm$ 0.33 & 95.46\% $\pm$ 0.02  & 70.66\% $\pm$ 0.07  & 5.00\\
		PACFL   &71.88\% $\pm$ 0.30 & 74.45\% $\pm$ 0.05   & 66.42\% $\pm$ 0.06 & 6.33\\
		FPFC-$\ell_{1}$  & 67.72\% $\pm$ 0.10  & 92.73\% $\pm$ 0.05 & 75.54\% $\pm$ 0.08  & 5.00\\
		FPFC    & \textbf{79.31\% $\pm$ 0.08}  &   93.51\% $\pm$ 0.06  & 75.73\% $\pm$ 0.07  &\textbf{3.00} \\
		\bottomrule
	\end{tabular}
\end{table}

When newcomers join after the federation procedure, FPFC can leverage the models of other devices to help newcomers learn their personalized models. The newcomer $ i $ conducts local training first and then sends its model to the server and the server computes the values of $ \theta_{ij} $ and $ v_{ij} $ between newcomer $i$ and other previously participating devices $ j $. The server then sends the corresponding $ \zeta_{i} $ to the newcomer $ i $ and the newcomer updates its local model based on the received $ \zeta_{i} $. This process is repeated until convergence is achieved. As not all baselines have provided official approaches for handling newcomers after training, we developed the following strategies for each method. For the LOCAL method, each newcomer updates its model using its local data. For FedAvg, newcomers utilize the final global model obtained from the server. For LG and Per-FedAvg, newcomers receive the final global model from the server and fine-tune it. For IFCA, each newcomer determines its cluster identity by selecting the cluster with the lowest loss. For CFL and PACFL, we adopt the official approaches described in their papers to obtain personalized models for newcomers. 

To evaluate the performance of newcomers' models, we consider three scenarios where 20\% of the participants join the federated network after the federated training process. Table \ref{table:generalization} reports the average local test accuracy and standard deviation of the newcomers on three different datasets. As observed, FedAvg has the worst performance as the global model is not well-suited for all devices. The two personalized baselines, LG and Per-FedAvg, outperform FedAvg as they enable newcomers to fine-tune the global model based on their local data. Compared to other baselines, FPFC has comparable performance and  demonstrates good generalization capabilities on various datasets.
\begin{figure}[ht]
	\centering
	\vspace{-0.5cm}
	\includegraphics[width=0.8\columnwidth]{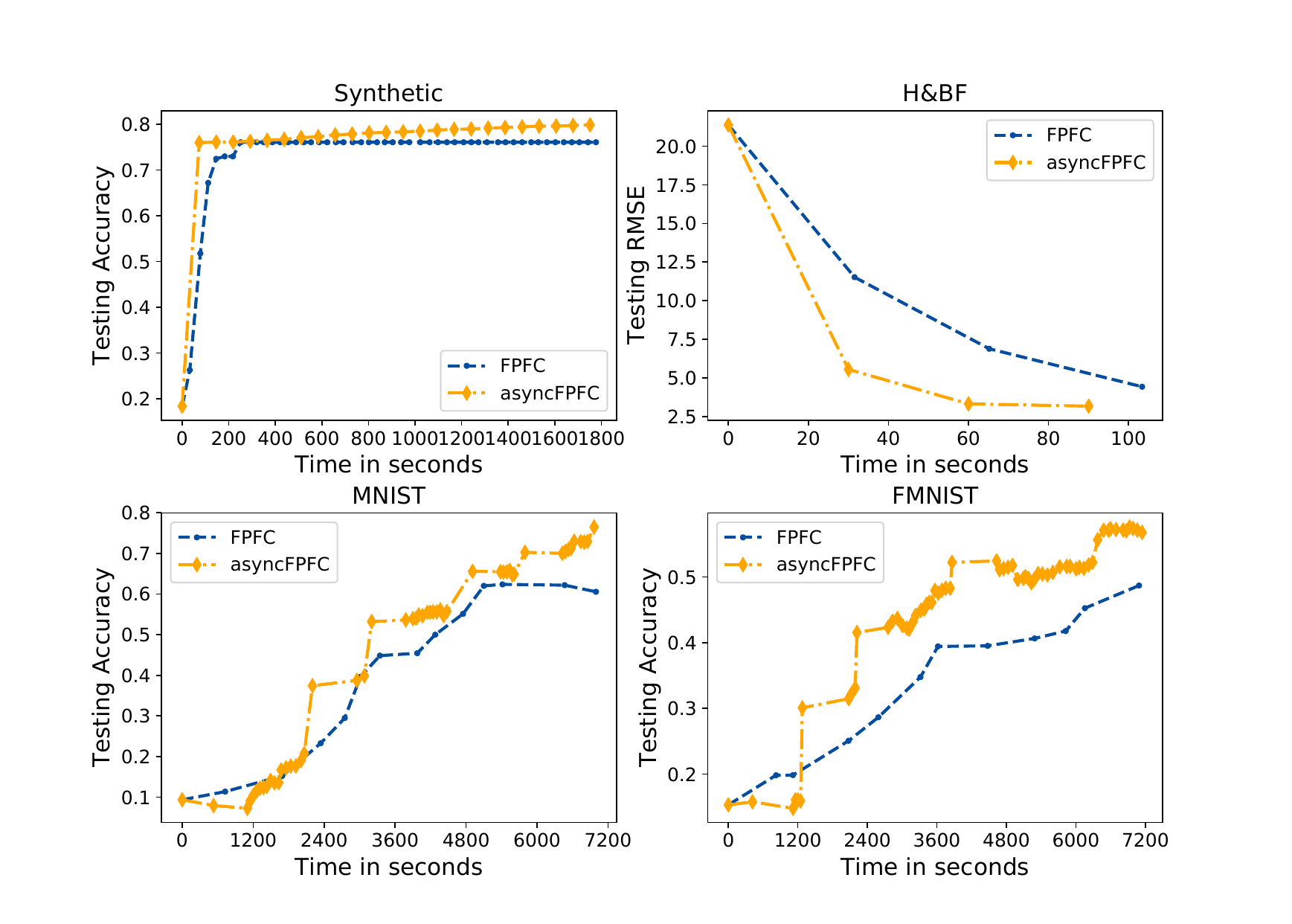}
	\vspace{-0.4cm}
	\caption{Synchronous versus asynchronous update.}
	\label{fig:async}
\end{figure}
\subsubsection{AsyncFPFC}
Note that in {FPFC}, the server needs to select a subset of devices and requires synchronous aggregation at each communication round. However, such algorithms may suffer from synchronization delays due to slower devices, especially in the FL setting. Asynchronous methods have been studied widely due to their advantages over synchronized methods \cite{Quoc2021, Chaturapruek2015, Assran2020, Hamid2021}. They do not need a global synchronous aggregation at each round and the server can perform updates once it receives the information from one device. Although the implementation and debugging of asynchronous algorithms are difficult, such methods can make more efficient use of computing resources, leading to faster convergence. In this part, we propose {asyncFPFC}, an asynchronous variant of {FPFC}. The detailed implementation of asyncFPFC is provided in Appendix A.

To illustrate that our algorithm can employ the asynchronous update strategy and the advantage of {asyncFPFC} over {FPFC}, we conduct diverse experiments using synthetic datasets and three real datasets. The number of computing nodes is 8 for H\&BF data and 20 for synthetic, MNIST, and FMNIST datasets. For {FPFC}, we randomly sample $ 3 $ nodes to perform the synchronous aggregation for H\&BF dataset, and 5 nodes for synthetic, MNIST, and FMNIST. Since the nodes on a computing cloud have similar configurations, we follow the procedure used by \cite{Quoc2021}, where they simulate the devices' heterogeneity of computation and communication by adding variable delays to each node’s update process, such that some nodes have a slower speed to update and upload its models to the server. Empirically, we set the computing node to have a random delay in the range [0, 20] seconds for H\&BF and synthetic, and [0, 800] seconds for MNIST and FMNIST. The results are depicted in Fig. \ref{fig:async}. As it can be seen, {asyncFPFC} outperforms {FPFC} in terms of training time when the computing power of the nodes is heterogeneous.

\subsubsection{Communication-efficient FPFC}

Communication is the key challenge in FL due to the limited bandwidth and device instability. Several different approaches are proposed to address this issue by increasing the computation over communication rounds, such as Local SGD \cite{Dieuleveut2019, Spiridonoff2021}. However, taking many steps to update each device’s local model will lead to immediate local over-fitting. A heuristic communication strategy is to gradually increase the number of local updates. Intuitively, at the beginning of training, devices do not have any prior information about the data distribution, so frequent communication among devices is required to collaboratively train models. As each device updates its local model towards its optima, less communication is required. 
\begin{figure}[h]
	\centering
	\vspace{-0.5cm}
	\includegraphics[width=0.8\columnwidth]{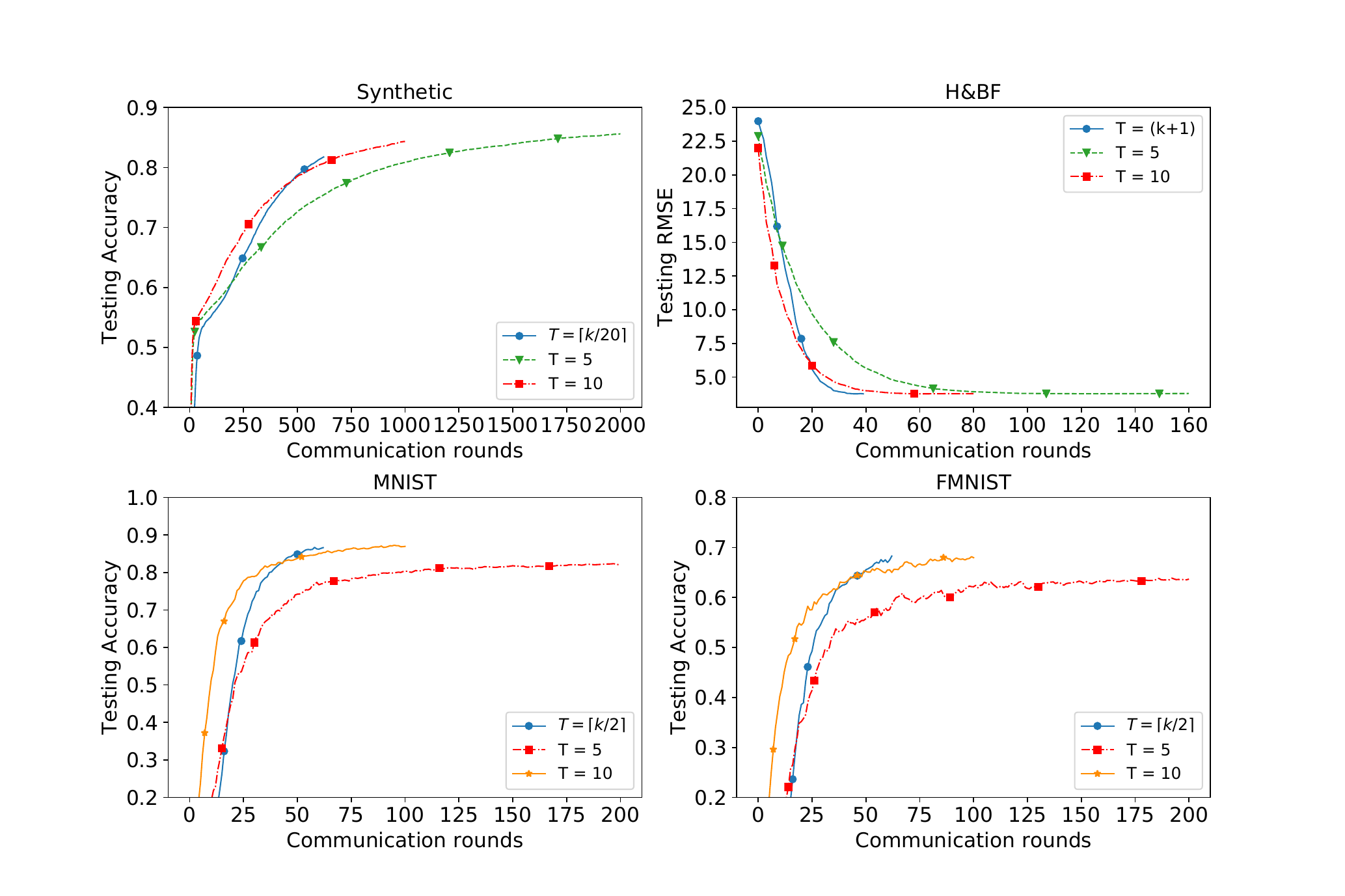}
	\vspace{-0.4cm}
	\caption{Running FPFC with different communication strategies.}
	\label{fig:vary}
\end{figure}

Following the procedure in \cite{Spiridonoff2021}, we compare the performance of FPFC using different communication strategies. For a fair comparison, we have fixed the total number of local iterations $ K \times T $ to a constant value, where $ K $ is the communication rounds and $ T $ is the number of local updates. Specifically, for synthetic and H$ \& $BF datasets, we run FPFC with local GD updates, and the total iterations $ K \times T $ is $ 10^{4} $ for the synthetic dataset, 800 for H$ \& $BF. We compare different strategies, namely, local GD with a constant step $ T = 5 $, local GD with $ T = 10 $, and local GD with a growing step, $ T = \lceil k/20 \rceil$ for synthetic, $ T = k+1 $ for H$ \& $BF. For MNIST and FMNIST, we run FPFC with local mini-batch SGD updates, and the total iterations $ K \times T $ is $ 10^3 $. We compare local SGD with a constant step $ T = 5 $, $ T = 10 $, and local SGD with a growing step, $ T = \lceil k/2 \rceil $. When the total number of iterations $ K \times T $ is fixed, conducting a different number of local updates $ T $ will result in different communication rounds. For example, for the MNIST dataset with  $ K \times T = 10^3 $, applying different strategies such as $ T = 5 $, $ T = 10 $, and $ T = \lceil k/2 \rceil $ will lead to communication rounds of 200, 100, and 63, respectively. Fig. \ref{fig:vary} shows that the communication strategy with increasing local steps outperforms all the other methods. For MNIST, FPFC with a growing local step can achieve $ 86.62\% $ accuracy after $ 62 $ communication rounds, whereas {FPFC} with $ T = 10 $ reaches $ 85.31\% $, and {FPFC} with $ T = 5 $ only attains $ 77.05\% $ after the same rounds. For FMNIST, after $ 62 $ communication rounds, the testing accuracy obtained by {FPFC} with a growing local step is $ 68.28\% $, for {FPFC} with $ T = 10 $ and $ T = 5 $, the accuracies are $ 65.62\% $ and $ 59.34\% $, respectively. We conclude that frequent communication at the beginning is beneficial to reduce the communication rounds. 

\section{Conclusion}
In this paper, we introduced a novel clustered FL method called FPFC to address the challenges of FL where devices are distributed and partitioned into clusters. By employing a nonconvex pairwise fusion penalty, FPFC can identify the cluster structure and estimate the optimal models in each cluster. We have provided convergence guarantees for FPFC with nonconvex objectives and established the statistical convergence rate under a linear model. Additionally, we proposed a warmup strategy for regularization parameter tuning in FL settings. Experimental results have demonstrated the advantages of FPFC, including its superior performance in terms of prediction accuracy, clustering capability, robustness to attacks, generalization ability, and communication efficiency. However, it is worth noting that FPFC has a key limitation in terms of its slow convergence rate due to the utilization of the ADMM scheme in FL. Future directions involve integrating state-of-the-art privacy-preserving methods and analyzing the convergence results of FPFC with the warmup tuning strategy and asyncFPFC algorithm.

\bibliographystyle{plain}
\bibliography{main}

\appendix
\onecolumn

\part*{\Large Appendix\\Clustered Federated Learning based on Nonconvex Pairwise Fusion}

\section{Details of Algorithm 1}
\setcounter{equation}{0}
\renewcommand\theequation{A.\arabic{equation}}

\label{app:A}
We first reformulate (4) into  a constrained problem. Next, we apply the classical DR splitting scheme to the constrained reformulation and present the standard ADMM algorithm. Then, we revise the $\omega$-step in a distributed and inexact manner. Finally, we randomize its local update step to obtain the final algorithm. \\

\noindent{\bf Constrained reformulation.} 
By introducing a set of new parameters $\bm\theta_{ij}=\bm\omega_i-\bm\omega_j$, we can rewrite (4) as the following constrained optimization problem:
\begin{equation*}
\begin{aligned}
&\min_{\bm\omega,\bm\theta} \sum_{i=1}^m f_i(\bm\omega_i)+\frac{1}{2m}\sum_{i=1}^m\sum_{j=1}^m \tilde{g}(\|\theta_{ij}\|)\\
&\text{s.t.}\ \bm\omega_i-\bm\omega_j = \bm \theta_{ij},\ i,j\in[m].	
\end{aligned}
\end{equation*}
\noindent{\bf Standard ADMM algorithm.} 
The solution of the aforementioned problem can be obtained by minimizing the following augmented Lagrangian:
\begin{equation}
\begin{aligned}
\tilde{\mathcal{L}}_{\rho}(\bm\omega,\bm{\theta}, v)&= f(\bm{\omega})+ \frac{1}{2m}\sum_{i=1}^m\sum_{j=1}^m \big[\tilde{g}(\|\theta_{ij}\|) +\langle v_{ij},\bm\omega_i-\bm\omega_j-\theta_{ij}\rangle +\frac{\rho}{2}\|\bm\omega_i-\bm\omega_j-\theta_{ij}\|^2\big],
\end{aligned}	
\end{equation}
where $f(\bm{\omega}) = \sum_{i=1}^mf_i(\bm\omega_i) $, $\bm\theta=\{\bm\theta_{ij}^\top, i,j\in [m]\}^\top$, $v=\{v_{ij}^\top, i,j \in [m]\}^\top$, and $\rho>0$ is the penalty parameter. We apply the DR splitting method adopted by ADMM to \eqref{eq:lag} to obtain the following iterations:
\begin{small}
	\begin{align}
	\bm\omega^{k+1} &=\arg\min_{\bm\omega}\sum_{i=1}^m f_i(\bm\omega_i)+\frac{\rho}{2m}\sum_{i=1}^m\sum_{j=1}^m \|\bm\omega_i-\bm\omega_j-\bm\theta_{ij}^k+\frac{v_{ij}^k}{\rho}\|^2\nonumber\\
	\bm{\theta}_{ij}^{k+1} &= \arg\min\limits_{\bm{\theta}_{ij}} \tilde{g}(\|\theta_{ij}\|) + \frac{\rho}{2} \|\bm{\omega}_i^{k+1} - \bm{\omega}_j^{k+1}+\frac{v_{ij}^{k}}{\rho} - \bm{\theta}_{ij}\|^2 \label{app:theta}\\
	v_{ij}^{k+1} &= v_{ij}^{k} + \rho( \bm{\omega}_i^{k+1} - \bm{\omega}_j^{k+1}-\bm{\theta}_{ij}^{k+1}), \label{eq:v}
	\end{align} 
\end{small}
where the second and third steps are conducted for all $i<j$. For $i>j$, the updated quantities can be directly obtained via $\bm\theta_{ij}=-\bm\theta_{ji}$ and $ v_{ij}=- v_{ji}$, and for $i=j$, $\bm\theta_{ii}= v_{ii}= 0$. \\

\noindent{\bf Distributed and inexact update of $\omega$.}
In practice, the $\bm\omega$-update step, can be conducted in a distributed manner, i.e., $\bm\omega^{k+1}$ can be obtained by running the following sub-iteration for sufficiently many steps, for $c=0,\ldots, C-1 $:
\begin{equation}
\label{eq:subiteration}
\small
\bm\omega_i^{k,c+1}=\arg\min_{\bm \omega_i} f_i(\bm\omega_i)+\frac{\rho}{2m}\sum_{j=1}^m \|\bm\omega_i-\bm\omega_j^{k,c}-\bm\theta_{ij}^k+\frac{v_{ij}^k}{\rho}\|^2, \ i\in [m]
\end{equation} 
where $\bm\omega_i^{k,0}=\bm\omega_i^k$ and $\bm\omega_i^{k+1}\approx \bm\omega_i^{k,C}$ for $i\in [m]$. However, this parallel approach has three main drawbacks in the FL setting. First, the central server has to send all model parameters to all devices at each sub-iteration, resulting in large communication cost, especially when the number of devices or/and sub-iteration rounds is large. Second, each step must solve a minimization problem exactly, which is unpractical in the FL setup due to the heterogeneous computational power of devices. Lastly, each device must know the current model parameters of all other devices, which may result in privacy leaks. 

As such, to reduce communication cost, we only conduct sub-iteration \eqref{eq:subiteration} for one step, i.e., $C=1$, instead of multiple steps. Considering the system heterogeneity of a federated network, we allow the minimization $\eqref{eq:subiteration}$ to be solved approximately. Specifically, we define $h_i(\omega_i)=f_i(\omega_i)+\frac{\rho}{2m}\sum_{j=1}^m\|\omega_i-\omega_j^{k}-\theta^k_{ij}+\frac{v_{ij}^k}{\rho}\|^2$. We replace the exact
solution $\hat{\omega}_i^{k+1}= \arg\min\limits_{{\omega_i}} h_i(\omega_i)$ by an approximated solution $\omega_i^{k+1}\approx h_i(\omega_i)$ up to a given accuracy $\epsilon_i$ such that $\|\omega_i^{k+1}-\hat{\omega}_i^{k+1}\|\leq \epsilon \|\omega_i^{k+1}-\omega_i^k\|$. The approximated solution $\omega_i^{k+1}$ can be obtained by running $T_i$ epochs of local gradient decent updates 
\begin{equation} 
\small
\label{eq:local update1}
\bm{\omega}_{i}^{k,t+1} =  \bm{\omega}_{i}^{k,t} - \alpha[\nabla f_i(\bm{\omega}_{i}^{k,t})+ \rho( \bm{\omega}_{i}^{k,t} - \bm{\zeta}_{i}^{k})], t=0,\ldots, T_i-1,
\end{equation}
where $\alpha$ is the step size, $\bm\zeta_i^{k}=\frac{1}{m}\sum_{j=1}^m (\bm\omega_j^k+\bm\theta_{ij}^k-\frac{v_{ij}^k}{\rho})$ and $\bm\omega_i^{k,0}=\bm\omega_i^k$. We set $\bm \omega^{k+1}=\bm\omega^{k,T_i}$. Because $\bm\zeta_i^{k}$ is an average of parameters, the $i$-th device cannot infer the model parameters or private data of any other devices, which preserves the privacy of devices. In practice, $\nabla f_i(\bm{\omega}_{i}^{k,t})$ can be full or stochastic gradient of the local empirical loss $f_i$ at $\bm{\omega}_{i}^{k,t}$. 

Putting \eqref{app:theta}, \eqref{eq:v}, and  \eqref{eq:local update1} together, we obtain the following distributed and inexact ADMM variant:
\begin{equation}
\small
\begin{aligned}
\begin{cases}
\label{eq:parallel}
&\bm\omega_i^{k,0} =\bm\omega_i^k, \forall i\in [m]\\
&\bm{\omega}_{i}^{k,t+1} =  \bm{\omega}_{i}^{k,t} - \alpha[\nabla f_i(\bm{\omega}_{i}^{k,t})+ \rho( \bm{\omega}_{i}^{k,t} - \bm{\zeta}_{i}^{k})], t=0,\ldots, T_i-1,\\
&\bm\omega_i^{k+1} =\bm\omega_i^{k,T_i}, \forall i\in [m]\\
&\bm{\theta}_{ij}^{k+1} = \arg\min\limits_{\bm{\theta}_{ij}} \tilde{g}(\|\theta_{ij}\|) + \frac{\rho}{2} \|\bm{\omega}_i^{k+1} - \bm{\omega}_j^{k+1}+\frac{v_{ij}^{k}}{\rho} - \bm{\theta}_{ij}\|^2, \\
&v_{ij}^{k+1} = v_{ij}^{k} + \rho( \bm{\omega}_i^{k+1} - \bm{\omega}_j^{k+1}-\bm{\theta}_{ij}^{k+1}), 
\end{cases}
\end{aligned}
\end{equation}
and $\bm\theta_{ji}^{k+1}=-\bm\theta_{ij}^{k+1}, v_{ji}^{k+1}=- v_{ij}^{k+1}, \forall i<j  $. \\

\noindent{\bf Randomized-block update.} Instead of performing updates for all devices as in \eqref{eq:parallel}, we adopt a randomized-block strategy, where only a subset of devices $\mathcal{A}_k$ perform local updates then send their local models to server.  
Then, for the inactive devices $i\notin \mathcal{A}_k$, the local model is unchanged, i.e., $\omega_i^{k+1}=\omega_i^k$. Hence, there is no need for communication between these inactive devices and the server. 
Moreover, for any two inactive devices $i$ and $j$, the parameters $\theta_{ij}$ and $v_{ij}$ are also unchanged, i.e., $\theta_{ij}^{k+1}=\theta_{ij}^k$ and $v_{ij}^{k+1}=v_{ij}^k$. As a result, the server only needs to update the parameters between two devices with at least one active device. \\

\noindent{\bf Full version of Algorithm 1.} Finally, we obtain the full algorithm of {FPFC} (Algorithm \ref{alg:2}). For completeness, we also provide the algorithm of {FPFC-$\ell_{1}$}. \\
\begin{algorithm}
	\caption{FPFC and FPFC-$\ell_{1}$}
	\label{alg:2}
	\begin{algorithmic}
		\STATE {\bfseries Input:} number of devices $m$; number of communication rounds $K$; number of local gradient epochs $T_i$, $i\in [m]$; stepsize $\alpha$; parameters of $\tilde{g}(\cdot)$: $\xi$, $\lambda$, $a$; and $\rho$.
		\STATE {\bfseries Output:} $ \{\bm\omega_{i}^{K} \}_{i \in [m]} $.
		\STATE Initialize each device $i\in [m]$ with $\bm\omega_1^0=\ldots=\bm\omega_m^0$.  
		Initialize the server with $\bm\zeta_i^0=\bm\omega_i^0$ for $i\in[m]$, $\bm\theta_{ij}^0= 0$ and $v_{ij}^0= 0$ for $i,j\in [m]$.  
		\FOR{$k=0$ {\bfseries to} $K-1$}
		\STATE 1: \textbf{[Active devices]} Server randomly selects a subset of devices $\mathcal{A}_{k}$.
		\STATE 2: \textbf{[Communication]} Server sends $ \bm{\zeta}_{i}^{k}$ to each device $i\in \mathcal{A}_k$.
		\STATE 3: \textbf{[Local update]} For each device $ i \in \mathcal{A}_{k}$: 
		\item
		\begin{small}
			$\bm\omega_i^{k,0}= \bm\omega_i^{k}$,
		\end{small}
		\FOR{$ t = 0$ {\bfseries to} $T_i-1 $}
		\item 
		\vspace{-4mm}
		\begin{small}
			\begin{equation*}
			\bm{\omega}_{i}^{k,t+1} =  \bm{\omega}_{i}^{k,t} - \alpha[\nabla f_i(\bm{\omega}_{i}^{k,t})+ \rho( \bm{\omega}_{i}^{k,t} - \bm{\zeta}_{i}^{k})]
			\end{equation*}
		\end{small}
		\ENDFOR
		\item
		\begin{small}
			$\bm\omega_i^{k+1} = \bm\omega_i^{k,T_i}$.
		\end{small}\\
		Each device $i\notin \mathcal{A}_k$ does nothing, i.e., $\bm\omega_i^{k+1}=\bm\omega_i^k$. 
		\STATE 4: \textbf{[Communication]} Each device $i\in \mathcal{A}_k$ sends $ \bm{\omega}_{i}^{k+1} $ back to the server.
		\STATE 5: \textbf{[Server update]} For $ i \in \mathcal{A}_{k} $ or $ j \in \mathcal{A}_{k} $ ($ i<j $):
		\item Server computes $\bm{\delta}_{ij}^{k+1} = \bm{\omega}_{i}^{k+1} - \bm{\omega}_{j}^{k+1} +\frac{v_{ij}^{k}}{\rho}$ and updates
		\item  {FPFC}:  
		\begin{equation*}
		\small
		\hspace{-4mm}
		\bm{\theta}_{ij}^{k+1} = \begin{cases} 
		\frac{\xi \rho}{\lambda + \xi\rho} \bm{\delta}_{ij}^{k+1}, &\|\bm{\delta}^{k+1}_{ij}\|  \leq \xi + \frac{\lambda}{\rho} \\
		(1 - \frac{\lambda}{\rho \|\bm{\delta}_{ij}^{k+1}\|}) \bm{\delta}_{ij}^{k+1}, & \xi + \frac{\lambda}{\rho} \leq  \|\bm{\delta}_{ij}^{k+1}\| \leq \lambda + \frac{\lambda}{\rho} \\  
		\frac{\max\{0,1 - \frac{a \lambda}{(a-1)\rho \|\bm{\delta}_{ij}^{k+1}\|}\}}{1-1/[(a-1)\rho]}\bm{\delta}_{ij}^{k+1}, & \lambda + \frac{\lambda}{\rho} < \|\bm{\delta}_{ij}^{k+1}\| \leq  a\lambda \\ \bm{\delta}_{ij}^{k+1}, & \|\bm{\delta}_{ij}^{k+1}\| > a \lambda. \end{cases}
		\end{equation*}
		\item  {FPFC-$\ell_{1}$}:  
		\begin{equation*}
		\small
		\hspace{-4mm}
		\bm{\theta}_{ij}^{k+1} = \max\{0,1 - \frac{\lambda}{\rho \|\bm{\delta}_{ij}^{k+1}\|}\}\bm{\delta}_{ij}^{k+1}.
		\end{equation*}
		
		\item  $v_{ij}^{k+1} = v_{ij}^{k} + \rho(\bm{\omega}_{i}^{k+1} - \bm{\omega}_{j}^{k+1} - \bm{\theta}_{ij}^{k+1}).$
		\item   $ \bm{\theta}_{ji}^{k+1} = -\bm{\theta}_{ij}^{k+1} $, \ $ v_{ji}^{k+1} =- v_{ij}^{k+1} $.
		\item For $ i \notin \mathcal{A}_{k}  $ and $ j \notin \mathcal{A}_{k} $: $ \bm{\theta}_{ij}^{k+1} = \bm{\theta}_{ij}^{k}$, $ v_{ij}^{k+1} =v_{ij}^k $.
		\item For $ i \in [m] $, server updates $\bm\zeta_i^{k+1}=\frac{1}{m}\sum_{j=1}^m(\bm\omega_j^{k+1}+\bm\theta_{ij}^{k+1}-\frac{v_{ij}^{k+1}}{\rho})$.
		\ENDFOR
	\end{algorithmic}
\end{algorithm}

\noindent{\bf Algorithm of asyncFPFC.}
\label{app:A.3}
Here, we provide the full algorithm of {asyncFPFC} (Algorithm \ref{alg:3}). 

\begin{algorithm}[htb]
	\caption{Asynchronous FPFC (asyncFPFC)}
	\label{alg:3}
	\begin{algorithmic}
		\STATE {\bfseries Input:} number of devices $m$; number of communication rounds $K$; number of local gradient epochs $T_i$, $i\in [m]$; stepsize $\alpha$; parameters of $\tilde{g}(\cdot)$: $\xi$, $\lambda$, $a$; and $\rho$.
		\STATE {\bfseries Output:} $ \{\bm\omega_{i}^{K} \}_{i \in [m]} $.
		\STATE Initialize each device $i\in [m]$ with $\bm\omega_1^0=\ldots=\bm\omega_m^0$.  
		Initialize the server with $\bm\zeta_i^0=\bm\omega_i^0$ for $i\in[m]$, $\bm\theta_{ij}^0= 0$ and $v_{ij}^0= 0$ for $i,j\in [m]$.  
		\FOR{$k=0$ {\bfseries to} $K-1$}
		\STATE 1: \textbf{[Communication]} Device $i_{k}$ completes its local update and then sends $ \bm{\omega}_{i_{k}}^{k+1} $ back to the server. 	Other devices maintain $\bm\omega_i^{k+1}=\bm\omega_i^k$ for $ i \ne i_{k} $. 
		\STATE 2: \textbf{[Server update]} Server computes $\bm{\delta}_{i_{k}j}^{k+1} = \bm{\omega}_{i_{k}}^{k+1} - \bm{\omega}_{j}^{k+1} +\frac{v_{i_{k}j}^{k}}{\rho}$ and updates
		\begin{equation*}
		\hspace{-9mm}
		\bm{\theta}_{i_{k}j}^{k+1} = \begin{cases} 
		\frac{\xi \rho}{\lambda + \xi\rho} \bm{\delta}_{i_{k}j}^{k+1}, &\|\bm{\delta}^{k+1}_{i_{k}j}\|  \leq \xi + \frac{\lambda}{\rho} \\
		(1 - \frac{\lambda}{\rho \|\bm{\delta}_{i_{k}j}^{k+1}\|}) \bm{\delta}_{i_{k}j}^{k+1}, & \xi + \frac{\lambda}{\rho} \leq  \|\bm{\delta}_{i_{k}j}^{k+1}\| \leq \lambda + \frac{\lambda}{\rho} \\  
		\frac{\max\{0,1 - \frac{a \lambda}{(a-1)\rho \|\bm{\delta}_{i_{k}j}^{k+1}\|}\}}{1-1/[(a-1)\rho]}\bm{\delta}_{i_{k}j}^{k+1}, & \lambda + \frac{\lambda}{\rho} < \|\bm{\delta}_{i_{k}j}^{k+1}\| \leq  a\lambda \\ \bm{\delta}_{i_{k}j}^{k+1}, & \|\bm{\delta}_{i_{k}j}^{k+1}\| > a \lambda. \end{cases}
		\end{equation*}
		
		\item  $v_{i_{k}j}^{k+1} = v_{i_{k}j}^{k} + \rho(\bm{\omega}_{i_{k}}^{k+1} - \bm{\omega}_{j}^{k+1} - \bm{\theta}_{i_{k}j}^{k+1}).$
		\item   $ \bm{\theta}_{ji_{k}}^{k+1} = -\bm{\theta}_{i_{k}j}^{k+1} $, \ $ v_{ji_{k}}^{k+1} =- v_{i_{k}j}^{k+1} $.
		\item For other devices $ i \ne i_{k} $: $ \bm{\theta}_{ij}^{k+1} = \bm{\theta}_{ij}^{k}$, $ v_{ij}^{k+1} =v_{ij}^k $.
		\item Server updates $\bm\zeta_{i_{k}}^{k+1}=\frac{1}{m}\sum_{j=1}^m(\bm\omega_j^{k+1}+\bm\theta_{i_{k}j}^{k+1}-\frac{v_{i_{k}j}^{k+1}}{\rho})$.
		
		\STATE 3: \textbf{[Communication]} Server sends $ \bm{\zeta}_{i_{k}}^{k+1}$ to device $i_{k}$.		
		\ENDFOR
	\end{algorithmic}
\end{algorithm}

\section{Proof of Theorem 1}
\label{app:B}
\setcounter{equation}{0}
\renewcommand\theequation{B.\arabic{equation}}
\begin{proof}
	Because $f_i$ is $L_f$-smooth and $\mu =\rho-L_->0$, $h_i$ is $L_h$-smooth with $L_h=L_f+\rho$ and $\mu$-strongly convex. By the standard results of gradient descent \cite{Polyak1987}, after $ t $ epochs, we have 
	\begin{equation*}
	\|\bm{\omega}_{i}^{k+1,t+1} - \hat{\bm{\omega}}_{i}^{k+1} \|^{2} \leq c^{t} \|\bm{\omega}_{i}^{k}- \hat{\bm{\omega}}_{i}^{k+1}\|^2, 
	\end{equation*} 
	where $ \hat{\bm{\omega}}_{i}^{k+1} $ is the exact solution of $ \min_{\bm{\omega}_{i}} h_{i}(\bm{\omega}_{i}) $ and $ c = 1-\alpha \frac{2\mu L_{f}}{L_{f}+\rho+\mu}$. Then, we get 
	\begin{equation*}
	\small
	\begin{aligned}
	\|\bm{\omega}_{i}^{k+1,t+1} - \hat{\bm{\omega}}_{i}^{k+1} \| \leq (\sqrt{c})^{t} \|\bm{\omega}_{i}^{k}- \hat{\bm{\omega}}_{i}^{k+1}\| \leq (\sqrt{c})^{t} (\|\bm{\omega}_{i}^{k}- \bm{\omega}_{i}^{k+1,t+1}\| + \|\bm{\omega}_{i}^{k+1,t+1}- \hat{\bm{\omega}}_{i}^{k+1}\|).
	\end{aligned}
	\end{equation*}
	Thus, we obtain 
	\begin{equation}
	\small
	\label{eq:omega_i}
	\|\bm{\omega}_{i}^{k+1,t+1} - \hat{\bm{\omega}}_{i}^{k+1}\| \leq \frac{(\sqrt{c})^{t}}{1 - (\sqrt{c})^{t}} \|\bm{\omega}_{i}^{k+1,t+1} -\bm{\omega}_{i}^{k}\|.
	\end{equation}
	For any $\epsilon_i >0$, as long as $ t \geq \frac{2\log(\frac{\epsilon_i}{1+\epsilon_i})}{\log c}$, we have  
	\begin{equation}
	\small
	\label{eq:epsilon1}
	\frac{(\sqrt{c})^{t}}{1 - (\sqrt{c})^{t}} \leq \epsilon_i.
	\end{equation}
	The proof can be completed by combing \eqref{eq:omega_i} and \eqref{eq:epsilon1}. 
\end{proof}

\section{Proof of Theorem 2}
\label{app:C}
\setcounter{equation}{0}
\renewcommand\theequation{C.\arabic{equation}}
To save space, we introduce a matrix
\begin{equation*}
\mathbf{A}_{ij} = [\mathbf{0}_{d\times (i-1)d}, \mathbf{I}_{d},\mathbf{0}_{(j-i)d}, -\mathbf{I}_{d} , \mathbf{0}_{d \times (m - j-1)d}] \in \mathbb{R}^{d \times md},
\end{equation*} 
and then we have $\bm\omega_i-\bm\omega_j=\mathbf{A}_{ij}\bm\omega$ for $i,j\in [m]$. We define 
\begin{align}
\bm{\eta}_{ij}^{k} & = v_{ij}^{k+1} + \rho (\bm{\theta}_{ij}^{k+1} - \bm{\theta}_{ij}^{k} + \bm{\omega}_{j}^{k+1} - \bm{\omega}_{j}^{k}), \label{eq:eta}\\
\tilde{\bm{\eta}}_{ij}^{k} & = v_{ij}^{k+1} + \rho (\bm{\theta}_{ij}^{k+1} - \bm{\theta}_{ij}^{k}), \quad i, j \in [m]. \label{eq:tilde eta}
\end{align}
Note that $ \bm{\eta}_{ij}^{k} = \tilde{\bm{\eta}}_{ij}^{k} + \rho (\bm{\omega}_{j}^{k+1} - \bm{\omega}_{j}^{k}) $ and $ \tilde{\bm{\eta}}_{ij}^{k} = -\tilde{\bm{\eta}}_{ji}^{k} $. At each communication round, we allow only a subset of devices $ i \in \mathcal{A}_{k} $ to perform local updates. The subset $ \mathcal{A}_{k} $ is an i.i.d. realization of $ \hat{\mathcal{A}} $, and  $ \hat{\mathcal{A}} $ is the collection of all subsets of $ [m] $. Let $ \mathcal{F}_{k} = \sigma(\mathcal{A}_{0},...,\mathcal{A}_{k}) $ be the $ \sigma $-algebra generated by $\mathcal{A}_{0},...,\mathcal{A}_{k}$. Our analysis allows arbitrary sampling strategies that satisfy Assumption 3. For convenience, we assume that at each communication round of {FPFC}, all $ v_{ij} $ and $ \theta_{ij} $($ i,j \in [m] $) are updated by \eqref{app:theta} and \eqref{eq:v}, respectively. 

\subsection{Proof of Lemma 1}
\begin{proof}
	From the definition \eqref{eq:lag} of the augmented Lagrangian $\tilde{\mathcal{L}}_{\rho}(\bm\omega,\bm \theta, v)$, \eqref{eq:v}, \eqref{eq:eta}, and \eqref{eq:tilde eta}, we have 
	\begin{equation*}
	\small
	\begin{aligned}
	&\quad \tilde{\mathcal{L}}_{\rho}( \bm{\omega}^{k}, \bm{\theta}^{k},v^{k}) -  \tilde{\mathcal{L}}_{\rho}( \bm{\omega}^{k+1}, \bm{\theta}^{k}, v^{k}) \\
	&= \{f(\bm{\omega}^{k}) + \frac{1}{2m} \sum_{i=1}^{m}\sum_{j=1}^{m} [\tilde{g}(\|\theta_{ij}^{k}\|) + \left<v_{ij}^{k}, \mathbf{A}_{ij}\bm{\omega}^{k} -\bm{\theta}_{ij}^{k} \right> +  \frac{\rho}{2}\|\mathbf{A}_{ij}\bm{\omega}^{k} - \bm{\theta}_{ij}^{k}\|^2]\}  -\{f(\bm{\omega}^{k+1}) + \frac{1}{2m} \sum_{i=1}^{m}\sum_{j=1}^{m} [\tilde{g}(\|\theta_{ij}^{k}\|) \\
	&\quad + \left<v_{ij}^{k}, \mathbf{A}_{ij}\bm{\omega}^{k+1} -\bm{\theta}_{ij}^{k} \right> + \frac{\rho}{2}\|\mathbf{A}_{ij}\bm{\omega}^{k+1} - \bm{\theta}_{ij}^{k}\|^2]\} \\
	&= f(\bm{\omega}^{k}) - f(\bm{\omega}^{k+1}) + \frac{1}{2m}\sum_{i=1}^{m}\sum_{j=1}^{m} [\left<v_{ij}^{k}, \mathbf{A}_{ij}\bm{\omega}^{k} - \mathbf{A}_{ij}\bm{\omega}^{k+1} \right>+  \frac{\rho}{2} \| \mathbf{A}_{ij}\bm{\omega}^{k+1} - \bm{\theta}_{ij}^{k} + \mathbf{A}_{ij}\bm{\omega}^{k} - \mathbf{A}_{ij}\bm{\omega}^{k+1}\|^2 -\frac{\rho}{2}\|\mathbf{A}_{ij}\bm{\omega}^{k+1}- \bm{\theta}_{ij}^{k}\|^2]\\
	&=f(\bm{\omega}^{k}) - f(\bm{\omega}^{k+1}) + \frac{1}{2m}\sum_{i=1}^{m}\sum_{j=1}^{m} [\left<v_{ij}^{k}, \mathbf{A}_{ij}\bm{\omega}^{k} - \mathbf{A}_{ij}\bm{\omega}^{k+1} \right>  +  \frac{\rho}{2} \|\mathbf{A}_{ij}\bm{\omega}^{k} - \mathbf{A}_{ij}\bm{\omega}^{k+1}\|^2  + \rho \left<\mathbf{A}_{ij}\bm{\omega}^{k+1} - \bm{\theta}_{ij}^{k}, \mathbf{A}_{ij}\bm{\omega}^{k} - \mathbf{A}_{ij}\bm{\omega}^{k+1} \right>],
	\end{aligned}	
	\end{equation*}
	by \eqref{eq:v},	
	\begin{equation}
	\label{eq:027}
	\small
	\begin{aligned}
	&\quad \tilde{\mathcal{L}}_{\rho}( \bm{\omega}^{k}, \bm{\theta}^{k},v^{k}) -  \tilde{\mathcal{L}}_{\rho}( \bm{\omega}^{k+1}, \bm{\theta}^{k}, v^{k}) \\
	&=f(\bm{\omega}^{k}) - f(\bm{\omega}^{k+1}) + \frac{1}{2m}\sum_{i=1}^{m}\sum_{j=1}^{m} [\left<v_{ij}^{k+1} + \rho(\bm{\theta}_{ij}^{k+1} - \bm{\theta}_{ij}^{k}), (\bm{\omega}_{i}^{k} - \bm{\omega}_{j}^{k}) - (\bm{\omega}_{i}^{k+1} - \bm{\omega}_{j}^{k+1}) \right>  + \frac{\rho}{2} \| (\bm{\omega}_{i}^{k} - \bm{\omega}_{j}^{k}) - (\bm{\omega}_{i}^{k+1} - \bm{\omega}_{j}^{k+1}) \|^2]  \\
	&=f(\bm{\omega}^{k}) - f(\bm{\omega}^{k+1})  + \frac{1}{2m}\sum_{i=1}^{m}\sum_{j=1}^{m}[\left<v_{ij}^{k+1} + \rho(\bm{\theta}_{ij}^{k+1} - \bm{\theta}_{ij}^{k}), \bm{\omega}_{i}^{k} - \bm{\omega}_{i}^{k+1} \right>  + \left<v_{ij}^{k+1} + \rho(\bm{\theta}_{ij}^{k+1} - \bm{\theta}_{ij}^{k}), \bm{\omega}_{j}^{k+1} - \bm{\omega}_{j}^{k} \right> + \frac{\rho}{2} \|\bm{\omega}_{i}^{k} - \bm{\omega}_{i}^{k+1} \|^2 \\
	& \quad + \frac{\rho}{2} \|\bm{\omega}_{j}^{k} - \bm{\omega}_{j}^{k+1} \|^2 + \rho \left< \bm{\omega}_{i}^{k} - \bm{\omega}_{i}^{k+1}, \bm{\omega}_{j}^{k+1} - \bm{\omega}_{j}^{k}\right>]  \\
	&\overset{\eqref{eq:tilde eta}}{=}f(\bm{\omega}^{k}) - f(\bm{\omega}^{k+1}) + \frac{1}{2m}\sum_{i=1}^{m}\sum_{j=1}^{m} [\left<\bm{\eta}_{ij}^{k}, \bm{\omega}_{i}^{k} - \bm{\omega}_{i}^{k+1} \right>  +  \left<\tilde{\bm{\eta}}^{k}_{ij}, \bm{\omega}_{j}^{k+1} - \bm{\omega}_{j}^{k} \right>] + \frac{\rho}{2}\sum_{i=1}^{m}\|\bm{\omega}_{i}^{k} - \bm{\omega}_{i}^{k+1}\|^2,
	\end{aligned}	
	\end{equation}	
	where we use \eqref{eq:eta}, \eqref{eq:tilde eta}, and the elementary expression
	\begin{equation*}
	\begin{aligned}
	\frac{1}{2m}\sum_{i=1}^{m}\sum_{j=1}^{m}[\frac{\rho}{2}\|\bm{\omega}_{i}^{k} - \bm{\omega}_{i}^{k+1}\|^2 + \frac{\rho}{2}\|\bm{\omega}_{j}^{k} - \bm{\omega}_{j}^{k+1}\|^2] = \frac{\rho}{2}\sum_{i=1}^{m}\|\bm{\omega}_{i}^{k} - \bm{\omega}_{i}^{k+1}\|^2
	\end{aligned}
	\end{equation*}
	to obtain the last equality. 
	Using the relation between $ \bm{\eta}_{ij}^{k} $ and $ \tilde{\bm{\eta}}_{ij}^{k}$, we have
	\begin{equation*}
	\small
	\begin{aligned}
	&\quad \frac{1}{2m}\sum_{i=1}^{m}\sum_{j=1}^{m}[\left<\bm{\eta}_{ij}^{k}, \bm{\omega}_{i}^{k} -\bm{\omega}_{i}^{k+1} \right> + \left<\tilde{\bm{\eta}}_{ij}^{k}, \bm{\omega}_{j}^{k+1}  - \bm{\omega}_{j}^{k} \right>] \\
	&=\frac{1}{2m}\sum_{i=1}^{m}\sum_{j=1}^{m}[\left<\bm{\eta}_{ij}^{k}, \bm{\omega}_{i}^{k} -\bm{\omega}_{i}^{k+1} \right> + \left<\bm{\eta}_{ij}^{k}, \bm{\omega}_{j}^{k+1} - \bm{\omega}_{j}^{k} \right> - \rho \|\bm{\omega}_{j}^{k+1} - \bm{\omega}_{j}^{k}\|^2]\\
	&= \frac{1}{2m}\sum_{i=1}^{m}\sum_{j=1}^{m}[\left<\bm{\eta}_{ij}^{k}, \bm{\omega}_{i}^{k} -\bm{\omega}_{i}^{k+1} \right> +  \left<\bm{\eta}_{ij}^{k},(\bm{\omega}_{i}^{k} - \bm{\omega}_{i}^{k+1}) +(\bm{\omega}_{j}^{k+1} - \bm{\omega}_{j}^{k})  - (\bm{\omega}_{i}^{k} - \bm{\omega}_{i}^{k+1}) \right>- \rho \|\bm{\omega}_{j}^{k+1} - \bm{\omega}_{j}^{k}\|^2]\\
	&= \frac{1}{2m}\sum_{i=1}^{m}\sum_{j=1}^{m}[2\left<\bm{\eta}_{ij}^{k}, \bm{\omega}_{i}^{k} -\bm{\omega}_{i}^{k+1} \right> + \left<\tilde{\bm{\eta}}_{ij}^{k},(\bm{\omega}_{j}^{k+1} - \bm{\omega}_{j}^{k}) - (\bm{\omega}_{i}^{k} - \bm{\omega}_{i}^{k+1}) \right>  + \rho \left<\bm{\omega}_{j}^{k+1} - \bm{\omega}_{j}^{k}, \bm{\omega}_{i}^{k+1}-\bm{\omega}_{i}^{k} \right>].
	\end{aligned}
	\end{equation*}
	Because 
	\begin{equation*}
	\small
	\begin{aligned}
	&\quad \frac{1}{2m}\sum_{i=1}^{m}\sum_{j=1}^{m} \left<\tilde{\bm{\eta}}_{ij}^{k}, (\bm{\omega}_{j}^{k+1} - \bm{\omega}_{j}^{k}) -(\bm{\omega}_{i}^{k} - \bm{\omega}_{i}^{k+1}) \right>\\
	&= \frac{1}{2m}\sum_{i<j} \left<\tilde{\bm{\eta}}_{ij}^{k}, (\bm{\omega}_{j}^{k+1} - \bm{\omega}_{j}^{k}) -(\bm{\omega}_{i}^{k} - \bm{\omega}_{i}^{k+1}) \right> + \left<\tilde{\bm{\eta}}_{ji}^{k}, (\bm{\omega}_{i}^{k+1} - \bm{\omega}_{i}^{k}) -(\bm{\omega}_{j}^{k} - \bm{\omega}_{j}^{k+1}) \right>\\
	&= \frac{1}{2m}\sum_{i<j} \left<\tilde{\bm{\eta}}_{ij}^{k}, (\bm{\omega}_{j}^{k+1} - \bm{\omega}_{j}^{k}) -(\bm{\omega}_{i}^{k} - \bm{\omega}_{i}^{k+1}) \right>  + \left<-\tilde{\bm{\eta}}_{ij}^{k}, (\bm{\omega}_{j}^{k+1} - \bm{\omega}_{j}^{k}) -(\bm{\omega}_{i}^{k} - \bm{\omega}_{i}^{k+1}) \right>=0,
	\end{aligned}
	\end{equation*}
	we have
	\begin{equation}\label{eq:029}
	\begin{aligned} & \frac{1}{2m}\sum_{i=1}^{m}\sum_{j=1}^{m}[\left<\bm{\eta}_{ij}^{k}, \bm{\omega}_{i}^{k} -\bm{\omega}_{i}^{k+1} \right> + \left<\tilde{\bm{\eta}}_{ij}^{k}, \bm{\omega}_{j}^{k+1} - \bm{\omega}_{j}^{k} \right>]  = \frac{1}{m}\sum_{i=1}^{m}\sum_{j=1}^{m}[\left<\bm{\eta}_{ij}^{k}, \bm{\omega}_{i}^{k} - \bm{\omega}_{i}^{k+1} \right> + \frac{\rho}{2}\left<\bm{\omega}_{i}^{k+1} - \bm{\omega}_{i}^{k}, \bm{\omega}_{j}^{k+1}-\bm{\omega}_{j}^{k} \right>].
	\end{aligned}
	\end{equation}
	Combining \eqref{eq:027} and \eqref{eq:029}, we obtain
	\begin{equation*}
	\small
	\begin{aligned}
	\tilde{\mathcal{L}}_{\rho}(\bm{\omega}^{k}, \bm{\theta}^{k}, v^{k}) - \tilde{\mathcal{L}}_{\rho}(\bm{\omega}^{k+1}, \bm{\theta}^{k}, v^{k}) &= f(\bm{\omega}^{k}) - f(\bm{\omega}^{k+1}) + \frac{1}{m}\sum_{i=1}^{m}\sum_{j=1}^{m}[\left<\bm{\eta}_{ij}^{k}, \bm{\omega}_{i}^{k}-\bm{\omega}_{i}^{k+1} \right> \\ 
	& \quad + \frac{\rho}{2}\left<\bm{\omega}_{i}^{k+1}-\bm{\omega}_{i}^{k}, \bm{\omega}_{j}^{k+1} - \bm{\omega}_{j}^{k} \right> ] + \frac{\rho}{2} \sum_{i=1}^{m}\|\bm{\omega}_{i}^{k} - \bm{\omega}_{i}^{k+1}\|^2. 
	\end{aligned}
	\end{equation*}
	Let
	$
	r_{i}^{k} = f_{i}(\bm{\omega}_{i}^{k}) - f_{i}(\bm{\omega}_{i}^{k+1}) + \frac{1}{m} \sum_{j=1}^{m}[\left<\bm{\eta}_{ij}^{k}, \bm{\omega}_{i}^{k} - \bm{\omega}_{i}^{k+1}\right> + \frac{\rho}{2}\left<\bm{\omega}_{i}^{k+1}-\bm{\omega}_{i}^{k}, \bm{\omega}_{j}^{k+1}-\bm{\omega}_{j}^{k} \right>] + \frac{\rho}{2}\|\bm{\omega}_{i}^{k} -\bm{\omega}_{i}^{k+1}\|^2$. For $ i \notin \mathcal{A}_{k} $, $ \omega_{i}^{k+1} = \omega_{i}^{k} $; thus, we have 
	\begin{equation}
	\small
	\label{eq:rik}
	\tilde{\mathcal{L}}_{\rho}(\bm{\omega}^{k}, \bm{\theta}^{k}, v^{k}) - \tilde{\mathcal{L}}_{\rho}(\bm{\omega}^{k+1}, \bm{\theta}^{k}, v^{k})=\sum_{i \in \mathcal{A}_{k}} r_i^k.
	\end{equation}
	Define $h_i(\bm\omega_i)=f_i(\bm\omega_i)+\frac{\rho}{2m}\sum_{j=1}^m \|\bm\omega_i-\bm\omega_j^k-\bm\theta_{ij}^k+\frac{ v_{ij}^k}{\rho}\|^2$. For $ i \in \mathcal{A}_{k} $, using \eqref{eq:eta}, we have 
	\begin{equation}\label{eq:nable h}
	\small
	\begin{aligned}
	\nabla h_{i}(\bm{\omega}_{i}^{k+1}) &= \nabla f_{i}(\bm{\omega}_{i}^{k+1}) + \frac{\rho}{m}\sum_{j=1}^{m}(\bm{\omega}_{i}^{k+1} - \bm{\omega}_{j}^{k} - \bm{\theta}_{ij}^{k} + \frac{1}{\rho}v_{ij}^{k})  \overset{\eqref{eq:eta}}{= }\nabla f_{i}(\bm{\omega}_{i}^{k+1}) + \frac{1}{m}\sum_{j=1}^{m} \bm{\eta}^{k}_{ij}.
	\end{aligned}
	\end{equation}
	Using the last expression and $L_f$-smoothness of $ f_{i} $, we can derive that 
	\begin{equation}
	\label{eq:030}
	\begin{aligned}
	r_{i}^{k} &\geq \frac{\rho}{2m} \sum_{j=1}^{m} \left<\bm{\omega}_{i}^{k+1} - \bm{\omega}_{i}^{k}, \bm{\omega}_{j}^{k+1} - \bm{\omega}_{j}^{k} \right> + \frac{\rho - L_{f}}{2}\|\bm{\omega}_{i}^{k} - \bm{\omega}_{i}^{k+1}\|^2 + \left<\nabla h_{i}(\bm{\omega}_{i}^{k+1}), \bm{\omega}_{i}^{k} - \bm{\omega}_{i}^{k+1} \right> .
	\end{aligned}
	\end{equation}
	Summing \eqref{eq:030} for $ i \in \mathcal{A}_{k} $, we obtain 
	\begin{equation}\label{eq:sum r}
	\small
	\begin{aligned}
	\sum_{i\in \mathcal{A}_{k}} r_{i}^{k} &\geq \frac{\rho}{2m}\|\sum_{i\in \mathcal{A}_{k}} (\bm{\omega}_{i}^{k+1} - \bm{\omega}_{i}^{k})\|^2 +  \sum_{i\in \mathcal{A}_{k}} \frac{\rho - L_{f}}{2}\|\bm{\omega}_{i}^{k} - \bm{\omega}_{i}^{k+1}\|^2  + \sum_{i\in \mathcal{A}_{k}} \left<\nabla h_{i}(\bm{\omega}_{i}^{k+1}),\bm{\omega}_{i}^{k} - \bm{\omega}_{i}^{k+1} \right> \\
	& \geq \sum_{i \in \mathcal{A}_{k}} \frac{\rho - L_{f}}{2}\|\bm{\omega}_{i}^{k} - \bm{\omega}_{i}^{k+1}\|^2 - \sum_{i \in \mathcal{A}_{k}} \|\nabla h_{i}(\bm{\omega}_{i}^{k+1})\| \|\bm{\omega}_{i}^{k} - \bm{\omega}_{i}^{k+1}\|.
	\end{aligned}
	\end{equation}
	According to Theorem 1, for each device $ i \in \mathcal{A}_{k} $, there exists an $\epsilon_i\in [0,1]$ such that $T_{i} = \frac{2\log(\frac{\epsilon_{i}}{1+\epsilon_{i}})}{\log c} $, and we have $ \|\bm{\omega}_{i}^{k+1} - \hat{\bm{\omega}}_{i}^{k+1}\| \leq (c^{-\frac{T_{i}}{2}} -1)^{-1} \|\bm{\omega}_{i}^{k+1} - \bm{\omega}_{i}^{k}\|$. Then, due to the $L_h$-smoothness of $h_i$ with $L_h=L_f+\rho$, we have that for $ i \in \mathcal{A}_{k} $  
	\begin{equation}\label{eq:h e}
	\small
	\|\nabla h_{i}(\bm{\omega}_{i}^{k+1})\| \leq L_{h} \|\bm{\omega}_{i}^{k+1} - \hat{\bm{\omega}}_{i}^{k+1}\| 
	\leq (c^{-\frac{T_{i}}{2}} -1)^{-1} L_{h} \|\bm{\omega}_{i}^{k+1} - \bm{\omega}_{i}^{k}\|.
	\end{equation}
	Inserting \eqref{eq:h e} into \eqref{eq:sum r}, we obtain 
	\begin{equation*}
	\small
	\sum_{i\in \mathcal{A}_{k}} r_{i}^{k} \geq  \sum_{i \in \mathcal{A}_{k}} \frac{\rho - L_{f} - 2(c^{-\frac{T_{i}}{2}} -1)^{-1}L_{h}}{2}\|\bm{\omega}_{i}^{k} - \bm{\omega}_{i}^{k+1}\|^2.
	\end{equation*}
	Together with \eqref{eq:rik}, we then obtain  
	\begin{equation}\label{eq:descent of w}
	\small
	\begin{aligned}
	\tilde{\mathcal{L}}_{\rho}(\bm{\omega}^{k}, \bm{\theta}^{k}, v^{k}) -  \tilde{\mathcal{L}}_{\rho}(\bm{\omega}^{k+1}, \bm{\theta}^{k}, v^{k})  \geq  \sum_{i\in \mathcal{A}_{k}} \frac{\rho - L_{f} - 2(c^{-\frac{T_{i}}{2}} -1)^{-1}L_{h}}{2}\|\bm{\omega}_{i}^{k} - \bm{\omega}_{i}^{k+1}\|^2.
	\end{aligned}
	\end{equation}
	This completes the proof. 	
\end{proof}

\subsection{Proof of Lemma 2}

\begin{proof}
	From the definition of the augmented Lagrangian $\tilde{\mathcal{L}}_{\rho}(\bm\omega,\bm \theta, v)$ \eqref{eq:lag}, we have 
	\begin{equation}\label{eq:49}
	\small
	\begin{aligned}
	\tilde{\mathcal{L}}_{\rho}(\bm{\omega}^{k+1}, \bm{\theta}^{k+1}, v^{k+1})  &= \tilde{\mathcal{L}}_{\rho}(\bm{\omega}^{k+1}, \bm{\theta}^{k+1}, v^{k}) + \frac{1}{2m}\sum_{i=1}^{m}\sum_{j=1}^{m} \left< v_{ij}^{k+1} - v_{ij}^{k}, \mathbf{A}_{ij}\bm{\omega}^{k+1} - \bm{\theta}_{ij}^{k+1}\right> \\
	&\overset{\eqref{eq:v}}{=} \tilde{\mathcal{L}}_{\rho}(\bm{\omega}^{k+1}, \bm{\theta}^{k+1}, v^{k}) + \frac{1}{2m\rho}\sum_{i=1}^{m}\sum_{j=1}^{m}\|v_{ij}^{k+1} - v_{ij}^{k}\|^2
	\end{aligned}
	\end{equation}
	and 
	\begin{equation}\label{eq:50}
	\small
	\begin{aligned}
	\tilde{\mathcal{L}}_{\rho}(\bm{\omega}^{k+1}, \bm{\theta}^{k}, v^{k}) - \tilde{\mathcal{L}}_{\rho}(\bm{\omega}^{k+1}, \bm{\theta}^{k+1}, v^{k}) &=\frac{1}{2m}\sum_{i=1}^{m}\sum_{j=1}^{m} [\tilde{g}(\|\theta_{ij}^{k}\|) - \tilde{g}(\|\bm{\theta}_{ij}^{k+1}\|) + \left<v_{ij}^{k}, \bm{\theta}_{ij}^{k+1}-\bm{\theta}_{ij}^{k} \right> \\
	&\quad	 + \frac{\rho}{2} \|\mathbf{A}_{ij}\bm{\omega}^{k+1} - \bm{\theta}_{ij}^{k}\|^2 - \frac{\rho}{2}\|\mathbf{A}_{ij}\bm{\omega}^{k+1} - \bm{\theta}_{ij}^{k+1}\|^2].
	\end{aligned}
	\end{equation}
	Using \eqref{app:theta}, we obtain that 
	\begin{equation}
	\small
	\label{eq:theta v}
	\begin{aligned}
	\nabla \tilde{g}(\|\bm{\theta}_{ij}^{k+1}\|) = v_{ij}^{k} + \rho(\bm{\omega}_{i}^{k+1} + \bm{\omega}_{j}^{k+1} - \bm{\theta}_{ij}^{k+1}) =  v_{ij}^{k+1}.
	\end{aligned}
	\end{equation}
	From the definition of smoothed SCAD penalty, we know that $\tilde{g}$ is $L_{\tilde{g}}$-smooth with $L_{\tilde{g}}=\max\{\frac{\lambda}{\xi},\frac{1}{a-1}\}$. Using the $L_{\tilde{g}}$-smoothness of $\tilde{g}$   
	and \eqref{eq:theta v}, we can derive from \eqref{eq:50} that  
	\begin{equation}\label{eq:theta v2}
	\small
	\begin{aligned}
	\tilde{\mathcal{L}}_{\rho}(\bm{\omega}^{k+1}, \bm{\theta}^{k}, v^{k}) - \tilde{\mathcal{L}}_{\rho}(\bm{\omega}^{k+1}, \bm{\theta}^{k+1}, v^{k}) &\geq \frac{1}{2m} \sum_{i=1}^{m}\sum_{j=1}^{m}[\left<v_{ij}^{k+1} - v_{ij}^{k}, \bm{\theta}_{ij}^{k} - \bm{\theta}_{ij}^{k+1}\right> + \frac{\rho}{2}\|\mathbf{A}_{ij}\bm{\omega}^{k+1} - \bm{\theta}_{ij}^{k}\|^2 \\
	&\quad - \frac{\rho}{2}\|\mathbf{A}_{ij}\bm{\omega}^{k+1} - \bm{\theta}_{ij}^{k+1}\|^2 - \frac{L_{\tilde{g}}}{2}\|\bm{\theta}_{ij}^{k} - \bm{\theta}_{ij}^{k+1}\|^2].
	\end{aligned}
	\end{equation}
	By \eqref{eq:v}, we have 
	\begin{equation*}
	\small
	\begin{aligned}
	\left<v_{ij}^{k+1}-v_{ij}^{k}, \bm{\theta}_{ij}^{k}-\bm{\theta}_{ij}^{k+1} \right> + \frac{\rho}{2}\|\mathbf{A}_{ij}\bm{\omega}^{k+1} - \bm{\theta}_{ij}^{k}\|^2 
	&\overset{\eqref{eq:v}}{=} \left<v_{ij}^{k+1}-v_{ij}^{k}, \bm{\theta}_{ij}^{k} - \bm{\theta}_{ij}^{k+1} \right> + \frac{\rho}{2}\|\frac{1}{\rho}(v_{ij}^{k+1} - v_{ij}^{k}) + (\bm{\theta}_{ij}^{k+1} - \bm{\theta}_{ij}^{k})\|^2 \\
	&= \frac{1}{2\rho}\|v_{ij}^{k+1} - v_{ij}^{k}\|^2 + \frac{\rho}{2}\|\bm{\theta}_{ij}^{k+1} - \bm{\theta}_{ij}^{k}\|^2.
	\end{aligned}
	\end{equation*}
	Inserting the last equality into \eqref{eq:theta v2}, we can derive that  
	\begin{equation}\label{eq:theta v3}
	\small
	\begin{aligned}
	\tilde{\mathcal{L}}_{\rho}(\bm{\omega}^{k+1}, \bm{\theta}^{k}, v^{k}) - \tilde{\mathcal{L}}_{\rho}(\bm{\omega}^{k+1}, \bm{\theta}^{k+1}, v^{k}) &\geq \frac{1}{2m}\sum_{i=1}^{m}\sum_{j=1}^{m}[\frac{1}{2\rho}\|v_{ij}^{k+1} - v_{ij}^{k}\|^2 + \frac{\rho}{2}\|\bm{\theta}_{ij}^{k+1} - \bm{\theta}_{ij}^{k}\|^2 - 
	\frac{\rho}{2}\|\mathbf{A}_{ij}\bm{\omega}^{k+1} - \bm{\theta}_{ij}^{k+1}\|^2 \\
	&\quad - \frac{L_{\tilde{g}}}{2}\|\bm{\theta}_{ij}^{k} - \bm{\theta}_{ij}^{k+1}\|^2] \\
	&\overset{\eqref{eq:v}}{=} \frac{1}{2m}\sum_{i=1}^{m}\sum_{j=1}^{m}\frac{\rho - L_{\tilde{g}}}{2}\|\bm{\theta}_{ij}^{k} - \bm{\theta}_{ij}^{k+1}\|^2.
	\end{aligned}
	\end{equation}
	Using the $L_{\tilde{g}}$-smoothness of $\tilde{g}$ and \eqref{eq:theta v}, we can bound $\| v_{ij}^{k+1}- v_{ij}^k\|$ as  
	\begin{equation}\label{eq:theta v4}
	\begin{aligned}
	\|v_{ij}^{k+1} - v_{ij}^{k}\| \leq L_{\tilde{g}}\|\bm{\theta}_{ij}^{k+1} - \bm{\theta}_{ij}^{k}\|.
	\end{aligned}
	\end{equation}
	By \eqref{eq:theta v3} and \eqref{eq:theta v4}, we can further derive from \eqref{eq:49} that 
	\begin{equation}\label{eq:descent}
	\small
	\begin{aligned}
	\tilde{\mathcal{L}}_{\rho}(\bm{\omega}^{k+1}, \bm{\theta}^{k+1}, v^{k+1}) &=\tilde{\mathcal{L}}_{\rho}(\bm{\omega}^{k+1}, \bm{\theta}^{k+1},v^{k}) + \frac{1}{2m\rho}\sum_{i=1}^{m}\sum_{j=1}^{m}\|v_{ij}^{k+1} - v_{ij}^{k}\|^2 \\
	&\overset{\eqref{eq:theta v3},\eqref{eq:theta v4}}{\leq} \tilde{\mathcal{L}}_{\rho}(\bm{\omega}^{k+1}, \bm{\theta}^{k}, v^{k}) - \frac{1}{2m}(\frac{\rho - L_{\tilde{g}}}{2} - \frac{L_{\tilde{g}}^{2}}{\rho})\sum_{i=1}^{m}\sum_{j=1}^{m}\|\bm{\theta}_{ij}^{k} - \bm{\theta}_{ij}^{k+1}\|^2.
	\end{aligned}
	\end{equation}
	Together with Lemma \ref{lem:descent duo to w}, we have
	\begin{equation*}
	\small
	\begin{aligned}
	\tilde{\mathcal{L}}_{\rho}(\bm{\omega}^{k}, \bm{\theta}^{k}, v^{k}) -  \tilde{\mathcal{L}}_{\rho}(\bm{\omega}^{k+1}, \bm{\theta}^{k+1}, v^{k+1})
	&\geq  \sum_{i\in \mathcal{A}_{k}} \frac{\rho - L_{f} - 2(c^{-\frac{T_{i}}{2}} -1)^{-1}L_{h}}{2}\|\bm{\omega}_{i}^{k} - \bm{\omega}_{i}^{k+1}\|^2 \\
	&\quad + \frac{1}{2m}(\frac{\rho - L_{\tilde{g}}}{2} - \frac{L_{\tilde{g}}^{2}}{\rho})\sum_{i=1}^{m}\sum_{j=1}^{m}\|\bm{\theta}_{ij}^{k} - \bm{\theta}_{ij}^{k+1}\|^2.
	\end{aligned}
	\end{equation*}	
	This completes the proof. 
\end{proof}

\subsection{Proof of Lemma 3}
\begin{proof}
	Let us define $u_{ij}^k= \matr{A}_{ij} \bm{\omega}^{k}$. We have 
	\begin{equation*}
	\begin{aligned}
	\small
	f(\bm{\omega}^{k}) + \frac{1}{2m}\sum_{i=1}^{m}\sum_{i=1}^{m} \tilde{g}(\|u_{ij}^k\|)  \geq \tilde{F}(\tilde{\omega}^{*})= \min\limits_{\bm{\omega}, \bm{\theta}_{ij}}\{f(\bm{\omega}) + \frac{1}{2m}\sum_{i=1}^{m}\sum_{i=1}^{m} \tilde{g}(\|\bm{\theta}_{ij}\|): \matr{A}_{ij} \bm{\omega} - \bm{\theta}_{ij} = \bm{0}\}.
	\end{aligned}
	\end{equation*}
	Using this expression and definition of  $\tilde{\mathcal{L}}_{\rho}(\bm\omega,\bm \theta, v)$ \eqref{eq:lag}, we have
	\begin{equation}
	\label{eq:041}
	\begin{aligned}
	\tilde{\mathcal{L}}_{\rho}(\bm{\omega}^{k}, \bm{\theta}^{k}, v^{k}) &= f(\bm{\omega}^{k}) + \frac{1}{2m}\sum_{i=1}^{m}\sum_{j=1}^{m} [\tilde{g}(\|\theta_{ij}^{k}\|) + \left<v_{ij}^{k}, \matr{A}_{ij}\bm{\omega}^{k} - \bm{\theta}^{k}_{ij} \right>  + \frac{\rho}{2}\|\matr{A}_{ij}\bm{\omega}^{k} - \bm{\theta}^{k}_{ij}\|^2] \\
	&= f(\bm{\omega}^{k}) + \frac{1}{2m}\{\sum_{i=1}^{m}\sum_{j=1}^{m} [\tilde{g}(\|\theta_{ij}^{k}\|) + \left<\nabla \tilde{g}(\|\bm{\theta}_{ij}^{k}\|), u_{ij}^k - \bm{\theta}^{k}_{ij} \right> +  \frac{\rho}{2}\|\matr{A}_{ij}\bm{\omega}^{k} - \bm{\theta}^{k}_{ij}\|^2] \\
	&\geq f(\bm{\omega}^{k}) + \frac{1}{2m} \sum_{i=1}^{m}\sum_{j=1}^{m} [\tilde{g}(\|u^k_{ij}\|) -\frac{L_{\tilde{g}}}{2}\|\matr{A}_{ij}\bm{\omega}^{k} - \bm{\theta}^{k}_{ij}\|^2  + \frac{\rho}{2}\|\matr{A}_{ij}\bm{\omega}^{k} - \bm{\theta}^{k}_{ij}\|^2] \\
	&= f(\bm{\omega}^{k}) + \frac{1}{2m}\sum_{i=1}^{m}\sum_{j=1}^{m}[\tilde{g}(\|u^k_{ij}\|) + \frac{\rho - L_{\tilde{g}}}{2}\|\matr{A}_{ij}\bm{\omega}^{k} - \bm{\theta}^{k}_{ij}\|^2] \geq F^*,
	\end{aligned}
	\end{equation}
	where the first inequality holds from the $L_{\tilde{g}}$-smoothness of $\tilde{g}$, and the last inequality holds because of $ \rho> \max\{\frac{\lambda}{\xi},\frac{1}{a-1}\} $, Assumption 2, and 
	\begin{equation*}
	\begin{aligned}
	\small
	F^* - \tilde{F}(\tilde{\omega}^{*}) &= f(\omega^{*}) + \frac{1}{2m}\sum_{i=1}^{m}\sum_{j=1}^{m} g(\| \omega_{i}^{*}- \omega_{j}^{*}\|) - f(\tilde{\omega}^{*})  - \frac{1}{2m}\sum_{i=1}^{m}\sum_{j=1}^{m} \tilde{g}(\|\tilde{\omega}_{i}^{*}- \tilde{\omega}_{j}^{*}\|) \\
	& \leq f(\tilde{\omega}^{*}) + \frac{1}{2m}\sum_{i=1}^{m}\sum_{j=1}^{m} g(\| \tilde{\omega}_{i}^{*}- \tilde{\omega}_{j}^{*}\|) - f(\tilde{\omega}^{*}) - \frac{1}{2m}\sum_{i=1}^{m}\sum_{j=1}^{m} \tilde{g}(\|\tilde{\omega}_{i}^{*}- \tilde{\omega}_{j}^{*}\|) \leq 0.
	\end{aligned}
	\end{equation*}
	The last inequality holds because of Proposition 1.
\end{proof}

\subsection{Proof of Theorem 2} 
\begin{proof}	
	(a)
	First, 
	\begin{equation*}\label{eq:nabla lw}
	\small
	\begin{aligned}
	\nabla_{\bm{\omega}_{i}} \tilde{\mathcal{L}}_{0} (\bm{\omega}^{k+1}, \bm{\theta}^{k+1}, v^{k+1})	= \nabla f_{i}(\bm{\omega}_{i}^{k+1}) +\frac{1}{2m}\sum_{j=1}^{m}[(v_{ij}^{k+1}- v_{ji}^{k+1})] = \nabla f_{i}(\bm{\omega}_{i}^{k+1}) + \frac{1}{m}\sum_{j=1}^{m}v_{ij}^{k+1}.
	\end{aligned}
	\end{equation*}
	Using \eqref{eq:nable h},  we can derive from the last equality that 
	\begin{equation*}
	\small
	\begin{aligned}
	\nabla_{\bm{\omega}_{i}} \tilde{\mathcal{L}}_{0} ( \bm{\omega}^{k+1}, \bm{\theta}^{k+1},v^{k+1}) = \nabla h_{i}(\bm{\omega}_{i}^{k+1}) - \frac{\rho}{m}\sum_{j=1}^{m}(\bm{\omega}_{j}^{k+1} - \bm{\omega}_{j}^{k} + \bm{\theta}_{ij}^{k+1} - \bm{\theta}_{ij}^{k}).
	\end{aligned}
	\end{equation*}
	Then, by \eqref{eq:h e} and the Jensen inequality, for $ i \in \mathcal{A}_{k} $, we have  
	\begin{equation*}
	\small
	\begin{aligned}
	\|\nabla_{\bm{\omega}_{i}} \tilde{\mathcal{L}}_{0} ( \bm{\omega}^{k+1}, \bm{\theta}^{k+1},v^{k+1}) \|^2 \leq 3[(c^{-\frac{T_{i}}{2}}-1)^{-2}L^{2}_{h}\|\bm{\omega}_{i}^{k+1} - \bm{\omega}_{i}^{k}\|^2 + \frac{\rho^2}{m}\sum_{j=1}^{m} \|\bm{\omega}_{j}^{k+1} - \bm{\omega}_{j}^{k}\|^2 + \frac{\rho^2}{m} \sum_{j=1}^{m} \|\bm{\theta}_{ij}^{k+1} - \bm{\theta}_{ij}^{k}\|^2].
	\end{aligned}
	\end{equation*}
	Furthermore,
	\begin{equation*}
	\small
	\begin{aligned}
	\sum_{i \in \mathcal{A}_{k}}\|\nabla_{\bm{\omega}_{i}} \tilde{\mathcal{L}}_{0} ( \bm{\omega}^{k+1}, \bm{\theta}^{k+1},v^{k+1}) \|^2 \leq 3 \sum_{i \in \mathcal{A}_{k}}[(c^{-\frac{T_{i}}{2}}-1)^{-2}L^{2}_{h}+\rho^2]\|\bm{\omega}_{i}^{k+1} - \bm{\omega}_{i}^{k}\|^2 + \frac{3\rho^2}{m}\sum_{i=1}^{m}\sum_{j=1}^{m} \|\bm{\theta}_{ij}^{k+1} - \bm{\theta}_{ij}^{k}\|^2.
	\end{aligned}
	\end{equation*}	
	Because $ \rho> \max\{\frac{L_f}{1-2c^{\frac{T}{2}}},\frac{2\lambda}{\xi},\frac{2}{a-1}\} $, we have $ (\frac{\rho - L_{\tilde{g}}}{2} - \frac{L_{\tilde{g}}^{2}}{\rho})>0 $ and $ \frac{\rho - L_{f} - 2(c^{-\frac{T_{i}}{2}} -1)^{-1}L_{h}}{2} > 0 $. By Lemma \ref{lem:descent of L},
	\begin{equation*}
	\small
	\begin{aligned}
	\sum_{i \in \mathcal{A}_{k}}\|\nabla_{\bm{\omega}_{i}} \tilde{\mathcal{L}}_{0} ( \bm{\omega}^{k+1}, \bm{\theta}^{k+1},v^{k+1}) \|^2 \leq [\frac{6[(c^{-\frac{T}{2}} - 1)^{-2}L_{h}^2 + \rho^2]}{\rho - L_{f} - 2(c^{-\frac{T}{2}} - 1)^{-1}L_{h}} + \frac{12\rho^3}{\rho^2 - L_{\tilde{g}}\rho - 2L_{\tilde{g}}^2}][\tilde{\mathcal{L}}_{\rho}(\bm{\omega}^{k}, \bm{\theta}^{k}, v^{k})- \tilde{\mathcal{L}}_{\rho}(\bm{\omega}^{k+1}, \bm{\theta}^{k+1}, v^{k+1})].
	\end{aligned}	
	\end{equation*}			
	Taking the expectation of random variable $ \mathcal{A}_{k} $ conditioned on $ \mathcal{F}_{k-1} $, we then have 
	\begin{equation}\label{eq:right0}
	\small
	\begin{aligned}
	&\quad \mathbb{E}[\sum_{i \in \mathcal{A}_{k}}\|\nabla_{\bm{\omega}_{i}} \tilde{\mathcal{L}}_{0} ( \bm{\omega}^{k+1}, \bm{\theta}^{k+1},v^{k+1}) \|^2 | \mathcal{F}_{k-1}] \\
	&\leq [\frac{6[(c^{-\frac{T}{2}} - 1)^{-2}L_{h}^2 + \rho^2]}{\rho - L_{f} - 2(c^{-\frac{T}{2}} - 1)^{-1}L_{h}} + \frac{12\rho^3}{\rho^2 - L_{\tilde{g}}\rho - 2L_{\tilde{g}}^2}][\tilde{\mathcal{L}}_{\rho}(\bm{\omega}^{k}, \bm{\theta}^{k}, v^{k})- \mathbb{E}[\tilde{\mathcal{L}}_{\rho}(\bm{\omega}^{k+1}, \bm{\theta}^{k+1}, v^{k+1})|\mathcal{F}_{k-1}]].
	\end{aligned}
	\end{equation}
	Meanwhile, 
	\begin{equation*}
	\small
	\begin{aligned}
	\mathbb{E}[\sum_{i \in \mathcal{A}_{k}}\|\nabla_{\bm{\omega}_{i}} \tilde{\mathcal{L}}_{0} ( \bm{\omega}^{k+1}, \bm{\theta}^{k+1},v^{k+1}) \|^2 | \mathcal{F}_{k-1}] &= \sum_{\mathcal{A}} P(\mathcal{A}_{k} = \mathcal{A})\sum_{i \in \mathcal{A}}\|\nabla_{\bm{\omega}_{i}} \tilde{\mathcal{L}}_{0} ( \bm{\omega}^{k+1}, \bm{\theta}^{k+1},v^{k+1}) \|^2 \\
	&= \sum_{i=1}^{m}p_{i}\|\nabla_{\bm{\omega}_{i}} \tilde{\mathcal{L}}_{0} ( \bm{\omega}^{k+1}, \bm{\theta}^{k+1},v^{k+1}) \|^2 \geq \hat{p} \sum_{i=1}^{m}\|\nabla_{\bm{\omega}_{i}} \tilde{\mathcal{L}}_{0} ( \bm{\omega}^{k+1}, \bm{\theta}^{k+1},v^{k+1}) \|^2.
	\end{aligned}
	\end{equation*}
	Inserting this inequality into \eqref{eq:right0} and taking the total expectation $ \mathcal{F}_{k} $, we obtain
	\begin{equation*}
	\small
	\begin{aligned}
	& \quad \mathbb{E}[\|\nabla_{\bm{\omega}} \tilde{\mathcal{L}}_{0} ( \bm{\omega}^{k+1}, \bm{\theta}^{k+1},v^{k+1}) \|^2 ] \\
	& \leq \frac{1}{\hat{p}}[\frac{6[(c^{-\frac{T}{2}} - 1)^{-2}L_{h}^2 + \rho^2 ]}{\rho - L_{f} - 2(c^{-\frac{T}{2}} - 1)^{-1}L_{h}} + \frac{12\rho^3}{\rho^2 - L_{\tilde{g}}\rho - 2L_{\tilde{g}}^2}][\mathbb{E}\tilde{\mathcal{L}}_{\rho}(\bm{\omega}^{k}, \bm{\theta}^{k}, v^{k})- \mathbb{E}\tilde{\mathcal{L}}_{\rho}(\bm{\omega}^{k+1}, \bm{\theta}^{k+1}, v^{k+1})].
	\end{aligned}
	\end{equation*}
	Summing this inequality from $ k = 0 $ to $ K-1 $ and multiplying the result by $ \frac{1}{K} $, we get 
	\begin{equation*}
	\small
	\begin{aligned}  & \frac{1}{K}\sum_{k=0}^{K-1}\mathbb{E}[\|\nabla_{\bm{\omega}} \tilde{\mathcal{L}}_{0} ( \bm{\omega}^{k+1}, \bm{\theta}^{k+1},v^{k+1}) \|^2 ]  \leq \frac{1}{K\hat{p}}[\frac{6[(c^{-\frac{T}{2}} - 1)^{-2}L_{h}^2 + \rho^2 ]}{\rho - L_{f} - 2(c^{-\frac{T}{2}} - 1)^{-1}L_{h}} + \frac{12\rho^3}{\rho^2 - L_{\tilde{g}}\rho - 2L_{\tilde{g}}^2}](\tilde{\mathcal{L}}_{\rho}^{0} - \mathbb{E} \tilde{\mathcal{L}}_{\rho}^{K}),
	\end{aligned}
	\end{equation*}
	where $\tilde{\mathcal{L}}_{\rho}^{0} =  \tilde{\mathcal{L}}_{\rho}(\omega^{0}, \theta^{0}, v^{0})$. By Lemma \ref{lem:bounded}, we know $ \mathbb{E} \tilde{\mathcal{L}}_{\rho}^{K} \geq F^*$. $ \theta^{0} = v^{0} = 0 $, $ \omega_{1}^{0} = ... = \omega_{m}^{0} $, $\tilde{\mathcal{L}}_{\rho}^{0} =  \tilde{F}(\omega^0)$. By Proposition 1, we know $ \tilde{F}(\omega^0) \leq F(\omega^0)+ \frac{m \xi \lambda}{4}$. Finally, we have
	\begin{equation*}
	\small
	\begin{aligned}  \frac{1}{K}\sum_{k=0}^{K-1}\mathbb{E}[\|\nabla_{\bm{\omega}} \tilde{\mathcal{L}}_{0} ( \bm{\omega}^{k+1}, \bm{\theta}^{k+1},v^{k+1}) \|^2 ]  \leq \frac{C_{1}(F(\omega^0) - F^* + m \xi \lambda/4)}{K},
	\end{aligned}
	\end{equation*}
	where $ C_{1} =  \frac{1}{\hat{p}}[\frac{6[(c^{-\frac{T}{2}} - 1)^{-2}L_{h}^2 + \rho^2]}{\rho - L_{f} - 2(c^{-\frac{T}{2}} - 1)^{-1}L_{h}} + \frac{12\rho^3}{\rho^2 - L_{\tilde{g}}\rho - 2L_{\tilde{g}}^2}]$.
	
	(b)	First,
	\begin{equation*}
	\small
	\begin{aligned}
	\sum_{i=1}^{m}\sum_{j=1}^{m}\|\nabla_{\theta_{ij}} \tilde{\mathcal{L}}_{0} (\bm{\omega}^{k+1}, \bm{\theta}^{k+1}, v^{k+1})\|^2  = \frac{1}{4m^2} \sum_{i=1}^{m}\sum_{j=1}^{m} \|\nabla g(\|\bm{\theta}_{ij}^{k+1}\|) - v_{ij}^{k+1}\|^2 = 0.
	\end{aligned}
	\end{equation*}
	Taking the total expectation $ \mathcal{F}_{k} $, we obtain
	\begin{equation*}
	\mathbb{E}\|\nabla_{\bm{\theta}} \tilde{\mathcal{L}}_{0} (\bm{\omega}^{k+1},\bm{\theta}^{k+1},v^{k+1})\|^2 =  0.
	\end{equation*}		
	(c)
	By the update of $v_{ij}^k$ \eqref{eq:v}, we have  
	\begin{equation*}
	\begin{aligned}
	\nabla_{v_{ij}} \tilde{\mathcal{L}}_{0} (\bm{\omega}^{k+1}, \bm{\theta}^{k+1}, v^{k+1}) = \frac{1}{2m}(\matr{A}_{ij}\bm{\omega}^{k+1} - \bm{\theta}_{ij}^{k+1}) \overset{\eqref{eq:v}}{=} \frac{1}{2m\rho}(v_{ij}^{k+1} - v_{ij}^{k}).
	\end{aligned}	
	\end{equation*}
	Then, by \eqref{eq:theta v4}, we obtain  
	\begin{equation*}
	\small
	\sum_{i=1}^{m}\sum_{j=1}^{m}\|\nabla_{v_{ij}} \tilde{\mathcal{L}}_{0} (\bm{\omega}^{k+1}, \bm{\theta}^{k+1}, v^{k+1})\|^2  \leq \frac{L_{\tilde{g}}^2}{4m^2\rho^2}\sum_{i=1}^{m}\sum_{j=1}^{m}\|\bm{\theta}_{ij}^{k+1} - \bm{\theta}_{ij}^{k}\|^2.
	\end{equation*}
	By Lemma \ref{lem:descent of L},
	\begin{equation*}
	\small
	\begin{aligned}
	\sum_{i=1}^{m}\sum_{j=1}^{m}\|\nabla_{v_{ij}} \tilde{\mathcal{L}}_{0} (\bm{\omega}^{k+1}, \bm{\theta}^{k+1}, v^{k+1})\|^2  \leq \frac{L_{\tilde{g}}^2}{m\rho (\rho^2 - L_{\tilde{g}}\rho - 2L_{\tilde{g}}^2)}[\tilde{\mathcal{L}}_{\rho}(\bm{\omega}^{k}, \bm{\theta}^{k}, v^{k})- \tilde{\mathcal{L}}_{\rho}(\bm{\omega}^{k+1}, \bm{\theta}^{k+1}, v^{k+1})].
	\end{aligned}
	\end{equation*}
	Now, taking the total expectation $ \mathcal{F}_{k} $,
	\begin{equation}
	\small
	\begin{aligned}
	\mathbb{E}\|\nabla_{v} \tilde{\mathcal{L}}_{0} (\bm{\omega}^{k+1}, \bm{\theta}^{k+1}, v^{k+1})\|^2  \leq \frac{L_{\tilde{g}}^2}{m\rho (\rho^2 - L_{\tilde{g}}\rho - 2L_{\tilde{g}}^2)}[	\mathbb{E}\tilde{\mathcal{L}}_{\rho}(\bm{\omega}^{k}, \bm{\theta}^{k}, v^{k})
	-\mathbb{E} \tilde{\mathcal{L}}_{\rho}(\bm{\omega}^{k+1}, \bm{\theta}^{k+1}, v^{k+1})].
	\end{aligned}
	\end{equation}	
	Summing this inequality from $ k = 0 $ to $ K $ and multiplying the result by $ \frac{1}{K} $,
	\begin{equation*}
	\begin{aligned}
	\frac{1}{K} \sum_{k=0}^{K-1}\mathbb{E}\|\nabla_{v} \tilde{\mathcal{L}}_{0} (\bm{\omega}^{k+1},\bm{\theta}^{k+1},v^{k+1})\|^2 \leq  \frac{L_{\tilde{g}}^2}{m\rho  (\rho^2 - L_{\tilde{g}}\rho - 2L_{\tilde{g}}^2)K}(\tilde{\mathcal{L}}_{\rho}^{0} - \tilde{\mathcal{L}}_{\rho}^{*}) \leq \frac{C_{2}(F(\omega^0) - F^*+m \xi \lambda/4)}{K},
	\end{aligned}
	\end{equation*}
	where $ C_{2} = \frac{L_{\tilde{g}}^2}{m\rho  (\rho^2 - L_{\tilde{g}}\rho - 2L_{\tilde{g}}^2)}$.
	
	For the smoothed SCAD penalty, define $	\tilde{\mathcal{G}}(\omega, \theta, v)= [\theta - \text{prox}_{ \tilde{\mathcal{L}}_{0}}(\theta)]$. By the definition of $\tilde{\mathcal{L}}_{0}$, we have
	\begin{equation}
	\begin{aligned}
	\text{prox}_{ \tilde{\mathcal{L}}_{0}}(\bm{\theta}_{ij}) &= \arg \min_{x}[\tilde{g}(x) - <v_{ij}, x> + \frac{1}{2} \|x - \bm{\theta}_{ij}\|^2]\\
	&= \arg \min_{x}[\tilde{g}(x) + \frac{1}{2} \|x - (\bm{\theta}_{ij} +  v_{ij})\|^2]= \text{prox}_{\tilde{g}}(\bm{\theta}_{ij} +  v_{ij}).
	\end{aligned}
	\end{equation}
	Similarly, we have $\text{prox}_{ \mathcal{L}_{0}}(\bm{\theta}_{ij}) = \text{prox}_{ g}(\bm{\theta}_{ij} +  v_{ij})  $. Then, the difference between $ \tilde{\mathcal{G}}_{}(\omega, \theta, v) $ and $\mathcal{G}_{}(\omega, \theta, v) $ is
	\begin{equation}
	\begin{aligned}
	\|\tilde{\mathcal{G}}_{}(\omega, \theta, v) -  \mathcal{G}_{}(\omega, \theta, v) \|^2=\sum_{i=1}^{m}\sum_{j=1}^{m} \|\text{prox}_{ \tilde{g}}(\bm{\theta}_{ij} +  v_{ij})- \text{prox}_{ g}(\bm{\theta}_{ij} +  v_{ij}) \|^2.
	\end{aligned}
	\end{equation}
	Because
	\begin{equation}
	\small
	\hspace{-5mm}
	\begin{aligned}
	\text{prox}_{ \tilde{g}}(\bm{\theta}_{ij} +  v_{ij}) = \arg \min_{x}[\tilde{g}(x) + \frac{1}{2} \|x - (\bm{\theta}_{ij} +  v_{ij})\|^2] = 
	\begin{cases}
	\frac{\xi}{\lambda  + \xi}(\bm{\theta}_{ij} +  v_{ij}), &\|\bm{\theta}_{ij} +  v_{ij}\| \leq \lambda  + \xi, \\
	(1 - \frac{\lambda }{\|\bm{\theta}_{ij} +  v_{ij}\|})(\bm{\theta}_{ij} +  v_{ij}),  &\lambda  + 
	\xi < \|\bm{\theta}_{ij} +  v_{ij}\| \leq \lambda  + \lambda, \\
	\frac{\max\{0,1 - \frac{a \lambda }{(a-1) \|\bm{\theta}_{ij} +  v_{ij}\|}\}}{1-1/[(a-1)\rho]}(\bm{\theta}_{ij} +  v_{ij}), &\lambda  + 
	\lambda < \|\bm{\theta}_{ij} +  v_{ij}\| \leq a \lambda, \\
	\bm{\theta}_{ij} +  v_{ij}, &\|\bm{\theta}_{ij} +  v_{ij}\|> a \lambda,
	\end{cases}
	\end{aligned}
	\end{equation}
	and 
	\begin{equation}
	\small
	\begin{aligned}
	\text{prox}_{ g}(\bm{\theta}_{ij} +  v_{ij}) = \arg \min_{x}[g(x) + \frac{1}{2} \|x - (\bm{\theta}_{ij} +  v_{ij})\|^2] = 
	\begin{cases}
	(1 - \frac{\lambda }{\|\bm{\theta}_{ij} +  v_{ij}\|})_{+}(\bm{\theta}_{ij} +  v_{ij}),  &\|\bm{\theta}_{ij} +  v_{ij}\| \leq \lambda  + \lambda, \\
	\frac{\max\{0,1 - \frac{a \lambda }{(a-1) \|\bm{\theta}_{ij} +  v_{ij}\|}\}}{1-1/[(a-1)\rho]}(\bm{\theta}_{ij} +  v_{ij}), &\lambda  + 
	\lambda < \|\bm{\theta}_{ij} +  v_{ij}\| \leq a \lambda, \\
	\bm{\theta}_{ij} +  v_{ij}, &\|\bm{\theta}_{ij} +  v_{ij}\|> a \lambda,
	\end{cases}
	\end{aligned}
	\end{equation}
	we have
	\begin{equation}
	\hspace{-5mm}
	\begin{aligned}
	\|\text{prox}_{ \tilde{g}}(\bm{\theta}_{ij} +  v_{ij})- \text{prox}_{ g}(\bm{\theta}_{ij} +  v_{ij}) \| = 
	\begin{cases}
	\| [\frac{\xi}{\lambda  + \xi} -  (1 - \frac{\lambda }{\|\bm{\theta}_{ij} +  v_{ij}\|})_{+}](\bm{\theta}_{ij} +  v_{ij}) \|, & \|\bm{\theta}_{ij} +  v_{ij}\| \leq \lambda  +\xi, \\
	0,& \text{else}.			
	\end{cases}
	\end{aligned}
	\end{equation}
	Consequently, we obtain 
	\[\|\tilde{\mathcal{G}}_{}(\omega, \bm{\theta}, v) -  \mathcal{G}_{}(\omega, \bm{\theta},v)  \|^2 \leq \frac{m^2\lambda^2 \xi^2}{(\lambda  + \xi)^2}. \]
	Since $ \mathbb{E}\|\nabla_{\bm{\theta}} \tilde{\mathcal{L}}_{0}(\bm{\omega}^{k+1}, \bm{\theta}^{k+1}, v^{k+1})\|^2 = 0 $, which is equivalent to $ \mathbb{E}\|\tilde{\mathcal{G}}_{}(\bm{\omega}^{k+1}, \bm{\theta}^{k+1}, v^{k+1})\|^2 = 0 $, together with Proposition 1, we obtain 
	\begin{equation*}
	\small
	\hspace{-5mm}
	\begin{aligned}
	\frac{1}{K}\sum_{k=0}^{K-1}\mathbb{E}[\|\nabla_{\bm{\omega}} \mathcal{L}_{0} ( \bm{\omega}^{k+1}, \bm{\theta}^{k+1},v^{k+1}) \|^2 ]  &\leq \frac{C_{1}(F(\omega^0) - F^* + m \xi \lambda/4)}{K},\\
	\mathbb{E}\|\mathcal{G}(\bm{\omega}^{k+1},\bm{\theta}^{k+1},v^{k+1})\|^2 &\leq  \frac{m^2\lambda^2 \xi^2}{(\lambda  + \xi)^2},\\
	\frac{1}{K} \sum_{k=0}^{K-1}\mathbb{E}\|\nabla_{v} \mathcal{L}_{0} (\bm{\omega}^{k+1},\bm{\theta}^{k+1},v^{k+1})\|^2
	& \leq \frac{C_{2}(F(\omega^0) - F^*+ m \xi \lambda/4)}{K}.
	\end{aligned}		
	\end{equation*}
\end{proof}

\section{Statistical convergence rate under a linear model}
\label{app:D}
\setcounter{equation}{0}
\renewcommand\theequation{D.\arabic{equation}}

\subsection{Notations and preparation}
We provide the prerequisite notations, definitions, and conditions in the following. Note that some of them have been provided in the main text, but we also provide them here for completeness. 

Denote the true values of parameters as $\bm{\omega}^{true}=({\bm\omega_1^{true}}^{\top},\ldots,{\bm\omega_m^{true}}^{\top})^\top$. Let $\bm\alpha^{true}=({\bm\alpha^{true}_1}^{\top},\ldots,{\bm\alpha^{true}_L}^{\top})^{\top}$ be the distinct values of $\bm{\omega}^{true}$, and then $\bm\omega_i^{true}=\bm\alpha_l^{true}$ for $i\in G_l$. Denote the minimal differences of the common model parameters between two clusters as 
\begin{equation*}
b=\min_{i\in G_l,j\in G_{l'},l\neq l'}\|\bm\omega_i^{true}-\bm\omega_j^{true}\|=\min_{l\neq l'}\|\bm\alpha_l^
{true}-\bm\alpha_{l'}^{true}\|.	
\end{equation*}
Let $n=\sum_{i=1}^m{n_i}$, $n_{G_l}=\sum_{i\in {G}_l}n_i$, $ n_{\min}=\min_{l\in [L]} n_{G_l} $ and $n_{\max}=\max_{l\in [L]} n_{G_l}$. 
Write $|G_{\min}|=\min_{l\in [L]}|G_{l}|$, where $|\cdot|$ is the cardinality of the set. For a matrix $ \matr{A} $, denote its smallest / largest eigenvalue as $ \lambda_{\min}(\matr{A}) $ / $  \lambda_{\max}(\matr{A}) $, $ \ell_{2} $-norm as $ \|\matr{A}\| = \sqrt{\lambda_{\max}(\matr{A}^\top \matr{A})} $, and Frobenius norm as $ \|\matr{A}\|_{F} = \sqrt{\text{tr}(\matr{A}^\top\matr{A})}$.

Define $\mathcal{M}_G=\{\bm\omega\in\mathbb{R}^{md}: \bm\omega_i=\bm\omega_j, \text{for any}\ i,j\in {G}_l, l \in [L]\}$. Let $\tilde{\bm Z}=\{Z_{il}\}$ be a $m\times L$ matrix with  
$Z_{il}=1$ for $i \in G_l$ and $Z_{il}=0$ otherwise. Let $\bm Z=\tilde{\bm Z}\otimes \bm I_d$. Therefore, each $\bm \omega \in \mathcal{M}_G$ can be rewritten as $\bm\omega = \bm Z \bm \alpha$, where $\bm \alpha=(\bm\alpha_1^\top, \ldots, \bm \alpha_L^\top)^\top$ and $\bm\alpha_l$ is the $d\times 1$ vector of the $l$-th cluster-specific parameters for $l=1,\ldots,L$. For $i\in [m]$, write $\bm X_i=((\bm X_i^1)^\top,\ldots,(\bm X_i^{n_i})^\top)^\top$, $\bm y_i=(y_i^s,\ldots,y_{n_i}^s)^\top$, and $\bm \tau_i=(\tau_i^1,\ldots,\tau_i^{n_i})^\top$. Define $\bm\tau_{G_l}=\{\bm\tau_i^\top,i\in G_l\}^\top$ for $l\in [L]$. Denote the $j$-th column of $\bm X_i$ as $\bm X_{ij}$, and then $\bm X_i=(\bm X_{i1},\ldots,\bm X_{id})$. Let $\bm{X}=\text{diag}(n_1^{-1/2}\bm X_1,\cdots,n_m^{-1/2}\bm X_m)$, $\bm y =(n_1^{-1/2}\bm y_{1}^{\top},\cdots,n_m^{-1/2}\bm y_{m}^{\top})^{\top}$, $\bm\tau =(n_1^{-1/2}\bm\tau_{1}^{\top},\cdots,n_m^{-1/2}\bm\tau_{m}^{\top})^{\top}$. Then, the linear model (9) and loss function $f(\omega)$ can be written compactly as follows: $\bm y=\bm X\bm \omega^{true}+\bm \tau$ and $f(\bm\omega)=\|\bm y-\bm X\bm\omega\|^2$.  

Consider the oracle estimator $\bm{\hat\omega^{\text{or}}}$ for $ \bm{\omega}$, under which the underlying clusters of devices $G_1,\ldots, G_L$ are known. Specifically,
\begin{equation}
\label{eq:alpha_or}
\bm{\hat\omega^{\text{or}}}=\underset{\bm{\omega}\in\mathcal{M}_{G}}{\arg\min}\|\bm{y}-\bm{X}\bm{\omega}\|^2.
\end{equation}
Correspondingly, the oracle estimator for the common parameters $\bm \alpha$ is 
\[\bm{\hat\alpha^{\text{or}}}=\underset{\bm{\alpha}\in\mathbb{R}^{Ld}}{\arg\min}\|\bm{y}-\bm{X}\bm{Z}\bm{\alpha}\|^2.\]
For simplicity, we denote the scaled SCAD penalty by $\rho(t)=\lambda^{-1}P_{a}(t,\lambda)$. Note that $ \rho(t) $ is nondecreasing and concave in $ t$ for $t \in [0, \infty)$ and $\rho(0)=0$. When $ |t| > a \lambda $, $\rho(t)$ is a constant. $\rho'(t)$ exists and is continuous except for a finite number of $t$ and $\rho'(0+)=1$. 

Let $P(\bm\omega,\lambda)=\frac{\lambda}{2m}\sum_{i=1}^{m}\sum_{j=1}^{m}\rho(\Vert\bm\omega_i-\bm\omega_j\Vert)$. Then, we can rewrite the objective function $F(\bm{\omega})$ as
\begin{equation}
\label{eq:obj1}
F(\bm{\omega})=f(\bm{\omega})+P(\bm{\omega}, \lambda).
\end{equation}
The following conditions are assumed:  
\begin{itemize}
	\item[(1)]  There exists a constant $C>0$ such that $\lambda_{\min}(\bm X_i^{\top}\bm X_i)\geq Cn_i$ for $i\in [m]$.
	\item[(2)] $\Vert\bm X_{ij}\Vert =\sqrt{n_i}$ for $i \in [m]$, $j \in [d]$.
	\item[(3)] For $l\in [L]$, the noise vector $\bm\tau_{G_{l}}$ has sub-Gaussian tails such that $P(|\langle \bm{c}, \bm{\tau}_{{G}_l}\rangle|>\Vert\bm{c}\Vert x)\leq 2e^{-c_1 x^2}$ for any vector $\bm{c}\in \mathbb{R}^{|G_l|}$ and $x>0$, where $0<c_1<\infty$ is a constant.
\end{itemize}
These three conditions are common assumptions in penalized regression in high-dimensional settings. Under these mild assumptions, we can establish the statistical convergence rate.

\subsection{Proof of Theorem 3}
The proof of Theorem 3 is equivalent to establishing the following two results. 

{\bf Result 1}: Suppose Conditions (1)-(3) hold. Then, we have
\[P(\underset{i\in [m]}{\sup}\Vert{\hat\omega}_i^{\text{or}}-\bm\omega_i^{true}\Vert \leq\Lambda_n)\geq 1-2Ldn^{-n_{\min} n_{\max}^{-1}},\]
where $\Lambda_n=C^{-1} |{G}_{\min}|^{-1}c_1^{-1/2}\sqrt{2n_{\min}^{-1}Ld\log n }$.

{\bf Result 2}: Suppose $b>a\lambda$, $\lambda\gg m |{G}_{\min}|^{-1}\sqrt{n_{\min}^{-1}Ld^3\log n }+m\sqrt{n_{\min}^{-1}d\log n }$, and the conditions in Result 1 hold. Then, there exists a local minimizer $ \bm{\omega}^{*} $ of $ F(\bm\omega) $ such that 
\begin{equation*}
P({\bm \omega}^{*}= \hat{\bm{\omega}}^{\text{or}})\geq 1-2Ldn^{-n_{\min}n_{\max}^{-1}}-2dn^{-1}.
\end{equation*}

\paragraph{Proof of Result 1.}\label{proofres1}
Recall that 
\begin{equation*}
\begin{aligned}
{\hat\alpha^{\text{or}}}= \underset{\bm{\alpha}\in\mathbb{R}^{Ld}}{\arg\min}\|\bm{y}-\bm{X}\bm{Z}\bm{\alpha}\|^2 = (\bm{Z^{\top}}\bm{X^{\top}}\bm{X}\bm{Z})^{-1}(\bm{Z^{\top}}\bm{X^{\top}})\bm{y}. 
\end{aligned}
\end{equation*}
Then, using the relation $\bm y=\bm X\bm \omega+\bm \tau$, we have 
\begin{equation*}
{\hat \alpha^{\text{or}}}-\bm\alpha^{true}=(\bm{Z^{\top}}\bm{X^{\top}}\bm{X}\bm{Z})^{-1}\bm{Z^{\top}}\bm{X^{\top}}\bm{\tau},
\end{equation*}
and 
\begin{equation}
\label{eq:alpha_or-alpha_true}
\Vert{\hat\alpha^{\text{or}}}-\bm\alpha^{true}\Vert\leq\Vert(\bm{Z^{\top}}\bm{X^{\top}}\bm{X}\bm{Z})^{-1}\Vert\Vert\bm{Z^{\top}}\bm{X^{\top}}\bm{\tau}\Vert.
\end{equation}
By Condition (1),
\begin{eqnarray}\label{21}
\Vert(\bm{Z^{\top}}\bm{X^{\top}}\bm{X}\bm{Z})^{-1}\Vert = \lambda_{
	\min}^{-1}(\bm{Z^{\top}}\bm{X^{\top}}\bm{X}\bm{Z})\leq C^{-1} |G_{\min}|^{-1}.
\end{eqnarray}
Moreover, 
\begin{equation}
\label{eq:inftynorm}
\Vert\bm{Z^{\top}}\bm{X^{\top}}\bm{\tau}\Vert_{\infty}=\underset{j \in [d], l\in [L]}{\sup}|\sum_{i=1}^m \bm X_{ij}^{\top}\bm\tau_i I(i\in G_{l})|. 
\end{equation}
Write $\bm X_{G_l,j}=\{\bm X_{ij}^\top,i \in G_l\}$ for $l \in [L]$. By \eqref{eq:inftynorm}, we have 
\begin{equation}
\small
\begin{aligned}
P(\Vert\bm{Z^{\top}}\bm{X^{\top}}\bm{\tau}\Vert_{\infty}>c_1^{-\frac{1}{2}}\sqrt{2n_{\min}^{-1}\log n})&\overset{\eqref{eq:inftynorm}}{\leq}\sum_{j=1}^{d}\sum_{l=1}^{L}P(|\sum_{i=1}^m n_i^{-1}\bm X_{ij}^{\top}\bm{\tau}_iI(i\in G_{l})|>c_1^{-\frac{1}{2}}\sqrt{2n_{\min}^{-1}\log n})\nonumber\\
&\leq \sum_{j=1}^{d}\sum_{l=1}^{L}P(n_{\min}^{-1}|\bm X_{G_{l},j}^{\top}\bm{\tau}_{G_{l}}|>c_1^{-\frac{1}{2}}\sqrt{2n_{\min}^{-1}\log n})\nonumber\\
&\leq \sum_{j=1}^{d}\sum_{l=1}^{L}P(|\bm X_{G_{l},j}^{\top}\bm{\tau}_{G_{l}}|>\Vert \bm X_{G_{l},j}\Vert n_{\min}^{1/2}n_{\max}^{-1/2}c_1^{-\frac{1}{2}}\sqrt{\log n})\nonumber\\
&\leq 2Ld e^{(-n_{\min}n_{\max}^{-1}\log n)} \nonumber= 2Ldn^{-n_{\min} n_{\max}^{-1}},\nonumber
\end{aligned}
\end{equation}
where the third inequality follows from $\Vert\bm X_{G_{l},j}\Vert \leq\sum_{i\in G_{l}}\Vert\bm X_{ij}\Vert \leq\sqrt{2n_{\max}}$.

Note that $\Vert\bm{Z^{\top}}\bm{X^{\top}}\bm{\tau}\Vert\leq\sqrt{Ld}\Vert\bm{Z^{\top}}\bm{X^{\top}}\bm{\tau}\Vert_{\infty}$. Then,
\begin{equation}
\label{22}
\begin{aligned}
P(\Vert\bm{Z^{\top}}\bm{X^{\top}}\bm{\tau}\Vert>C_0\sqrt{2Ldn_{\min}^{-1}\log n})\leq P(\Vert\bm{Z^{\top}}\bm{X^{\top}}\bm{\tau}\Vert_{\infty}>c_1^{-\frac{1}{2}}\sqrt{2n_{\min}^{-1}\log n})\leq 2Ldn^{-n_{\min} n_{\max}^{-1}}.
\end{aligned}
\end{equation}
By \eqref{eq:alpha_or-alpha_true}, \eqref{21}, and \eqref{22}, with probability at least $1-2Ldn^{- n_{\min} n_{\max}^{-1}}$, we have
\begin{align*}
\underset{i \in [m]}{\sup}\Vert{\hat\omega}_i^{\text{or}}-\bm\omega_i^{true}\Vert\leq\underset{l \in [L]}{\sup}\Vert{\hat\alpha}_l^{\text{or}}-\bm\alpha_l^{true}\Vert\leq\Vert{\hat\alpha^{\text{or}}}-\bm\alpha^{true}\Vert\leq\Lambda_n,
\end{align*}
where $\Lambda_n=C^{-1} |G_{\min}|^{-1}c_1^{-1/2}\sqrt{2Ldn_{\min}^{-1}\log n}$.
This completes the proof. 

\paragraph{Proof of Result 2}\label{proofres2}
Let $T:\mathcal{M}_{G}\rightarrow \mathbb{R}^{Ld}$ be the mapping such that $T(\bm\omega)$ is the $Ld\times 1$ vector consisting of $L$ vectors with dimension $d$, and its $l$-th vector component equals the common value of $\bm\omega_i$ for $i\in G_{l}$. Let $T^*:\mathbb{R}^{md}\rightarrow\mathbb{R}^{Ld}$ be the mapping such that $T^*(\bm{\omega})=\{\sum_{i\in G_{l}}\frac{n_i}{n_{G_{l}}}\bm{\omega}_i^\top, l \in [L]\}^\top$, where $n_{G_{l}}=\sum_{i\in G_{l}}n_i$. Obviously, when $\bm\omega\in\mathcal{M}_{G}$, $T(\bm\omega)=T^*(\bm\omega)$. For any $\bm{\omega}\in\mathbb{R}^{md}$, let $T^*(\bm{\omega})=\bm\alpha=({\bm\alpha_1}^{\top},\ldots,{\bm\alpha_L}^{\top})^{\top}$ and $\bm{\omega}^{inv}=T^{-1}(T^*(\bm{\omega}))=T^{-1}(\bm\alpha)$.

Consider the neighborhood of $\bm\omega^{true}$: 
$$\Theta=\{\bm{\omega}\in\mathbb{R}^{nd}: \underset{i \in [m]}{\sup}\Vert\bm\omega_i-\bm\omega_i^{true}\Vert \leq\Lambda_n.\}$$
By Result 1, there exists an event $E_1$ such that $\underset{i \in [m]}{\sup}\Vert{\hat\omega}_i^{\text{or}}-\bm\omega_i^{true}\Vert \leq\Lambda_n$ and $P(E_1^C)\leq 2Ldn^{-n_{\min} n_{\max}^{-1}}$. Thus, ${\hat\omega}^{\text{or}}\in\Theta$ in $E_1$.
We show that ${\hat\omega}^{\text{or}}$ is a strict local minimizer of the objective function $ F(\bm{\omega})$ with probability at least $ 1-2Ldn^{-n_{\min}n_{\max}^{-1}}-2dn^{-1}$ through the following two steps: \\
\noindent(i) In the event $E_1$, $F(\bm{\omega}^{inv})>F({\hat\omega}^{\text{or}})$ for any $\bm\omega\in\Theta$ and $\bm{\omega}^{inv}\neq{\hat\omega}^{\text{or}}$. \\
\noindent(ii) There is an event $E_2$ such that $P(E_2^C)<2dn^{-1}$. In $E_1\cap E_2$, $F(\bm{\omega})\geq F(\bm{\omega}^{inv})$ for any $\bm\omega\in\Theta$ for sufficiently large $n$.

With the results in (i) and (ii), for any $\bm\omega\in \Theta$ and $\bm\omega \neq {\hat\omega}^{\text{or}}$ in $E_1\cap E_2$, we have $ F(\bm\omega)> F({\hat\omega}^{\text{or}})$, so ${\hat\omega}^{\text{or}}$ is a strict local minimizer of $ F(\bm\omega)$ over the event $E_1\cap E_2$ with $P(E_1\cap E_2)\geq 1-2Ldn^{-n_{\min}n_{\max}^{-1}}-2dn^{-1}$ for sufficiently large $n$. 

We first prove the results in (i). Because ${\hat\omega}^{\text{or}}$ is the unique global minimizer of $f(\bm{\omega})$ for $\bm\omega\in\mathcal{M}_{G}$, $f(\bm{\omega}^{inv})>f({\hat\omega}^{\text{or}})$, so we need to consider $P(\bm{\omega}^{inv},\lambda)$. Because 
\begin{equation*}
\begin{aligned}
\Vert\bm\alpha_l-\bm\alpha_{l'}\Vert \geq\Vert\bm\alpha_l^{true}-\bm\alpha_{l'}^{true}\Vert-\Vert\bm\alpha_l-\bm\alpha_l^{true}\Vert-\Vert\bm\alpha_{l'}^{true}-\bm\alpha_{l'}\Vert\geq \Vert\bm\alpha_l^{true}-\bm\alpha_{l'}^{true}\Vert-2\underset{l \in [L]}{\sup}\Vert\bm\alpha_l-\bm\alpha_{l'}^{true}\Vert
\end{aligned}
\end{equation*}
and
\begin{equation}
\label{eq:alpha2}
\begin{aligned}
\sup_{l\in [L]}\|\bm\alpha_l-\bm\alpha_{l'}^{true}\|^2 
&= \sup_{l\in [L]}\||G_l|^{-1}\sum_{i\in G_l}\bm\omega_i-\bm\alpha_l^{true}\|^2 \\
&=\sup_{l\in [L]}\||G_l|^{-1}\sum_{i\in G_l}(\bm\omega_i-\bm\omega_i^{true})\|^2\\
& = \sup_{l\in [L]}|G_l|^{-2}\|\sum_{i\in G_l}(\bm\omega_i-\bm\omega_i^{true})\|^2\\
&\leq \sup_{l\in [L]}|G_l|^{-1}\sum_{i\in G_l}\|\bm\omega_i-\bm\omega_i^{true})\|^2\\
& \leq \sup_{l\in [L]}\|\bm\omega_i-\bm\omega_i^{true}\|^2 \leq \Lambda_n^2, 
\end{aligned}
\end{equation}
for all $l, l'\in [L]$, we have 
\begin{equation*}
\Vert\bm\alpha_l-\bm\alpha_{l'}\Vert \geq b-2\Lambda_n \geq a\lambda, 
\end{equation*}
where the last inequality follows from the assumption that $b>a\lambda\gg\Lambda_n$. Consequently, $\rho(\Vert\bm\alpha_l-\bm\alpha_{l'}\Vert)$ is a constant by the property of the SCAD function. Also note that 
\begin{equation*}
\begin{aligned}
P(\bm\omega^{inv},\lambda) =\frac{\lambda}{2m}\sum_{i=1}^{m}\sum_{j=1}^{m}\rho(\Vert\bm\omega_i^{inv}-\bm\omega_j^{inv}\Vert)=\frac{\lambda}{2m}\sum_{l\neq l'}{|G_{l}||G_{l'}|}\rho(\Vert\bm\alpha_l-\bm\alpha_{l'}\Vert).
\end{aligned}
\end{equation*}
So $P(\bm\omega^{inv},\lambda)$ is a constant. Therefore, $f(\bm{\omega}^{inv})+P(\omega^{inv},\lambda)>f({\hat\omega}^{\text{or}})+P(\hat{\bm\omega}^{\text{or}},\lambda)$, i.e., $F({\omega}^{inv})>F({\hat\omega}^{\text{or}})$ for all $\bm{\omega}^{inv}\neq {\hat\omega}^{\text{or}}$, and the result in (i) is proved. 

Next we prove the result in (ii). For $\bm\omega\in\Theta$, by Taylor's expansion, we have 
\[ F(\bm{\omega})- F(\bm{\omega}^{inv})=\Gamma_1+\Gamma_2,\]
where 
$\Gamma_1=-\langle \bm y-\bm{X} \tilde{\bm{\omega}}, \bm X(\bm\omega-\bm\omega^{inv})\rangle,\Gamma_2=\frac{\partial P(\bm{\omega},\lambda)}{\partial\bm{\omega}}\big|_{\bm{\omega}=\tilde{\bm{\omega}}}(\bm\omega-\bm{\omega}^{inv}),$
and $\tilde{\bm{\omega}}=c\bm{\omega}+(1-c)\bm{\omega}^{inv}$ for some constant $c\in(0,1)$. 

First, we consider $\Gamma_2$. We have  
\begin{equation}
\small
\hspace{-5mm}
\begin{aligned}
\Gamma_2 =\frac{\partial P(\bm{\omega},\lambda)}{\partial\bm{\omega}}\big|_{\bm{\omega}=\tilde{\bm{\omega}}}(\bm{\omega}-\bm{\omega}^{inv}) =\frac{\lambda}{2m}\sum_{i \ne j}\rho'(\Vert\tilde{\bm{\omega}}_{i}-\tilde{\bm{\omega}}_{j}\Vert)\langle\frac{\tilde{\bm{\omega}}_{i}-\tilde{\bm{\omega}}_{j}}{\Vert\tilde{\bm{\omega}}_{i}-\tilde{\bm{\omega}}_{j}\Vert} \left[(\bm\omega_i-\bm\omega_i^{inv})-(\bm\omega_j-\bm\omega_j^{inv})\right]\rangle.
\end{aligned}
\end{equation}
When $i,j\in G_{l}$, $\bm\omega_i^{inv}=\bm\omega_j^{inv}$ and $\tilde{\bm{\omega}}_{i}-\tilde{\bm\omega}_j=c(\bm\omega_i-\bm\omega_j)$. Thus, 
\begin{equation*} 
\small
\begin{aligned}
\Gamma_2  = &\frac{\lambda}{2m}\sum_{l=1}^{L} \sum_{\{i,j\in G_{l},i \ne j\}}\rho'(\Vert\tilde{\bm{\omega}}_{i}-\tilde{\bm\omega}_{j}\Vert)\frac{(\tilde{\bm{\omega}}_{i}- \tilde{\bm\omega}_{j})^\top}{\Vert\tilde{\bm{\omega}}_{i}-\tilde{\bm\omega}_{j}\Vert} \left[(\bm\omega_i-\bm\omega_i^{inv})-(\bm\omega_j-\bm\omega_j^{inv})\right]\\
&+ \frac{\lambda}{2m}\sum_{l\neq l'} \sum_{\{i\in G_l, j\in G_{l'}\}}\rho'(\Vert\tilde{\bm{\omega}}_{i}-\tilde{\bm\omega}_{j}\Vert)\frac{(\tilde{\bm{\omega}}_{i}- \tilde{\bm\omega}_{j})^\top}{\Vert\tilde{\bm{\omega}}_{i}-\tilde{\bm\omega}_{j}\Vert}\left[(\bm\omega_i-\bm\omega_i^{inv})-(\bm\omega_j-\bm\omega_j^{inv})\right].
\end{aligned}
\end{equation*}
Note that $\underset{i \in [m]}{\sup}\Vert\bm\omega_i^{inv}-\bm\omega_i^{true}\Vert= \underset{l \in [L]}{\sup}\Vert\bm\alpha_l-\bm\alpha_{l}^{true}\Vert \leq\Lambda_n $, which follows from \eqref{eq:alpha2}. Then, \begin{equation}\label{eq:Lambda_n}
\underset{i \in [m]}{\sup}\Vert\tilde{\bm{\omega}}_{i}-\bm\omega_i^{true}\Vert \leq c  \underset{i \in [m]}{\sup}\Vert\bm{\omega}_{i}-\bm\omega_i^{true}\Vert+(1-c) \underset{i \in [m]}{\sup}\Vert \bm{\omega}_{i}^{inv}-\bm\omega_i^{true}\Vert.
\end{equation}
Hence, for $i\in G_{l}$, $j\in G_{l'}$, $l\neq l'$, 
\begin{equation}
\begin{aligned}
\Vert\tilde{\bm{\omega}}_{i}-\tilde{\bm{\omega}}_{j}\Vert&\geq \min\limits_{i\in G_{l}, j\in G_{l'}}\Vert\bm\omega_i^{true}-\bm\omega_j^{true}\Vert-2\underset{k}{\sup}\Vert \tilde{\bm{\omega}}_{k}-\bm\omega_k^{true}\Vert \\\nonumber
&=\Vert\bm\alpha_l^{true}-\bm\alpha_{l'}^{true}\Vert-2\underset{k}{\sup}\Vert \tilde{\bm\omega}_k-\bm\omega_k^{true}\Vert \geq b-2\Lambda_n>a\lambda,
\end{aligned}
\end{equation}
and thus, $\rho'(\Vert\tilde{\bm{\omega}}_{i}-\tilde{\bm{\omega}}_{j}\Vert)=0$. Consequently, 
\begin{equation*} 
\begin{aligned}
\Gamma_2 &= \frac{\lambda}{2m}\sum_{l=1}^{L} \sum_{\{i,j\in G_{l},i \ne j\}}\rho'(\Vert\tilde{\bm{\omega}}_{i}-\tilde{\bm\omega}_{j}\Vert)\frac{(\tilde{\bm{\omega}}_{i}- \tilde{\bm\omega}_{j})^\top}{\Vert\tilde{\bm{\omega}}_{i}-\tilde{\bm\omega}_{j}\Vert} \left[(\bm\omega_i-\bm\omega_i^{inv})-(\bm\omega_j-\bm\omega_j^{inv})\right]\\
& =\frac{\lambda}{2m}\sum_{l=1}^{L}\sum_{\{i,j\in G_{l},i \ne j\}}\rho'(\Vert \tilde{\bm\omega}_i- \tilde{\bm\omega}_j \Vert)\Vert\bm\omega_i-\bm\omega_j\Vert,
\end{aligned}
\end{equation*}
where the last equality follows from $\tilde{\bm{\omega}}_{i}-\tilde{\bm\omega}_j=c(\bm\omega_i-\bm\omega_j)$. 
Furthermore, by the same reasoning as \eqref{eq:alpha2}, we have 
\begin{equation*}
\sup_{i\in[m]}\|\bm\omega_i^{inv}-\bm\omega_i^{true}\|=\sup_{l\in [L]}\|\bm\alpha_l-\hat{\bm\alpha}_l^{\text{or}}\|\leq \sup_{i\in [m]}\|\bm\omega_i-\hat{\omega}_i^{\text{or}}\|.
\end{equation*}
Then, it can be shown that 
\begin{equation*}
\begin{aligned}
\sup_{i\in [m]}\|\tilde{\bm\omega}_i-\tilde{\bm\omega}_j\|&\leq 2 \sup_{i\in [m]}\|\tilde{\bm\omega}_i-\bm\omega_i^{inv}\|\leq 2 \sup_{i\in [m]}\|{\bm\omega}_i-\bm\omega_i^{inv}\|\\
&\leq 2( \sup_{i\in [m]}\|{\bm\omega}_i-\hat{\bm\omega}_i^{\text{or}}\|+\sup_{i\in [m]}\|{\bm\omega}_i^{inv}-\bm\omega_i^{\text{or}}\|)\\
&\leq 4 \sup_{i\in [m]}\|{\bm\omega}_i-\hat{\bm\omega}_i^{\text{or}}\| \leq 4\Lambda_n.
\end{aligned}
\end{equation*}
Hence, $\rho'(\|\tilde{\bm\omega}_i-\tilde{\bm\omega}_j\|)\geq \rho'(4\Lambda_n)$ due to the concavity of $\rho(\cdot)$. As a result, 
\begin{equation}\label{eq:gamma3}
\begin{aligned}
\Gamma_2 \geq\frac{\lambda}{2m}\sum_{l=1}^{L}\sum_{\{i,j\in G_{l},i \ne j\}}\rho'(4\Lambda_n)\Vert\bm\omega_i-\bm\omega_j\Vert \geq\frac{\lambda}{m}\sum_{l=1}^{L}\sum_{\{i,j\in G_{l},i < j\}}\rho'(4\Lambda_n)\Vert\bm\omega_i-\bm\omega_j\Vert.
\end{aligned}
\end{equation}
Next, we consider $\Gamma_1$. 
Let $\tilde{v}=(\tilde{v}_1^\top,\cdots,\tilde{v}_m^\top)^\top=((\bm y-\bm{X}\tilde{\bm \omega})^\top\bm X)^\top$. Then, we have  
\begin{equation}\label{eq:gamma1}
\begin{split}
\Gamma_1  =-\tilde{v}^\top(\bm{\omega}-\bm{\omega}^{inv})
&=-\sum_{i=1}^{m}\tilde{v}_i^\top(\bm\omega_i-\bm\omega_i^{inv})\\
&=-\sum_{l=1}^{L}\sum_{\{i,j\in G_{l}, i<j\}}\frac{(n_j \tilde{v}_i-n_i \tilde{v}_{j})^\top(\bm\omega_i-\bm\omega_j)}{n_{ G_{l}}}\\
&\geq -\sum_{l=1}^{L}\sum_{\{i,j\in G_{l}, i<j\}}\frac{(n_j\Vert \tilde{v}_i\Vert + n_i\Vert \tilde{v}_j\Vert)\cdot\Vert\bm\omega_i-\bm\omega_j\Vert}{n_{ G_{l}}},
\end{split}
\end{equation}
where the second equality follows from $\bm{\omega}^{inv}=T^{-1}(T^*(\bm{\omega}))$ and $\sum_{i=1}^{m}\tilde{v}_i^\top \bm{\omega}_i^{inv} = \sum_{l=1}^L\sum_{i,j\in G_{l}}\frac{n_j}{n_{ G_{l}}}\tilde{v}_i^\top\bm{\omega}_j$. Moreover, 
\[\tilde{v}_i = n_i^{-1}(\bm X_i^\top \bm y_i - \bm X_i^\top\bm{X}_i \tilde{\bm{\omega}}_i)=n_i^{-1}(\bm X_i^\top\bm{X}_i(\bm{\omega}_i^{true}-\tilde{\bm{\omega}}_i)+\bm X_i^\top \bm\tau_i),\] 
and then
\begin{equation}\label{eq:vi}
\begin{split}
\Vert \tilde{v}_i\Vert&\leq n_i^{-1} (\Vert\bm X_i^\top\bm{X}_i\Vert_F\Vert\bm{\omega}_i^{true}-\tilde{\bm{\omega}}_i \Vert+\Vert\bm X_i^\top\bm\tau_i\Vert)\leq n_i^{-1}(\Vert\bm X_i^\top\bm{X}_i\Vert_F\Lambda_n+\Vert\bm X_i^\top\bm\tau_i\Vert).
\end{split}
\end{equation}
By Condition (2),
\[\Vert\bm X_i^\top\bm{X}_i\Vert_F=\sqrt{\sum_{l=1}^d \sum_{k=1}^d (\bm X_{il}^\top \bm X_{ik})^2}\leq dn_i.\]
By Condition (3),
\begin{equation}
\small
\begin{split}
P(\Vert\bm X_i^\top\bm\tau_i\Vert_{\infty}>\sqrt{c_1^{-1}n_i\log n}) &\leq\sum_{j=1}^{d}P(|\bm X_{ij}^\top\bm\tau_i|>\sqrt{c_1^{-1}n_i\log n})\leq \sum_{j=1}^{d}P(|\bm X_{ij}^\top\bm\tau_i|>\Vert\bm X_{ij}\Vert \sqrt{c_1^{-1}\log n})\leq 2dn^{-1}.
\end{split}
\end{equation}
Note that $\Vert\bm X_i^\top\bm\tau_i\Vert\leq\sqrt{d}\Vert\bm X_i^\top\bm\tau_i\Vert_{\infty}$, so
$P(\Vert\bm X_i^\top\bm\tau_i\Vert>\sqrt{c_1^{-1}dn_i\log n})\leq2dn^{-1}.$
Thus, there is an event $E_2$ such that $P(E_2^C)\leq 2dn^{-1}$, and over the event $E_2$,
\begin{equation}\label{eq:vi2}
\begin{split}
\Vert \tilde{v}_i\Vert \leq \sqrt{d}(\sqrt{d}\Lambda_n+\sqrt{c_1^{-1}n_i^{-1}\log n})&=\sqrt{d}(C^{-1} |G_{\min}|^{-1}c_1^{-1/2}d\sqrt{2n_{\min}^{-1}L\log n}+\sqrt{c_1^{-1}n_i^{-1}\log n})\\
&=\mathcal{O} (|G_{\min}|^{-1}\sqrt{n_{\min}^{-1}Ld^3 \log n}+ \sqrt{n_{\min}^{-1}d\log n})
\end{split}
\end{equation}
Because $n_j\leq n_{ G_{l}}$ for $j\in G_{l}$, we have 
\begin{equation*}
\begin{aligned}
n_j n_{ G_{l}}^{-1} |G_{\min}|^{-1}\sqrt{n_{\min}^{-1}Ld^3\log n}+mn_j n_{ G_{l}}^{-1}\sqrt{n_{\min}^{-1}d\log n}  \leq |G_{\min}|^{-1}\sqrt{n_{\min}^{-1} Ld^3\log n}+m\sqrt{n_{\min}^{-1}d\log n}.\nonumber
\end{aligned}
\end{equation*}
By the assumption that  $\lambda\gg m|G_{\min}|^{-1}\sqrt{Ld^3n_{\min}^{-1}\log n}+m\sqrt{n_{\min}^{-1}d\log n}$, we have 
\begin{equation*}
\frac{\lambda}{2m}\gg\frac{n_j\Vert \tilde{v}_i\Vert+n_i\Vert \tilde{v}_j\Vert}{n_{G_{l}}}
\end{equation*}
for $i,j\in G_{l}$, $i\neq j$, and $l\in [L]$. Note that $\rho'(4\Lambda_n)\rightarrow 1$ for the SCAD penalty.
Therefore, for sufficiently large $n$, we have
\begin{equation}
\begin{split}
F(\bm{\omega})-F(\bm{\omega}^{inv})=\Gamma_1+\Gamma_2 &\geq\sum_{l=1}^{L}\sum_{\{i,j\in G_{l},i<j\}}\Vert\bm\omega_i-\bm\omega_j\Vert\left(\frac{\lambda\rho'(4\Lambda_n)}{m}-\frac{n_j\Vert \tilde{v}_i\Vert+n_i\Vert \tilde{v}_j\Vert}{n_{G_{l}}}\right)>0.\nonumber
\end{split}
\end{equation}

Therefore, in $E_1\cap E_2$, $F(\bm{\omega})\geq F(\bm{\omega}^{inv})$ for any $\bm\omega\in\Theta$. The result in (ii) is proved.

\section{Simulation details and additional numerical experiments}
\setcounter{figure}{0}
\setcounter{table}{0}

\renewcommand\thefigure{E.\arabic{figure}}
\renewcommand\thetable{E.\arabic{table}}

In this section, we provide more details on the settings of numerical experiments and more extensive experimental results to demonstrate the performance of our method compared to others. 
\subsection{Datasets and models}
\noindent\textbf{Synthetic.} Regarding cluster structures, we consider five scenarios: (S1) balanced with $L=4$, where each cluster has 25 devices; (S2) unbalanced with $L=4$, where the cluster sizes are 10, 40, 10 and 40; (S3) balanced with $L=2$, where each cluster has 50 devices; (S4) unstructured with $L=1$, where all $m=50$ devices have an identical optimal model and parameters; and (S5) personalized with $L=m=50$, where each device owns a local personalized model. We use the cross-entropy loss functions. 

\noindent\textbf{Housing and Body fat (H\&BF).} The Housing dataset \cite{HARRISON197881,Chen2018} has 506 samples and 13 features. The response vector $ y $ is the median house price, and features include some factors that may affect the price of housing, such as per capita crime rate. The Body fat dataset contains 252 samples and 14 features. The response vector $ y $ is the percentage of body fat, and features include various body measurements. For consistency of the number of features, we randomly generate a vector from a normal distribution as the 14-th feature of the housing dataset. 

\noindent \textbf{MNIST \& FMNIST.}   We use two public benchmark data sets, MNIST \cite{Lecun1998} and FMNIST (Fashion-MNIST) \cite{Xiao2017}. We followed the procedures described in \cite{Sattler2021} to create the practical non-IID data setting, which has a clear cluster structure. Let us take MNIST as an example. We first partitioned 60000 training images and 10000 testing images into $m=20$ devices by a Dirichlet distribution, each of which belonging to one of $ L = 4 $ clusters. We then modify every devices’ data by randomly swapping two labels, depending on which cluster a device belongs to. Specifically, the first cluster consists of devices 1-5, where each devices' data labeled as "0" could be relabeled as "8" and vice versa. The second cluster consists of devices 6-10, where each devices' data labeled as  "1" and "7" could be switched out similarly, and so on. The testing data are processed in the same way. We use a multilayer convolutional neural network with cross-entropy loss. To account for some common properties of the data from all devices, we adopt the weight sharing technique from multi-task learning \cite{Caruana1997}, as used in {IFCA} \cite{Ghosh2020}. Specifically, when we train the neural network models, we share the weights for the first three layers and run our method only on the last layer, producing a predicted label for each image.

\noindent\textbf{Machines and libraries.} All the methods are implemented in PyTorch \cite{Paszke2019} version 1.9.0 running on a public computing cloud with an Intel(R) Xeon Gold 5218 CPU and 192GB memory. 

\subsection{Additional experiments and results}\label{app:E}
\subsubsection{Synthetic data}  
\begin{table*}
	\caption{Experimental results on synthetic datasets under scenario S2. }
	\label{table:unbalance}
	\centering
	\resizebox{\linewidth}{!}{
		\begin{tabular}{cccccccccc}
			\toprule
			Methods & LOCAL &FedAvg  &LG &Per-FedAvg  &IFCA  & CFL  &PACFL &FPFC-$\ell_{1}$ &FPFC  \\
			\midrule
			Acc  & 85.49\%$\pm$ 0.03 & 40.63\% $ \pm $ 0.01 & 73.25\%$ \pm $ 0.05 &65.67\% $ \pm  $ 0.06 & 58.30\%$\pm$ 0.18 & 89.31\%$\pm$ 0.03  &\textbf{91.36\% $ \pm $ 0.02}  &84.91\%$\pm$ 0.04 & 90.22\%$\pm$ 0.03\\
			Num  & $\times $   &$ \times $ &$ \times $  &$ \times $ & \textbf{2.00$ \pm $0.82}   &36.00$ \pm $6.48 & 8.00 $ \pm $ 5.10 &9.67$ \pm $4.19&11.67$\pm $8.81  \\
			ARI   & $ \times $   &$ \times $ &$ \times $ &$ \times $ &0.39$ \pm $0.41   & 0.19$\pm$0.02  &  0.78 $ \pm $ 0.15&0.94$ \pm $0.04   & \textbf{0.97$\pm$0.03} \\
			\bottomrule
		\end{tabular}
	}
\end{table*}
\begin{table*}
	\caption{Experimental results on synthetic datasets under scenario S3. }
	\label{table:group2}
	\centering
	\resizebox{\linewidth}{!}{
		\begin{tabular}{cccccccccc}
			\toprule
			Methods & LOCAL &FedAvg &LG &Per-FedAvg  &IFCA  &	CFL &PACFL  &FPFC-$\ell_{1}$ &FPFC  \\
			\midrule
			Acc  & 85.60\%$\pm$ 0.01  & 50.74\% $ \pm $ 0.03 & 77.62\%$ \pm $0.08  &67.86\% $ \pm  $ 0.02 & 64.11\%$\pm$ 0.19 & 90.52\%$\pm$ 0.01 & \textbf{91.64\% $ \pm $ 0.01} &87.12\%$\pm$ 0.02 & 90.06\%$\pm$ 0.01 \\
			Num  & $\times $   &$ \times $ & $\times $ & $\times $ & 1.33$ \pm $0.47   &9.67$ \pm $2.05 &\textbf{2.00$\pm$0.00}  &\textbf{2.00$\pm$0.00} &\textbf{2.00$\pm$0.00}   \\
			ARI   & $ \times $  & $ \times $   &$ \times $ & $\times $ &0.33$ \pm $0.47    &  0.52$\pm$0.10    & \textbf{1.00$\pm$0.00}  &\textbf{1.00$\pm$0.00}   & \textbf{1.00$\pm$0.00} \\
			\bottomrule
		\end{tabular}
	}
\end{table*}
\begin{table*}
	\caption{Experimental results on synthetic datasets under scenario S4.}
	\label{table:group1}
	\centering
	\resizebox{\linewidth}{!}{
		\begin{tabular}{cccccccccc}
			\toprule
			Methods & LOCAL &FedAvg &LG &Per-FedAvg &IFCA  &	CFL   &PACFL &FPFC-$\ell_{1}$ &FPFC  \\
			\midrule
			Acc  & 87.02\%$\pm$ 0.04 &  91.96\% $ \pm $ 0.02 &87.90\% $ \pm  $ 0.04   &91.99\% $ \pm  $ 0.02 & 91.90\%$\pm$ 0.02 & \textbf{93.00\%$\pm$ 0.03} &92.71\% $ \pm  $ 0.02 &90.99\%$\pm$ 0.02 &  90.99\%$\pm$ 0.02\\
			Num  & $\times $   & $\times $ & $\times $ &$ \times $ & \textbf{1.00$\pm$0.00} &\textbf{1.00$\pm$0.00} &\textbf{1.00$\pm$0.00}  &\textbf{1.00$\pm$0.00} &\textbf{1.00$\pm$0.00}   \\
			ARI   & $ \times $   & $\times $ & $\times $ &$ \times $  &\textbf{1.00$\pm$0.00}   & \textbf{1.00$\pm$0.00}    &\textbf{1.00$\pm$0.00}   & \textbf{1.00$\pm$0.00}  & \textbf{1.00$\pm$0.00} \\
			\bottomrule
		\end{tabular}
	}
\end{table*}
\begin{table*}
	\caption{Experimental results on synthetic datasets under scenario S5.}
	\label{table:group50}
	\centering
	\resizebox{\linewidth}{!}{
		\begin{tabular}{cccccccccc}
			\toprule
			Methods & LOCAL &FedAvg &LG &Per-FedAvg &IFCA  &	CFL   &PACFL &FPFC-$\ell_{1}$ &FPFC  \\
			\midrule
			Acc  &  \textbf{81.83\%$\pm$ 0.02} &9.44\% $ \pm $ 0.01 & 63.41\% $ \pm  $ 0.01 & 30.56\% $ \pm  $ 0.07 & 60.17\%$\pm$ 0.03  & 65.47\%$\pm$ 0.02 & 81.16\% $ \pm  $ 0.02 &80.27\%$\pm$ 0.02 & 80.27\%$\pm$ 0.02\\
			Num  & $\times $  & $\times $ & $\times $  &$ \times $ & 28.33$ \pm $2.05 &37.67$ \pm $1.25 &\textbf{50.00$\pm $0.00} &\textbf{50.00$\pm $0.00} &\textbf{50.00$\pm $0.00}    \\
			ARI   & $ \times $   & $\times $ & $\times $  &$ \times $  &0.00$ \pm $0.00   &0.00$\pm$0.00  & \textbf{1.00$\pm$0.00}      & \textbf{1.00$\pm$0.00} & \textbf{1.00$\pm$0.00}  \\
			\bottomrule
		\end{tabular}
	}
\end{table*}
Table \ref{table:unbalance} reports the results for scenario S2, i.e., unbalanced cluster structure with $L=4$. FPFC has favorable clustering performance. It can be observed that only {FPFC} can identify the general cluster structure and the overestimation of the number of clusters is due to the over-subdivision of two small clusters. The results of scenario S3 are presented in Table \ref{table:group2}. {FPFC} has the best or close to the best performance in prediction and clustering. Tables \ref{table:group1} and \ref{table:group50} list the results of scenarios S4 and S5. It can be seen that in the two worst scenarios, i.e., unstructured with $L=1$ and personalized with $L=m$, our {FPFC} still has competitive performance; in particular, it still can identify the cluster structure exactly, suggesting that {FPFC} is a safe choice in practice. 
\begin{figure*}[t]
	\centering
	\includegraphics[width=0.8\columnwidth]{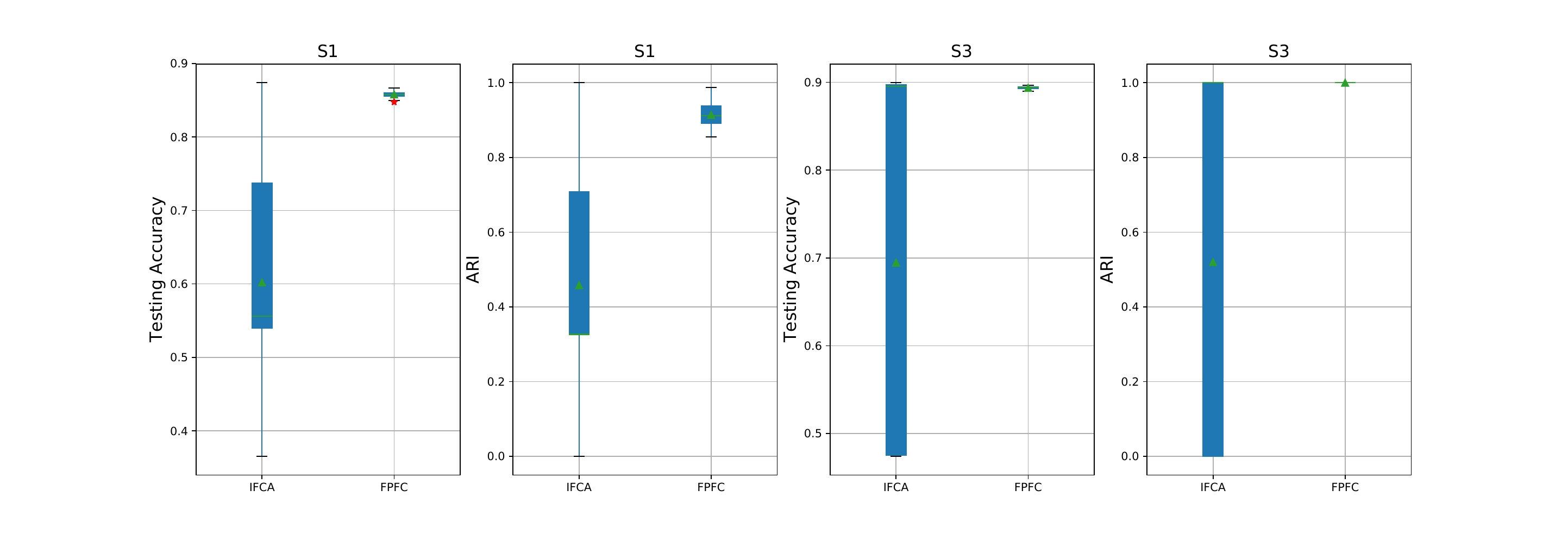}
	\caption{Distribution of testing accuracy and ARI over 50 random initializations on synthetic datasets.}
	\label{fig:initial}
\end{figure*}
\subsubsection{Sensitivity to Initial Values} \label{app:E.2.2}
Due to the nonconvexity of our objective function, the results of {FPFC} would rely on the initial values. As such, we randomly generate 50 initial values to test the initial value sensitivity. In Fig. \ref{fig:initial}, we present the distributions of testing accuracy and ARI for a random replicate under the scenarios S1 and S3. As we can see,  {FPFC} is very robust to the initial values, while {IFCA}, an iterative optimization algorithm that is also dependent on the initial values, displays large fluctuations in the quality of results. 
\begin{figure*}
	\centering
	\includegraphics[width=1\columnwidth]{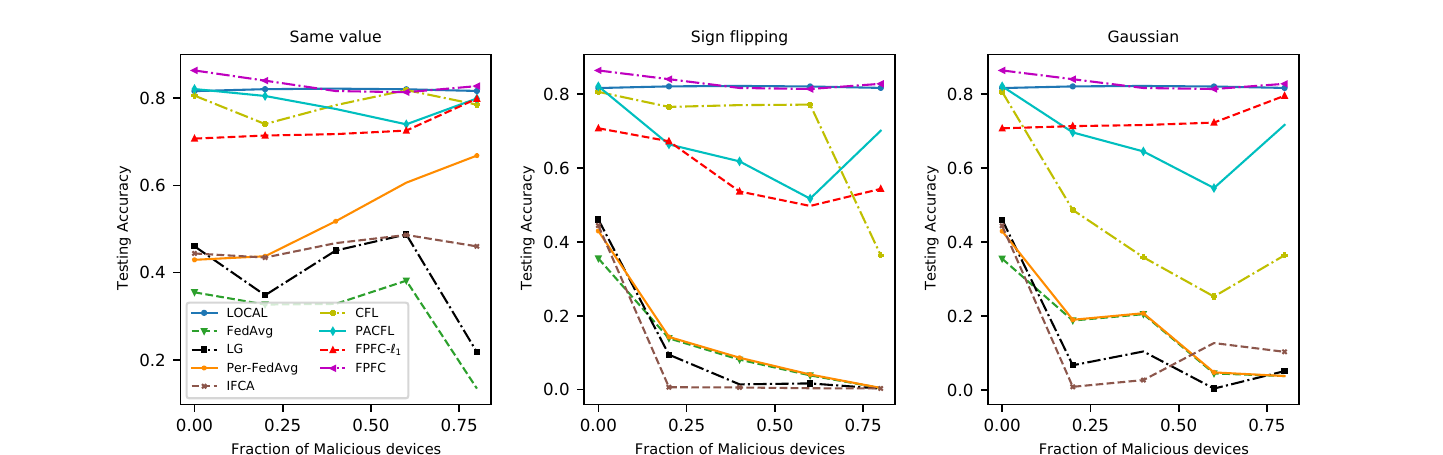}
	\caption{Robustness comparison of different methods on synthetic dataset.}
	\label{fig:syn_attack}
\end{figure*}
\begin{figure*}
	\centering
	\includegraphics[width=1\columnwidth]{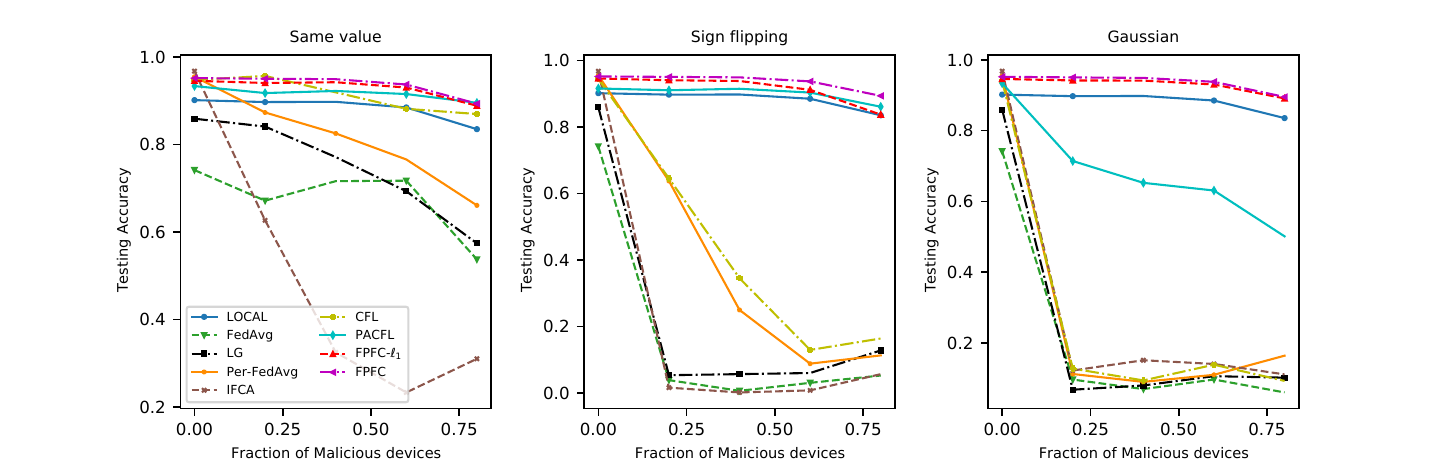}
	\caption{Robustness comparison of different methods on MNIST dataset.}
	\label{fig:mnist_attack}
\end{figure*}
\begin{figure}[h]
	\centering
	\vspace{-0.3cm}
	\includegraphics[width=0.7\columnwidth]{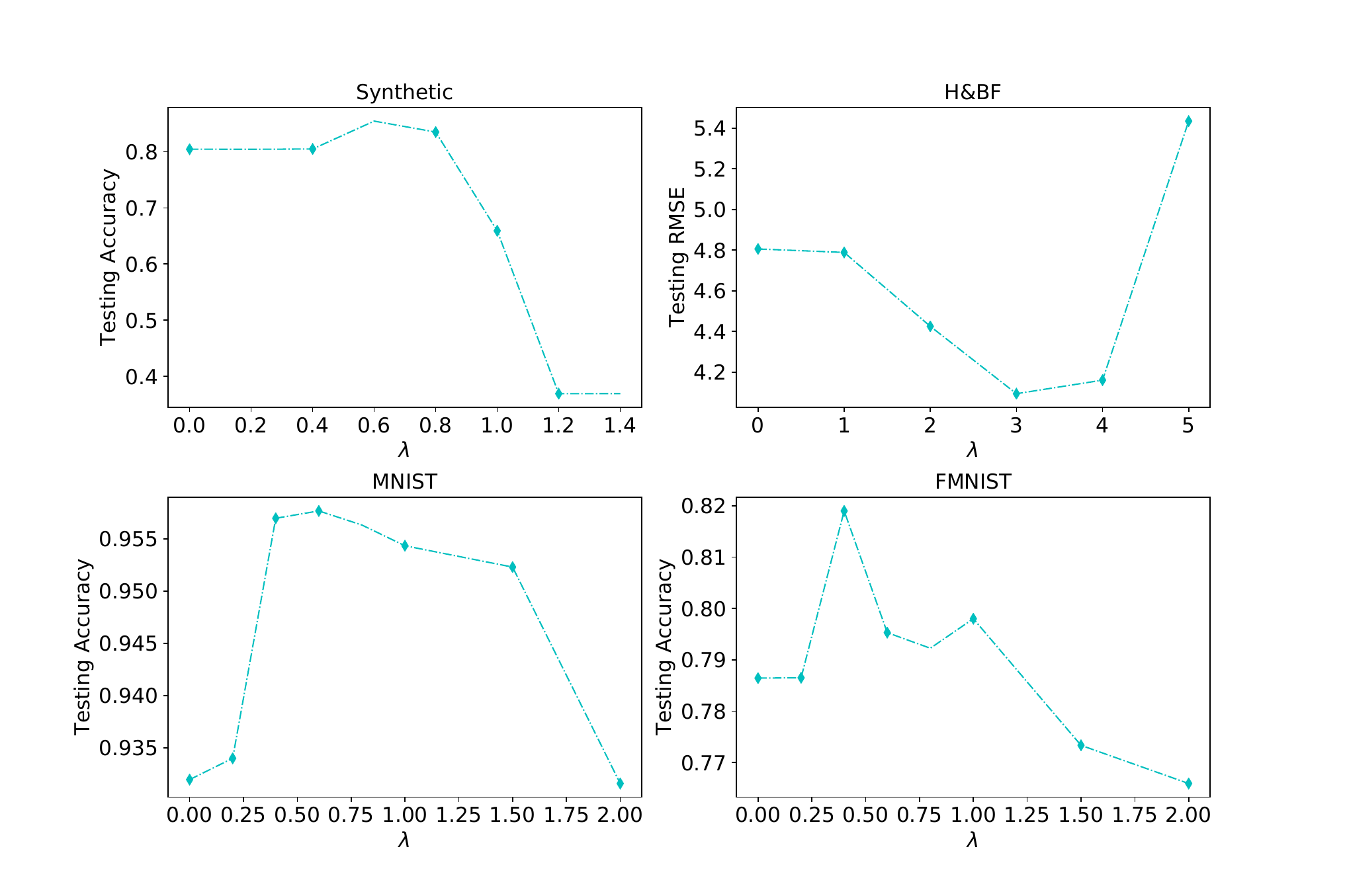}
	\vspace{-0.3cm}
	\caption{Testing accuracy/RMSE as a function of $\lambda$.}
	\label{fig:lambda}
\end{figure}

\begin{figure}[h]
	\centering
	\vspace{-0.3cm}
	\includegraphics[width=0.6\columnwidth]{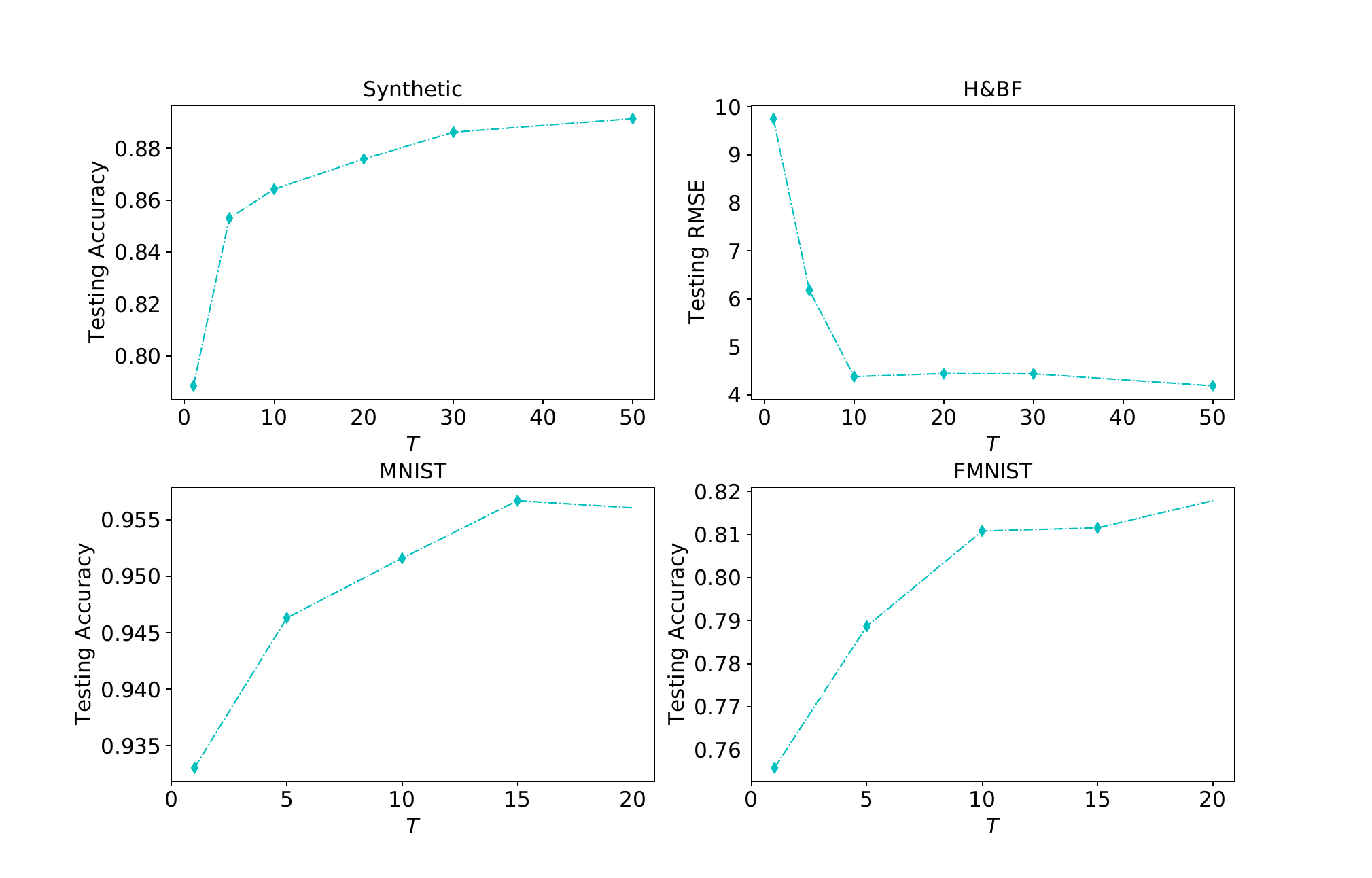}
	\vspace{-0.3cm}
	\caption{Testing accuracy/RMSE versus (maximal) number of local epochs.}
	\label{fig:epochs}
\end{figure}
\begin{figure}[h]
	\centering
	\vspace{-0.5cm}
	\includegraphics[width=0.6\columnwidth]{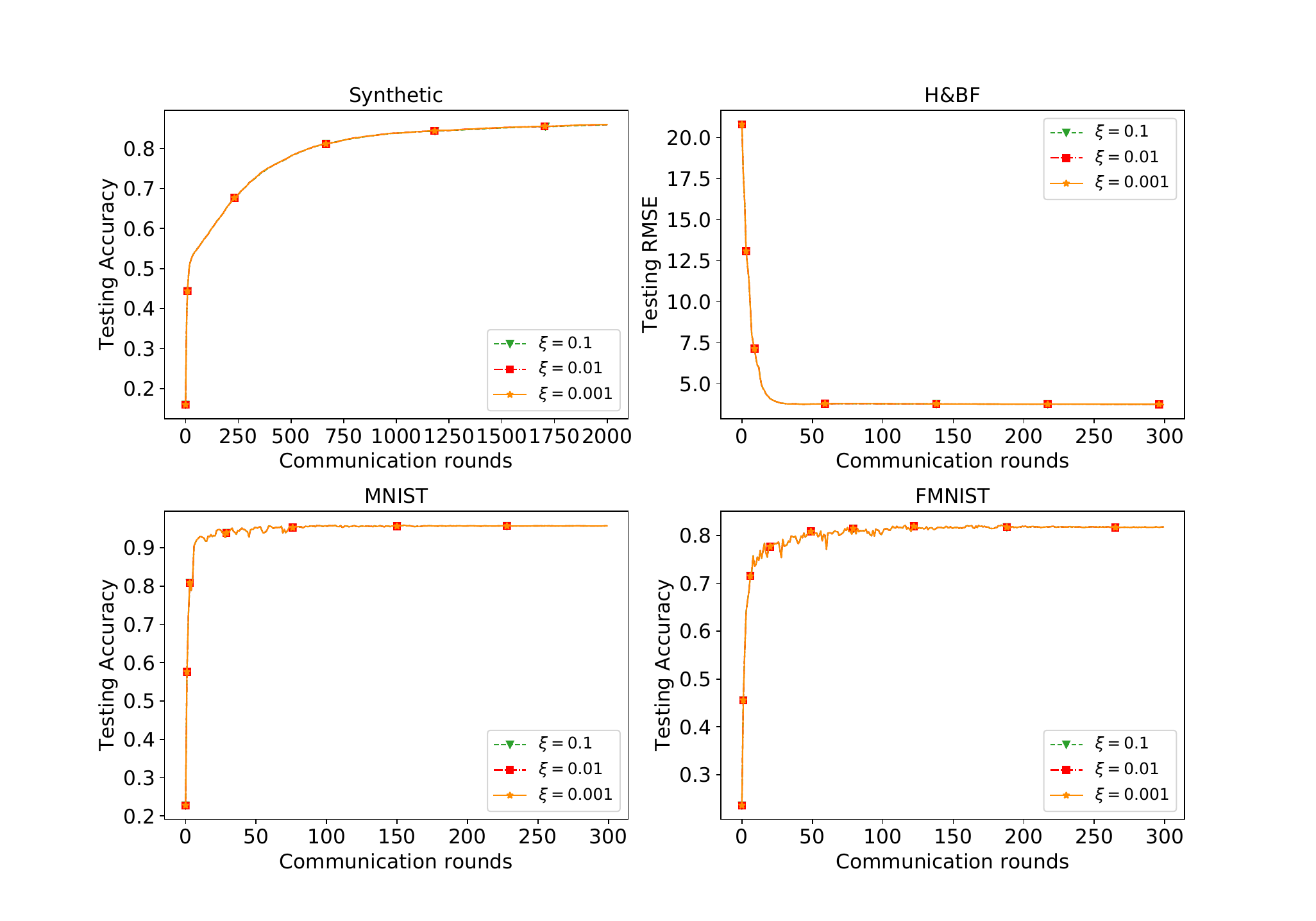}
	\vspace{-0.3cm}
	\caption{Testing accuracy/RMSE with different values of $ \xi $.}
	\label{fig:xi}
\end{figure}
\subsubsection{Complete Results on Robustness}
Fig. \ref{fig:syn_attack} and Fig. \ref{fig:mnist_attack} show the robustness results  on the synthetic and MNIST datasets, respectively, for different methods under the three Byzantine attacks. FPFC consistently exhibits higher robustness compared to other algorithms under various attack scenarios.

\subsubsection{Effect of Regularization Parameter}
In FPFC, the regularization parameter $\lambda$ controls the strength of fusion penalty. When $\lambda=0$, the penalty is absent, and the original model will degenerate into the fully personalized model; When $\lambda\to \infty$, all devices will be shrunk into one cluster, and the original model will degenerate into the global model in standard FL. Fig. \ref{fig:lambda} presents the testing RMSE/accuracy as a function of $\lambda$ on different datasets. Take synthetic data as an example, it can be observed that as $\lambda$ increases, the testing accuracy increases at first, peaks at approximately $ \lambda = 0.6 $, and then decreases, implying that a suitable $\lambda$ can improve prediction performance.

\subsubsection{Heterogeneous training}

Due to the different local storage and computational power, different devices may perform varied epochs of local updates. To simulate such heterogeneous training, 
we fix a global number of epochs $T$ and assign $T_i$ epochs (chosen uniformly at random in the range $[1,T]$) to each participating device. Fig. \ref{fig:epochs} shows the testing RMSE/accuracy versus $T$. We observe that {FPFC} has a high tolerance to heterogeneous epochs as long as the number of epochs is not too small. 

\subsubsection{Effect of $ \xi $}\label{app:E.2.6}
Due to the non-differentiability of the SCAD penalty at all points, we approximate the non-differentiable SCAD penalty by a continuously differentiable function, smoothed SCAD, and $\xi (\xi < \lambda)$ is a newly introduced parameter. Fig. \ref{fig:xi} shows the performance of different choices of $ \xi $. We can see that the parameter $ \xi $ will not affect
the results as long as it is small enough.

\end{document}